\documentclass[10pt]{article}
\usepackage[preprint]{tmlr}
\usepackage[utf8]{inputenc}
\usepackage[title]{appendix}

\usepackage{bm}
\usepackage{amsmath}
\usepackage{amssymb}
\usepackage{mathtools}
\usepackage{amsthm}
\usepackage{amsfonts}
\usepackage{thm-restate}

\def\eqref#1{equation~\ref{#1}}

\def\1{\bm{1}}

\def\mG{{\bm{G}}}

\def\mM{{\bm{M}}}
\def\mN{{\bm{N}}}

\def\mP{{\bm{P}}}

\DeclareMathAlphabet{\mathsfit}{\encodingdefault}{\sfdefault}{m}{sl}
\SetMathAlphabet{\mathsfit}{bold}{\encodingdefault}{\sfdefault}{bx}{n}

\def\gD{{\mathcal{D}}}

\def\gL{{\mathcal{L}}}
\def\gM{{\mathcal{M}}}

\def\gO{{\mathcal{O}}}
\def\gP{{\mathcal{P}}}

\def\gR{{\mathcal{R}}}

\def\sR{{\mathbb{R}}}

\newcommand{\E}{\mathbb{E}}

\newcommand{\softmax}{\mathrm{softmax}}

\newcommand{\diag}{\mathrm{diag}}

\DeclareMathOperator{\Tr}{Tr}

\newtheorem{theorem}{Theorem}
\newtheorem{assumption}{Assumption}
\newtheorem{corollary}{Corollary}
\newtheorem{lemma}{Lemma}
\newtheorem{proposition}{Proposition}

\newtheorem{remark}{Remark}

\newcommand{\nheads}{h}

\usepackage{microtype}
\usepackage[section]{placeins}
\usepackage{enumitem}
\usepackage{xcolor}
\usepackage[normalem]{ulem}

\usepackage{url}
\usepackage{xurl}
\usepackage[breaklinks=true]{hyperref}
\usepackage{zref-clever}
\zcsetup{
  cap,
  noabbrev,
  nameinlink=false
}
\zcRefTypeSetup{assumption}{
  Name-sg = Assumption,
  name-sg = assumption,
  Name-pl = Assumptions,
  name-pl = assumptions
}
\newcommand{\cref}[1]{\zcref{#1}}
\newcommand{\Cref}[1]{\zcref[S]{#1}}

\newcommand\blfootnote[1]{%
  \begingroup
  \renewcommand\thefootnote{}\footnotetext{#1}%
  \endgroup
}

\title{Incremental Learning of Sparse Attention Patterns \\ in Transformers}

\author{\name Oğuz Kaan Yüksel\thanks{Corresponding author.} \email oguz.yuksel@epfl.ch \\ \addr TML Lab, EPFL
\AND
\name Rodrigo Alvarez Lucendo \email rodrigo.alvarezlucendo@epfl.ch \\ \addr TML Lab, EPFL
\AND
\name Nicolas Flammarion \email nicolas.flammarion@epfl.ch \\ \addr TML Lab, EPFL
}

\begin{document}

\maketitle
\blfootnote{Published as a conference paper at ICML 2026. Available at \url{https://openreview.net/forum?id=vSRh1qU5sH}.}

\begin{abstract}

Transformers integrate information from multiple past positions through sparse attention patterns that copy from specific locations.
We study how these patterns form during training in a high-order Markov-chain setting, where the next token depends on several groups of past positions with differing statistical importance.
Learning proceeds in discrete stages, each marked by the emergence of a sparse attention pattern for the next most important group.
The dynamics are initially competitive, with all attention heads focusing on the dominant pattern, but later become cooperative, with different heads specializing in distinct patterns.
For a tractable surrogate, the dynamics reduce to those of a tensor factorization problem, enabling us to prove stagewise convergence.
Empirically, the number of stages completed depends on training set size, suggesting a form of implicit complexity control through early stopping.
Overall, we give a theoretical account of how structured attention patterns and head specialization emerge without an explicit curriculum.

\end{abstract}

\section{Introduction}

Transformers often acquire structure during training in stages: simple patterns appear first, and more complex behaviors emerge only later.
Such stage-wise behavior has been observed in sequential tasks and language-modeling settings, including sudden capability gains and the progressive formation of attention circuits \citep{abbe2023transformers,edelman2024evolution,chen2024sudden}.
However, even in simple settings, it remains unclear how individual attention heads coordinate.
Do heads \emph{specialize} in different patterns from the start, or do they first \emph{compete} for the same statistically dominant signal before differentiating?
Understanding this mechanism is important for explaining how structured attention patterns emerge from standard training dynamics.

A core primitive behind such structure formation is \emph{copying}: moving information from one position to another so that it can be used in downstream computation.
Copying is central in language modeling, where models must reuse phrases, entities, and contextual information, and in algorithmic reasoning, where intermediate values often need to be duplicated for use in later computations.
Transformers implement copying through sparse attention patterns that concentrate probability mass on selected past positions; the simplest such pattern attends to a single position and is a subcircuit of \emph{induction heads} \citep{elhage2021mathematical,olsson2022context}.
Exact sparsity requires the attention logits to separate in scale between attended and unattended positions.
Thus, sparse copying is not only a representational primitive but also a \emph{dynamical} one: the pressure to concentrate attention shapes how attention circuits form during training.

In this paper, we examine this mechanism in a controlled high-order Markov-chain task, illustrated in \Cref{fig:main}.
The next token depends on several groups of past positions, and each group is processed by a feature matrix with different statistical importance.
The model must therefore learn multiple copying patterns and combine the copied information to predict the next token.
\begin{figure*}[t]
      \begin{center}
            \includegraphics[width=0.8\linewidth]{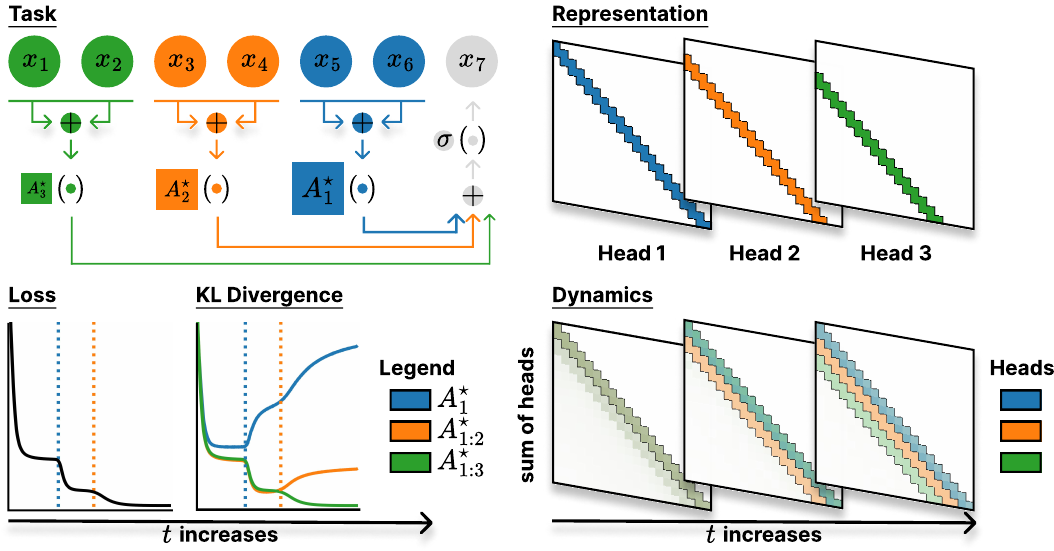}
      \end{center}
      \caption{
            (Top left) The task is based on a high-order Markov chain, where the next token depends on multiple past tokens with different importance weights.
            The context is divided into groups of positions, each aggregated and mapped to prediction logits by an associated feature matrix \(A_k^\star \in \sR^{d \times d}\), with larger \(\|A_k^\star\|\) indicating greater influence on the output.
            (Top right) An idealized representation of the task in a multi-head single-layer attention model.
            Each head represents an individual sparse attention pattern required to solve the task.
            (Bottom left) Transformers learn the task incrementally, with each stage corresponding to the acquisition of a sparse attention pattern, as indicated by the KL divergence between the transformer and predictors \(A_{1:i}^\star\) that depend only on a prefix of the importance-ordered groups, as defined in \Cref{eq:ground_truth}.
            (Bottom right) The learning dynamics transition from competitive, where all heads focus on the statistically most important pattern, to cooperative, where different heads specialize in different patterns.
      }
      \label{fig:main}
\end{figure*}
We find that single-block decoder transformers learn this task incrementally.
At first, all heads focus on the most statistically important positions, producing a competitive phase in which the heads learn the same sparse attention pattern instead of immediately dividing labor.
Only later do the heads specialize: one \emph{offshooting} head leaves the shared pattern and begins to copy from the next most important positions, while the remaining heads compensate for the residual contribution of the offshooting head.
This competitive-to-cooperative transition is the central phenomenon of the paper.
It shows that staged learning is not only visible in model performance, but also in how attention heads coordinate and specialize during training.

To explain this behavior, we introduce a simplified regression variant of the task and analyze its population gradient-flow dynamics.
Under this simplification, the training dynamics reduce to nonlinear tensor factorization, a well-studied class of nonconvex problems \citep{arora2019implicit,razin2021implicit,li2021towards,jin2023understanding}.
Our theoretical results give sufficient conditions under which near-symmetric initialization induces \emph{coupled dynamics}: all heads converge toward the same dominant sparse pattern during the competitive phase.
We then analyze how, near this coupled state, an offshooting head can move toward a new pattern while the other heads compensate, yielding the cooperative phase.
Together, these results connect the observed head-level specialization in transformers to saddle-to-saddle dynamics.
Our main contributions are as follows:
\begin{itemize}
      \item \textbf{A simple task for positional incremental learning.}
      We introduce a high-order Markov-chain task that highlights sparse attention formation as a key mechanism for staged learning.
      Unlike prior in-context Markov-chain settings with more intricate multi-layer circuits \citep{edelman2024evolution}, the staged dynamics already appear here with a single self-attention layer.

      \item \textbf{Competitive-to-cooperative head dynamics.}
      Transformers first enter a competitive phase, where all heads focus on the statistically most important positions, before transitioning to a cooperative phase, where different heads specialize in different patterns.
      This phenomenon appears across variations in the importance ordering, within-group attention weights, interval overlap, initialization scale, optimizer, and dataset size.

      \item \textbf{A gradient-flow account of the stages.}
      We analyze a simplified regression model in which near-symmetric heads evolve as a coupled system, jointly converging toward the same dominant pattern during the competitive phase.
      The analysis also captures the subsequent cooperative phase, where an offshooting head learns a new pattern while the remaining heads compensate.

      \item \textbf{Dataset size and effective complexity.}
      In low-data regimes, early-stopped models approach simpler effective predictors that use only the most important positions, suggesting a link between stage-wise learning and the effective complexity selected during training.
\end{itemize}

\section{Stage-wise Formation of Sparse Patterns}

In this section, we describe the data generation process, how transformers can represent the solution, the experimental evidence for incremental learning of sparse attention patterns, and how these dynamics vary across settings.

\subsection{Markov Chains with Importance Structure}
\label{sec:markov}

We consider a next-token classification task based on a discrete Markov chain of order \(w\) over a dictionary \(\gD\) with \(|\gD| = d\).
Each token in \(\gD\) is represented by a standard basis vector in \(\sR^d\).
The initial tokens \(x_{-w+1}, \ldots, x_0\) are sampled independently and uniformly from \(\gD\).
For each \(t \in [T]\), the remaining tokens are generated according to
\begin{equation}\label{eq:markov_chain}
      x_{t} \sim \softmax\left(\sum_{k=1}^{\nheads} A_k^\star \sum_{i \in I(k)} \alpha_i x_{t-i}\right)\,,
\end{equation}
where \(A_k^\star \in \sR^{d \times d}\) are fixed feature matrices, \(I(k)\) are disjoint subsets that partition \(\{1, \ldots, w\}\), and \(\alpha_i\) are nonnegative importance weights satisfying \(\sum_{i \in I(k)} \alpha_i = 1\) for all \(k \in [\nheads]\).
This task is simple but captures three features relevant to sequence modeling: (i) it is sequential, requiring the model to integrate information from past positions, (ii) it has positional structure, with the past lags partitioned into distinct groups, and (iii) different positions can have different importance, as determined by the feature matrices \(A_k^\star\) and scalars \(\alpha_i\).
As \(I(k)\) and \(A_k^\star\) can be permuted without changing the data generation process, we assume without loss of generality that \(\|A_1^\star\| \geq \|A_2^\star\| \geq \ldots \geq \|A_\nheads^\star\|\) so that \(I(1)\) is paired with the most influential feature matrix.

One particular choice of interest is to have \(I(k)\) be contiguous blocks of indices that start from the most recent position, i.e., for some \(0 = i_0 < i_1 < i_2 < \ldots < i_{\nheads-1} < i_{\nheads} = w\) and \(I(k) = \{i_{k-1}+1, \ldots, i_k\}\).
Contiguous blocks serve as one illustrative default, motivated by the stronger statistical correlations between nearby tokens in natural language.
Notably, when each \(I(k)\) is a singleton, the task recovers position-wise copying followed by a linear feature map.
Lastly, the task is not \emph{identifiable}: the sum in \Cref{eq:markov_chain} admits multiple decompositions with different feature matrices \(A_k^\star\) and partitions \(I(k)\); see \Cref{app:identifiability} for examples.

\subsection{Representing the Task with Sparse Attention}
\label{sec:representation}

We construct a single-layer multi-head attention model that solves the task.
The construction assigns one head to each group: the attention pattern copies the relevant past tokens, and the value matrix applies the corresponding feature map.

Let \(X \in \sR^{d \times (T+w)}\) be the input data matrix with columns \(x_{-w+1}, \ldots, x_{0}, x_1, \ldots, x_T\).
For a transparent construction, we encode positional information using one-hot vectors in \(\sR^{T+w}\) and concatenate them with the input as follows:
\footnote{We use one-hot positional encodings for clarity. The construction and experiments are compatible with relative or sinusoidal encodings that avoid \(\mathcal{O}(T)\) embedding dimension.}
\begin{equation*}
      \tilde{X} = \begin{pmatrix}
            X \\
            I_{T+w}
      \end{pmatrix} \in \sR^{(d+T+w) \times (T+w)}\,.
\end{equation*}
We denote the columns of \(\tilde{X}\) as \(\tilde{x}_i \in \sR^{d+T+w}\), the position-augmented embedding of token \(x_i\).
The transformer takes \(\tilde{X}\) as input and produces the output \(Y \in \sR^{d \times T}\) with columns \(y_0, \ldots, y_{T-1}\) as follows:
\begin{equation*}
      \begin{split}
            y_t &= \softmax\left(\sum_{k=1}^{\nheads} V_k \tilde{X} a_t^{(k)}\right)\,, \\
            a_t^{(k)} &= \softmax\left(\gM_{T-t}\left(\tilde{X}^\top K_k^\top Q_k \tilde{x}_{t}\right)\right)\,,
      \end{split}
\end{equation*}
where \(Q_k, K_k \in \sR^{(d+T+w) \times (d+T+w)}\) are the query and key matrices of head \(k\), \(V_k \in \sR^{d \times (d+T+w)}\) is the value matrix, and \(\gM_{p}\) sets the last \(p\) entries to \(-\infty\) to apply causal masking.

For head \(k\), we let \(a_t^{(k), \star}\) denote the ideal positional attention pattern corresponding to \(I(k)\), defined entrywise by
\[
      \left(a_t^{(k), \star}\right)_s =
      \begin{cases}
            \alpha_i\,, & \text{if } s = t-i \text{ for some } i \in I(k)\,, \\
            0\,, & \text{otherwise}\,,
      \end{cases}
\]
where \(s\) ranges over the positions in the input sequence.
We choose \(V_k\) to apply \(A_k^\star\) to the token coordinates and discard the positional coordinates, i.e., \(V_k = [A_k^\star \; 0]\).
With this choice,
\[
      V_k \tilde{X} a_t^{(k),\star} = A_k^\star \sum_{i \in I(k)} \alpha_i x_{t-i}\,,
\]
so each head recovers one contribution to the logits in \Cref{eq:markov_chain}.
By \emph{sparse attention}, we mean the limiting case in which attention mass is concentrated on a subset of positions and vanishes on the complement.
Exactly learning such a pattern requires the logit gap between attended and unattended positions to diverge.
In practice, we expect finite logit separations that approximate these sparse attention patterns.
This sparse pattern can be realized using only positional coordinates by setting
\[
      K_k^\top Q_k = \lambda \sum_{i \in I(k)} \sum_{p=w}^{T+w} e_{d+p-i} e_{d+p}^\top\,,
\]
where \(\lambda > 0\) is a scaling constant and \(e_i\) is the \(i\)-th standard basis vector in \(\sR^{d+T+w}\).
As \(\lambda \to \infty\), the attention scores converge to the desired sparse pattern.

This construction is not unique as there are many \(Q_k\) and \(K_k\) that can realize the same attention pattern.
In particular, there is a symmetry in which \((Q_k, K_k)\) can be replaced with \(\left(M^{-1} Q_k, M^{\top} K_k\right)\) for any invertible matrix \(M\) without changing the attention scores.
Moreover, as there are \(\nheads\) heads to learn, the construction has a permutation symmetry.
The permutation symmetry plays a central role in our analysis of the learning dynamics, as we discuss in \Cref{sec:coupled}.

\subsection{Transformers Learn Incrementally}
\label{sec:incremental}

\begin{figure*}[t!]
      \begin{center}
            \includegraphics[width=0.8\linewidth]{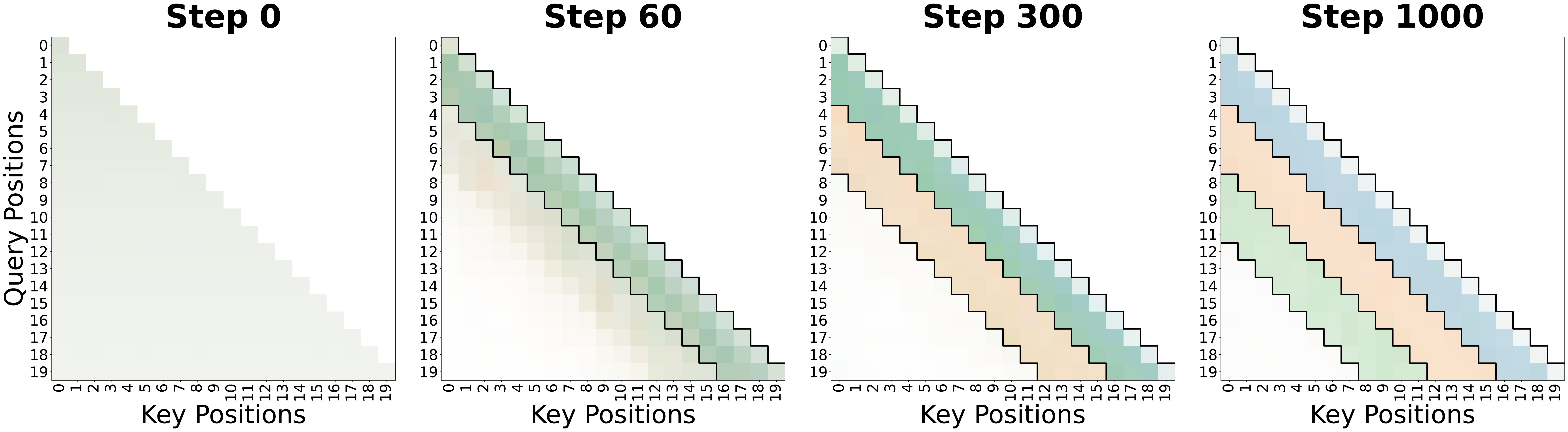}
      \end{center}
      \caption{
            The sum of learned attention patterns for \(\nheads=3, w=12\) at different stages of training with blue, yellow, and green corresponding to different heads.
            At \(t=0\), the attention is uniform as the model is randomly initialized.
            At \(t=60\), all heads learn from the positions in \(I(1)\), indicated by the overlapping blue, yellow, and green, with some residual attention also placed on the positions in \(I(2)\).
            At \(t=300\), a head learns from the positions in \(I(2)\) whereas two heads still focus on \(I(1)\).
            At \(t=1000\), the model finally learns to integrate all positions, with each head specializing in a different pattern.
            The main diagonal does not have the same intensity as the other positions, as it is learned directly from the input via the skip connection.
      }
      \label{fig:attention}
\end{figure*}

We train single-block decoder transformers with \(\nheads\) heads on sequences sampled from \Cref{eq:markov_chain}, minimizing cross-entropy loss over the full sequence. The initial tokens \(x_{-w+1}, \ldots, x_0\) are excluded since they are not drawn from the process.
We keep the architecture as close to standard practice as possible.
The architecture and optimization details are provided in \Cref{app:exp}.

We sample feature matrices \(A_k^\star\) uniformly from orthogonal matrices and then scale them by positive scalars \(m_k\).
These constants are chosen geometrically, i.e., \(m_k = m^{\nheads-k} b_0\) where \(m > 1\) is the multiplicative constant and \(b_0 > 0\) is the base scale.
This results in an importance hierarchy across the feature matrices, while features within the same matrix are equally important.
In particular, \(A_1^\star\) has the largest norm and thus contains the most influential features, while \(A_\nheads^\star\) has the smallest norm and contains the least important ones.
For simplicity, we choose \(I(k)\) as in \Cref{sec:markov} with equal-length intervals of size \(w/\nheads\), and set \(\alpha_i = 1/|I(k)|\) for \(i \in I(k)\).
See \Cref{app:exp} for experimental details and \Cref{app:additional_experiments} for additional experiments.

We observe that the transformers learn the task incrementally, with each stage corresponding to the acquisition of a sparse attention pattern as in \Cref{fig:attention}.
At initialization, all heads have uniform attention.
Early in training, they focus primarily on \(I(1)\), the most statistically important positions.
This is the \emph{competitive phase}: the heads compete to learn from these positions, resulting in overlapping attention patterns with initialization-induced deviations.
The dynamics then enter the \emph{cooperative phase}, where heads gradually specialize in different patterns, with one head learning from \(I(2)\) while another eventually focuses on \(I(3)\).
Notably, training consistently recovers the sparse, disjoint representation, even though the task admits multiple decompositions.

To probe these dynamics in function space, we train transformers with restricted maximum context lengths \(c = 4, 8, 12\).
With \(c = 4\), the model can only access \(I(1)\); with \(c = 8\), it can additionally access \(I(2)\); with \(c = 12\), it can access all relevant positions and solve the task fully.
In \Cref{fig:kl} (right), we plot the Kullback-Leibler (KL) divergence between each restricted model and a full-context transformer over training.
The full-context transformer first approaches the \(c = 4\) model, then the \(c = 8\) model, before approaching the \(c = 12\) model.
This indicates that the transformer jointly learns each attention pattern and its associated feature matrix at every stage, rather than acquiring them separately.

We obtain a similar picture by comparing the transformer's predictions to ground-truth predictors that depend only on the positions in \(I(1)\), \(I(1) \cup I(2)\), and \(I(1) \cup I(2) \cup I(3)\):
\begin{equation}\label{eq:ground_truth}
      f_{A_{1:i}^\star} = \softmax\left(\sum_{k=1}^i A_k^\star \sum_{j \in I(k)} \alpha_j x_{t-j}\right)\,.
\end{equation}
\Cref{fig:kl} (left) shows the same incremental pattern: the transformer first approaches \(f_{A_1^\star}\), then \(f_{A_{1:2}^\star}\), before approaching \(f_{A_{1:3}^\star}\).
This mirrors what \citet{edelman2024evolution} observe: stages characterized by sub-\(n\)-grams.

\begin{figure*}[t!]
      \begin{center}
            \includegraphics[page=1,width=0.4\linewidth]{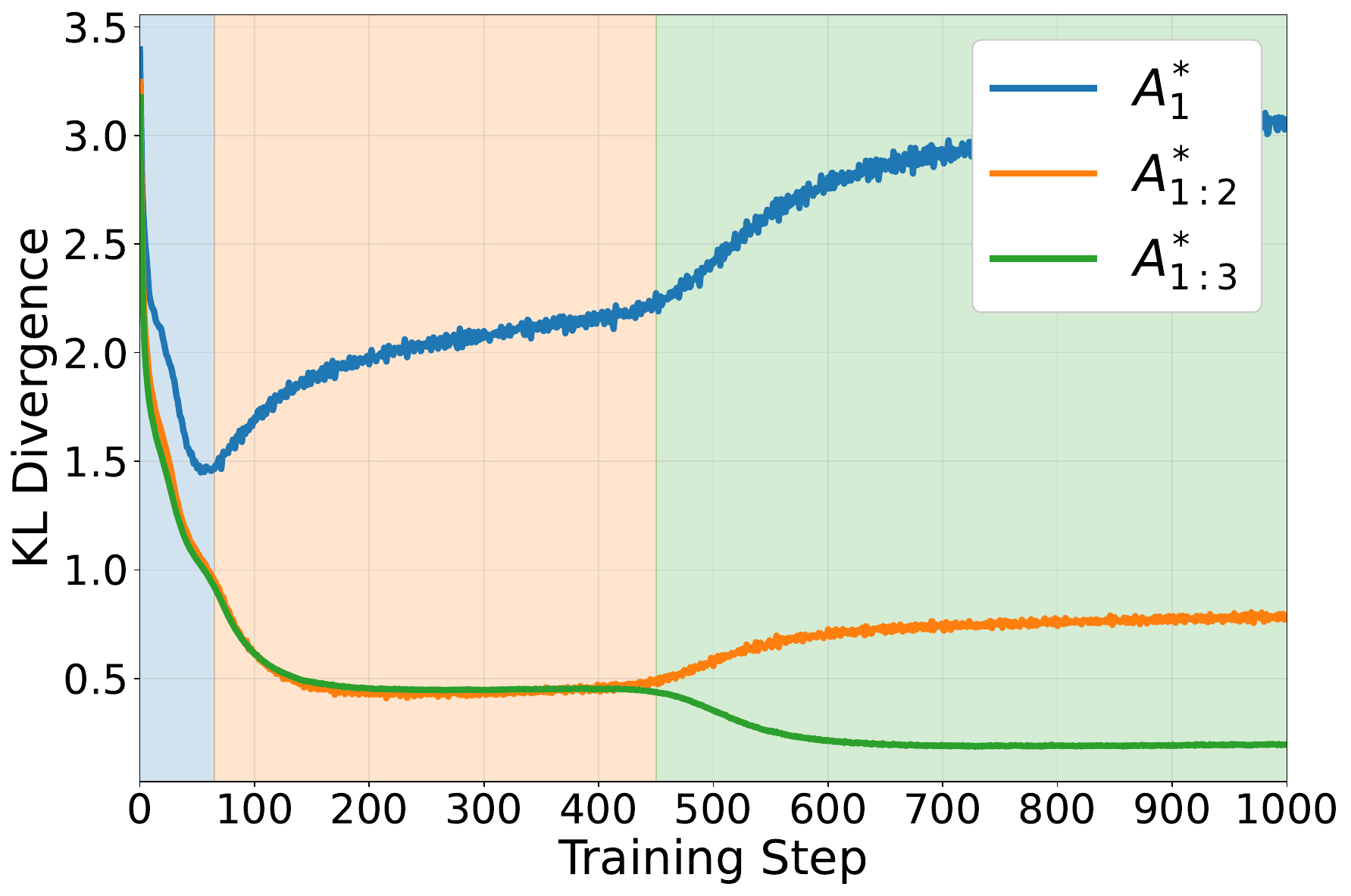}
            \includegraphics[page=1,width=0.4\linewidth]{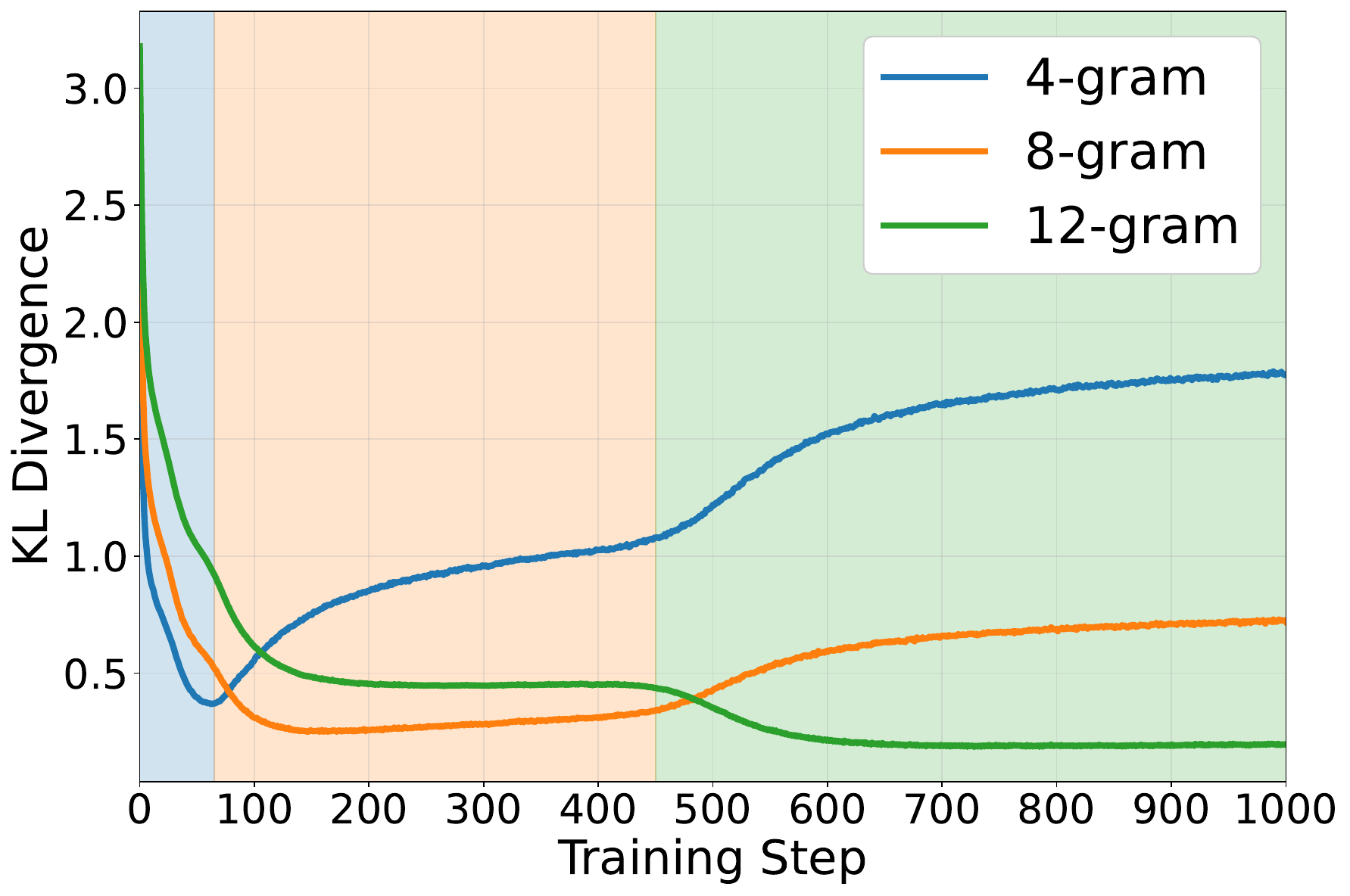}
      \end{center}
      \caption{
            Stage-wise learning in function space: each curve traces the proximity of the full model to a predictor along the training trajectory.
            Curves to intermediate predictors dip and then rise as the trajectory passes through their neighborhood and incorporates additional positions, while the curve to the full predictor decreases and stays low.
            (Left) KL divergence between the ground-truth predictors that depend only on the positions in \(I(1)\), \(I(1) \cup I(2)\), and \(I(1) \cup I(2) \cup I(3)\), and the predictions of the full-context transformer.
            (Right) KL divergence between the predictions of transformers with restricted context lengths \(c = 4, 8, 12\), and those of the full-context transformer.
      }
      \label{fig:kl}
\end{figure*}

\subsection{Ablation Studies}
\label{sec:minimal}

\begin{figure*}
      \begin{center}
            \includegraphics[page=1,width=0.4\linewidth]{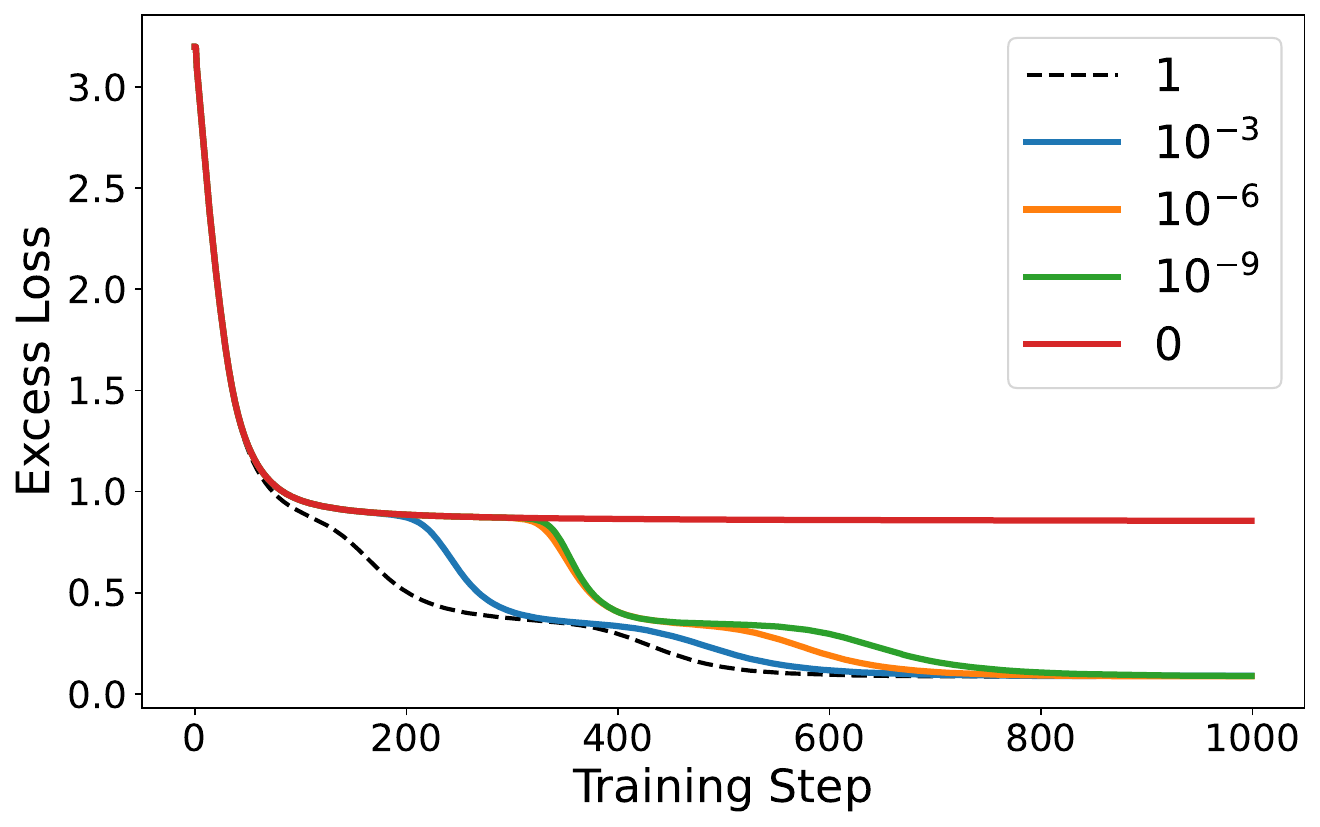}
            \includegraphics[page=1,width=0.4\linewidth]{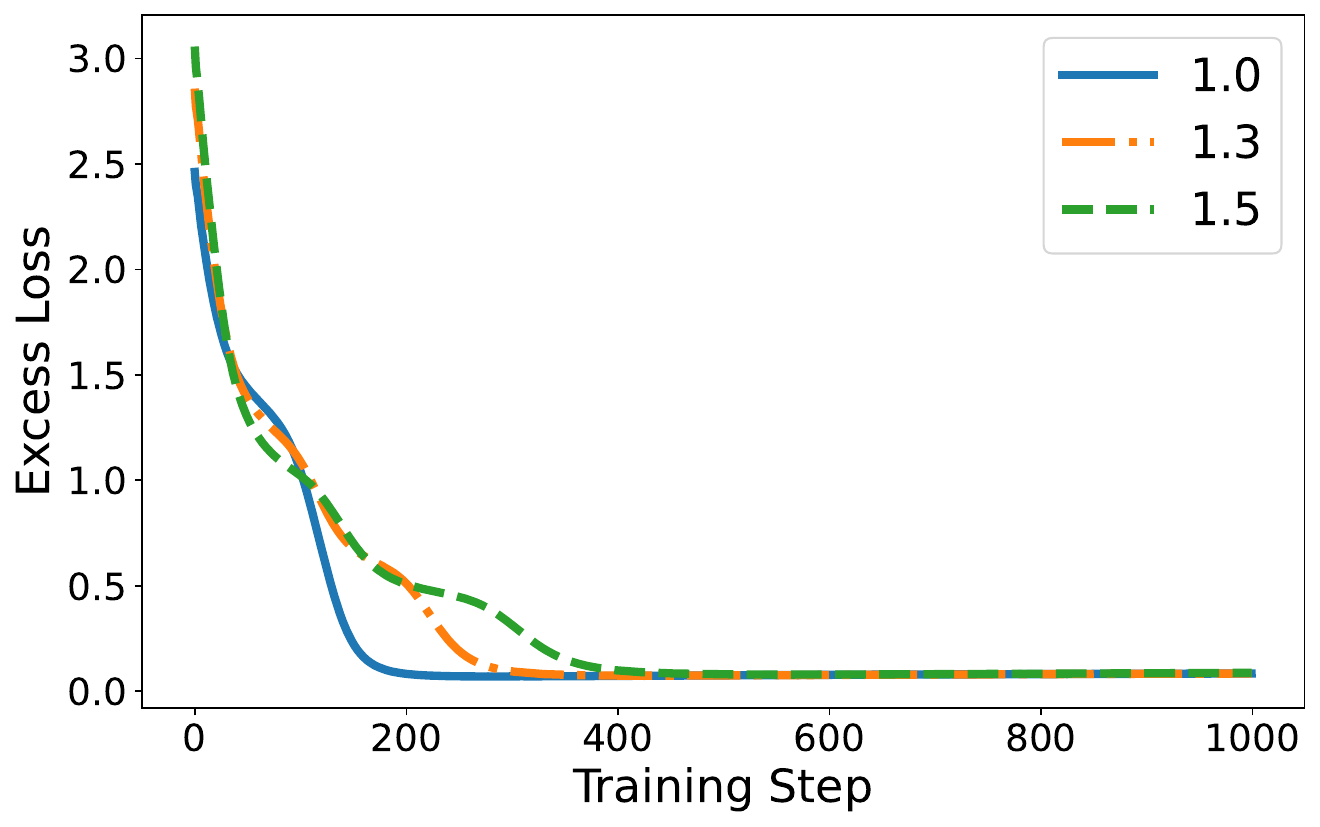}
      \end{center}
      \caption{
            (Left) Excess loss of the minimal architecture with different initialization scales.
            Smaller scales yield slower learning; at \(u=0\), the heads remain exactly symmetric and only the first pattern is learned.
            (Right) Excess loss of the minimal architecture with different multiplicative constants \(m\) that determine the importance hierarchy.
            When \(m > 1\), the importance gap between features attracts the initialization to each saddle sequentially, yielding three distinct stages.
            When \(m=1\), there is no importance ordering, so the initialization is not preferentially attracted to any saddle; one pattern is learned first due to random symmetry breaking, and the remaining two are acquired simultaneously, collapsing to two stages.
      }
      \label{fig:ablations}
\end{figure*}

To identify the architectural components needed for incremental learning, we remove elements absent from the idealized construction in \Cref{sec:representation}, such as layer normalization and residual connections.
We further reduce the product \(K_k^\top Q_k\) to a single matrix \(A_k\).
Individually or in combination, these simplifications do not significantly alter the incremental dynamics; we plot the results in \Cref{fig:main}.

We perform ablation studies with this minimal architecture.
We first vary the initialization scale of the attention matrices \(A_k\) and set value matrices to zero to isolate the attention dynamics.
We initialize \(A_k\) from a uniform distribution on \([-u, u]\), where \(u\) is the initialization scale.
\Cref{fig:ablations} (left) shows that the speed of incremental learning is affected by the initialization scale, with smaller scales resulting in slower learning.
At the extreme \(u=0\), we observe that the model only learns a single pattern and does not progress further.
The heads are exactly symmetric at \(u=0\), and breaking this symmetry requires a small perturbation.

We also vary the multiplicative constant \(m\) that controls the importance hierarchy in the data generation process.
\Cref{fig:ablations} (right) shows that the number of distinct stages decreases to two for \(m=1\), where there is no importance ordering.
When \(m=1\), the model first learns one pattern and then acquires the remaining two simultaneously.
For \(m=1.3\) and \(m=1.5\), three distinct stages emerge, though more intertwined for \(m=1.3\) with less pronounced bumps in the loss curve.
Overall, larger importance gaps yield more clearly separated stages; when the gaps shrink, multiple patterns can be acquired at nearly the same time.

Beyond initialization scale and the importance hierarchy, \Cref{app:additional_experiments} reports ablations over reversed orderings, overlapping intervals, non-uniform \(\alpha_i\), and additional heads.
The competitive-to-cooperative dynamics appear across these variations, indicating that the staged behavior is not specific to the particular choices in our default setup.
In less structured settings, such as overlapping intervals, there can be multiple sparse solutions, and the initialization and feature matrices can influence which solution is selected.

\subsection{Dataset Size and Effective Complexity}
\label{sec:generalization}

Lastly, we study the effect of dataset size on incremental learning.
As we decrease the dataset size below certain thresholds, the number of stages observed in training decreases, as shown in \Cref{fig:generalization} (left).
\Cref{fig:generalization} (right) plots the KL divergence between restricted-context transformers and the low-data trained transformer.
The trend is similar to that observed in \Cref{fig:kl}, but with fewer stage transitions as the dataset size decreases.

In low-data regimes, the early-stopped transformer approaches simpler effective predictors that depend only on the most important positions.
This is consistent with a regularization effect associated with the training trajectory, where shorter effective context corresponds to a simpler, misspecified hypothesis class.
\citet{yuksel2025long} argue that such misspecification can be beneficial in low-data regimes, improving sample efficiency.
Transformers with early stopping seem to select the effective context length automatically, suggesting a possible sample-complexity benefit.

Notably, this structure appears to persist in the overfit regime: with 600 samples, for example, the model concentrates attention on the first two patterns and gives minimal weight to the third, even when the corresponding feature matrices overfit to training noise.
This hierarchical allocation could explain why early stopping recovers a clean predictor on the top-\(k\) learnable patterns.

\begin{figure*}[!ht]
      \begin{center}
            \includegraphics[width=0.9\linewidth]{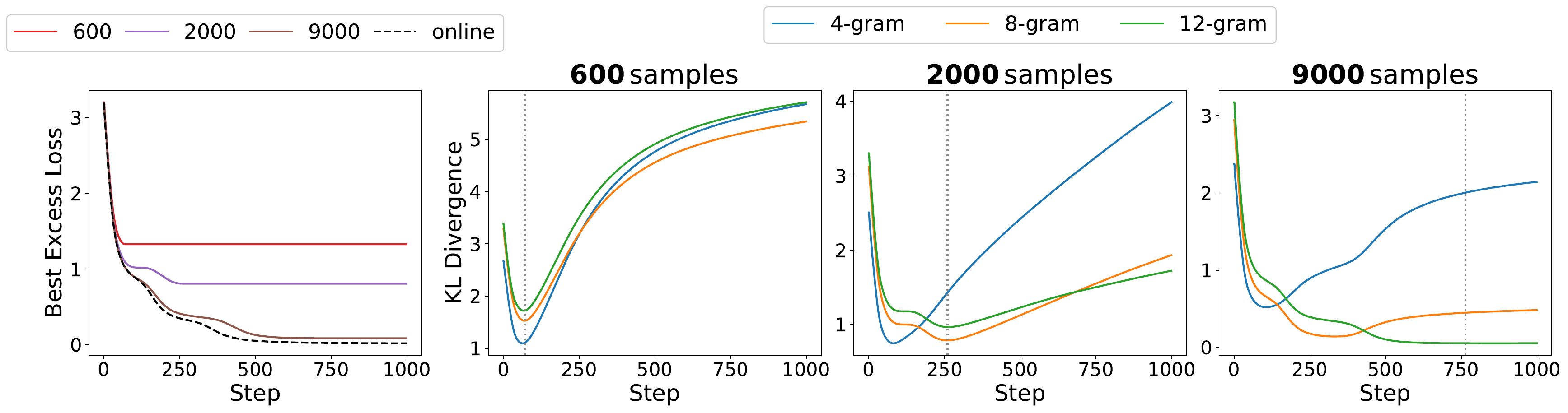}
      \end{center}
      \caption{
            Smaller datasets reduce the number of stages reached during training.
            (Left) The best validation loss as a function of the dataset size.
            (Right) The KL divergence between restricted-context transformers and the trained transformer.
            Dashed lines indicate the first step that obtains the best excess loss.
      }
      \label{fig:generalization}
\end{figure*}

\section{Training Dynamics on Regression Variant}

In this section, we introduce a tractable surrogate regression variant of the classification task from \Cref{sec:markov}, using the minimal architecture of \Cref{sec:minimal}.
We study its training dynamics by analyzing the population gradient flow.

\subsection{The Regression Model}
\label{sec:regression}

Consider the regression task with independent \(\left(x_1, \ldots, x_T\right) \sim \gP_X\) and \(\xi \sim \gP_\xi\), and
\begin{equation*}
      y^\star(X, \xi) = \sum_{k=1}^{\nheads} A_k^\star X s_k^\star + \xi\,,
\end{equation*}
where \(s_k^\star \in \sR^T\) is the vector with entries \(\alpha_i\) for \(i \in I(k)\) and zero otherwise.
For simplicity, we consider \(|I(k)| = 1\), so that each \(s_k^\star\) is one-hot.
Let \(m_k^\star = \|A_k^\star\|_F\) and \(V_k^\star = A_k^\star / \|A_k^\star\|_F\) for all \(k \in [\nheads]\), indexed in decreasing order of \(m_k^\star\).
For notational convenience, we also index the remaining positions by \(k = \nheads + 1, \ldots, T\), with \(s_k^\star\) the corresponding standard basis vectors, \(A_k^\star = 0\), and \(m_k^\star = 0\).

We make some assumptions regarding the distributions \(\gP_X, \gP_\xi\), and the feature matrices.
\begin{assumption}\label{assum:data}
      The noise is zero-mean, i.e., \(\E[\xi] = 0\) and the data is normalized, i.e.,
      \begin{equation*}
            \forall i, j \in [T]\,, \quad \E\left[x_i x_j^\top\right] = \1_{i = j} I_d\,.
      \end{equation*}
\end{assumption}
\begin{assumption}\label{assum:orthogonal}
      The feature matrices are orthogonal, i.e.,
      \begin{equation*}
            \forall i, j \in [\nheads]\,, \quad \langle V_i^\star, V_j^\star \rangle = \Tr\left((V_i^\star)^\top V_j^\star\right) = \1_{i = j}\,.
      \end{equation*}
\end{assumption}
\begin{assumption}\label{assum:distinct}
      The feature norms are distinct and positive, i.e., \(m_1^\star > \ldots > m_{\nheads}^\star > 0\).
\end{assumption}

These assumptions are made for analytical tractability.
The isotropic covariance assumption turns the population objective into a pure tensor-factorization problem rather than a matrix-sensing problem, the orthogonality assumption makes the feature ordering unambiguous, and distinct norms break the symmetry between features.
Note that any isotropic covariance reduces to this case by absorbing the scale into the norms \(m_k^\star\); e.g., random one-hot tokens.
Relaxing them would require replacing this scalar ordering with an ordering in the geometry induced by the data covariance and target tensor, which we leave outside the scope of the present analysis.

We use the minimal architecture obtained in \Cref{sec:minimal} with the following modifications.
The attention scores are computed only via the inner product of position vectors, instead of the concatenated position and data vectors.
As the problem is a regression task on the final token, we only need the last row of the matrix \(Q_k\) which we denote by \(q_k \in \sR^T\).
Then, the resulting model is as follows:
\begin{equation*}
      y_{\theta}(X) = \sum_{k=1}^{\nheads} V_k X s_k\,, \quad \text{with} \quad s_k = \softmax(q_k)\,,
\end{equation*}
where \(\theta = (V_1, \ldots, V_{\nheads}, q_1, \ldots, q_{\nheads})\) are the learnable parameters.
We set the loss to the mean-squared loss:
\begin{equation}\label{eq:loss}
      \gL(\theta) = \frac{1}{2} \E_{x_1,\ldots,x_{T}, \xi}\left[\|y_{\theta}(X) - y^\star(X,\xi)\|^2\right]\,.
\end{equation}

We study the gradient flow dynamics of the population loss in \Cref{eq:loss}, i.e., we consider the continuous-time limit of gradient descent with infinitesimal step size.

\paragraph{Tensor Notation.}
We construct tensors that are sums of outer products of matrices and vectors, i.e., \(\mM = \sum_{k=1}^{\nheads} B_k \otimes v_k\), where \(B_k \in \sR^{d \times d}\) and \(v_k \in \sR^T\).
The product \(X^\top \mM\) denotes \(X^\top \mM = \sum_{k=1}^{\nheads} \langle B_k, X \rangle v_k\) whereas the product \(\mM v\) denotes \(\mM v = \sum_{k=1}^{\nheads} B_k \langle v_k, v\rangle\).
The inner product between two tensors \(\mM = \sum_{k=1}^{\nheads} B_k \otimes v_k\) and \(\mN = \sum_{k=1}^{\nheads} B_k' \otimes v_k'\) is denoted by \(\langle \mM, \mN \rangle = \sum_{k=1}^{\nheads} \langle B_k, B_k' \rangle \langle v_k, v_k' \rangle\).
The Frobenius norm of a tensor \(\mM\) is given by \(\|\mM\|_F = \sqrt{\langle \mM, \mM \rangle}\).

The loss then reduces to a tensor factorization objective.
\begin{restatable}{proposition}{matrix}\label{prop:matrix}
      Under \Cref{assum:data}, the loss in \Cref{eq:loss} satisfies
      \begin{equation*}
            \gL(\theta) = \frac{1}{2} \|\mG - \mP\|_F^2 + \frac{1}{2} \E\left[\|\xi\|^2\right]\,, \quad \text{where} \enspace \mP = \sum_{k=1}^{\nheads} V_k \otimes s_k\,, \enspace \text{and} \enspace \mG = \sum_{k=1}^{\nheads} m_k^\star \left(V_k^\star \otimes s_k^\star\right)\,.
      \end{equation*}
      In particular, the gradient flow of \(\gL\) coincides with that of \(\frac{1}{2} \|\mG - \mP\|_F^2\).
\end{restatable}

\paragraph{Attention Reparameterization.}
Since the softmax is shift-invariant, the gradient flow preserves the mean of \(q_k\), so we can assume \(q_k\) has zero mean without loss of generality.
This implies that there is a one-to-one correspondence between \(q_k\) and \(s_k\) in the subspace of zero-mean vectors.
Therefore, it is possible to analyze the dynamics in terms of \(s_k\) instead of \(q_k\) with the notation \(\Pi(s) = \left(\diag(s) - s s^\top\right)\):
\begin{equation}\label{eq:dynamics}
      \begin{split}
            \dot{V}_k(t) &= \left(\mG - \mP(t)\right) s_k(t)\,, \\
            \dot{s}_k(t) &= \Pi(s_k(t))^2 V_k(t)^\top \left(\mG - \mP(t)\right)\,.
      \end{split}
\end{equation}

\paragraph{Numerical Simulations.}
We simulate these differential equations with initialization \(V_i = 0\) and \(s_i \approx \frac{1}{T} 1_T\).
The results recapitulate the incremental learning behavior observed in \Cref{fig:attention}.
We present the results in \Cref{app:simulations}.

\subsection{Coupled Dynamics Describe the Competitive Phase}
\label{sec:coupled}

We analyze a sufficient idealized regime for the competitive phase: near-symmetric initialization induces coupled dynamics, with \(s_1(0) = s_k(0), V_1(0) = V_k(0)\) for all \(k\).
Once the heads are coupled, they co-evolve, i.e., \(s_k(t) = s(t)\), \(V_k(t) = V(t)\) for all \(k\) and \(t\).

This leads to the following coupled dynamics:
\begin{equation}\label{eq:coupled}
      \begin{split}
            \dot{V}(t) &= \mG s(t) - \nheads \|s(t)\|^2 V(t)\,, \\
            \dot{s}(t) &= \Pi(s(t))^2 \left(V(t)^\top \mG - \nheads \|V(t)\|_F^2 s(t)\right)\,.
      \end{split}
\end{equation}

\begin{restatable}{theorem}{coupling}\label{thm:coupling}
      Assume that the initialization satisfies the following for all \(k \in [T]\):
      \begin{equation}\label{eq:ordering}
            \langle V(0), A_1^\star\rangle \geq \langle V(0), A_k^\star \rangle\,, \enspace \langle s(0), s_1^\star\rangle \geq \langle s(0), s_k^\star \rangle\,.
      \end{equation}
      Then, the dynamics of \(V(t)\) and \(s(t)\) converge to the following fixed point:
      \begin{equation}\label{eq:fixed_point}
            \lim_{t \to \infty} V(t) = \dfrac{m_1^\star}{\nheads} V_1^\star\,, \quad \lim_{t \to \infty} s(t) = s_1^\star\,.
      \end{equation}
\end{restatable}

The proof of \Cref{thm:coupling} combines an ordering argument with a Lyapunov argument.
The ordering condition in \Cref{eq:ordering} is preserved along the flow, since \(\dot{V}(t)\) and \(\dot{s}(t)\) reinforce the same order, and thus the first feature remains dominant in both \(V\) and \(s\).
In parallel, the loss is non-increasing along the flow, so the trajectory converges to a stationary point of the coupled dynamics.
Among the stationary points that respect the ordering, the only one is the fixed point in \Cref{eq:fixed_point}.

On its own, \Cref{thm:coupling} describes an idealized trajectory in which all heads are exactly equal.
\Cref{thm:bounded} connects it to the multi-head dynamics in \Cref{eq:dynamics}: if the heads start close to the symmetric initialization, they shadow the coupled trajectory for a finite time.

\begin{restatable}{theorem}{bounded}\label{thm:bounded}
      Assume that the following holds for \(\epsilon \ll 1\):
      \begin{equation*}
            \forall k \in [\nheads]: \|V(0) - V_k(0)\|_F \leq \epsilon \enspace \text{and} \enspace \|s(0) - s_k(0)\|_2 \leq \epsilon\,.
      \end{equation*}
      Then, there exists a constant \(c_1\), depending only on \(\nheads\), \(T\), \(\mG\), and \(\|V(0)\|_F\), such that \(\forall t \in \left[0, -\frac{\log \epsilon}{2 c_1}\right]\):
      \begin{equation*}
            \|V_k(t) - V(t)\|_F \leq \epsilon e^{c_1 t} \quad \text{and} \quad \|s_k(t) - s(t)\|_2 \leq \epsilon e^{c_1 t}\,.
      \end{equation*}
\end{restatable}

The proof of \Cref{thm:bounded} combines two ingredients.
First, the coupled trajectory is bounded, as \(V\) is pulled toward its optimal value \(\mG s / (\nheads \|s\|^2)\), which we show in \Cref{lem:apriori}.
Second, on a bounded set the multi-head vector field is Lipschitz, so nearby trajectories separate at most exponentially fast, which we show in \Cref{lem:shadowing}.
Therefore, the size of the initial asymmetry sets the coupling time of the heads, after which they might start to diverge.

Lastly, we remark that the initialization in \Cref{thm:coupling} can be further relaxed to a wider basin of attraction around the symmetric initialization of interest.
This follows from a similar argument as in \citet{zucchet2025emergence}, who study the escape time from this initialization when \(\nheads = 1\).

\begin{restatable}{remark}{initialization}\label{rmk:initialization}
      The initialization of interest is \(s_k(0) \approx \frac{1}{T} 1_T\) for all \(k \in [\nheads]\) as seen in \Cref{fig:attention}.
      By expanding the dynamics around this initialization with \(V_k \approx 0\) for all \(k \in [\nheads]\):
      \begin{equation*}
            \dot{V}_k(0) \approx \frac{1}{T} \mG 1_T\,, \quad \dot{s}_k(0) \approx 0\,.            
      \end{equation*}
      Similarly, second-order local approximation shows that \(s_k\) has the largest increase in the direction \(s_1^\star\).
      Therefore, we can quantify a wider basin of attraction for \Cref{thm:coupling} as all \(V_k\) and \(s_k\) move toward the initialization space defined by \Cref{eq:ordering}.
\end{restatable}

\subsection{Cooperation After Competition}
\label{sec:cooperative}

To study the cooperative phase after the initial competitive phase, we consider the dynamics of the loss at various initializations around the fixed point in \Cref{eq:fixed_point}.
Consider the following initialization scheme:
\begin{equation}\label{eq:initialization}
      \begin{split}
            V_1(0) &= \ldots = V_{\nheads-1}(0) \approx \dfrac{m_1^\star}{\nheads} V_1^\star\,, \quad V_{\nheads}(0) \approx \dfrac{m_1^\star}{\nheads} V_1^\star\,, \\
            s_1(0) &= \ldots = s_{\nheads-1}(0) = s_1^\star\,, \quad s_{\nheads}(0) \approx s_1^\star\,.
      \end{split}
\end{equation}
The dynamics of \(s_1, \ldots, s_{\nheads-1}\) remain fixed because the projection term vanishes at \(s_1^\star\), i.e., \(\Pi(s_k) = 0\) for all \(k \in [\nheads-1]\).
In addition, \(V_1, \ldots, V_{\nheads-1}\) are coupled due to the gradient flow.
Therefore, the whole system collapses to three equations: one for \(V\) that describes the ensemble and two for \(V'\) and \(s'\) that describe the offshooting head:
\begin{equation}\label{eq:cooperative}
      \begin{split}
            \dot{V}(t) &= m_1^\star V_1^\star - (\nheads-1) V(t) - \langle s_1^\star, s'(t)\rangle V'(t)\,, \\
            \dot{V}'(t) &= \mG s'(t) - (\nheads-1) \langle s_1^\star, s'(t)\rangle V(t) - \|s'(t)\|^2 V'(t)\,, \\
            \dot{s}'(t) &= \Pi(s'(t))^2 \left(V'(t)^\top \mG - (\nheads-1) \langle V'(t), V(t)\rangle s_1^\star - \|V'(t)\|_F^2 s'(t)\right)\,.
      \end{split}
\end{equation}

As in \Cref{thm:bounded}, \Cref{thm:boundedcoop} connects the cooperative system to the multi-head dynamics in \Cref{eq:dynamics}: if the heads start close to \Cref{eq:initialization}, they shadow the trajectory of \Cref{eq:cooperative} for a finite time.

\begin{restatable}{theorem}{boundedcoop}\label{thm:boundedcoop}
      Assume that the following holds for \(\epsilon \ll 1\):
      \begin{equation*}
            \begin{split}
                  \forall k \in [\nheads-1]: \|V(0) - V_k(0)\|_F \leq \epsilon, \|s_1^\star - s_k(0)\|_2 \leq \epsilon\,,
                  \enspace \text{and} \enspace \|V'(0) - V_{\nheads}(0)\|_F \leq \epsilon\,, \|s'(0) - s_{\nheads}(0)\|_2 \leq \epsilon\,.
            \end{split}
      \end{equation*}
      Let \(\Delta(t)\) be the deviation from the cooperative system in \Cref{eq:cooperative}:
      \begin{equation*}
            \begin{split}
                  \Delta(t) &= \max \Big\{\max_{k \in [\nheads-1]}\{\|V(t)-V_k(t)\|_F, \|s_1^\star-s_k(t)\|_2\}\,, \|V'(t)-V_{\nheads}(t)\|_F, \|s'(t)-s_{\nheads}(t)\|_2\Big\}\,.
            \end{split}
      \end{equation*}
      If \(\delta \coloneqq \inf_{t \geq 0} \|s_1^\star-s'(t)\| > 0\), there exists a constant \(c_1\), depending only on \(\nheads\), \(T\), \(\mG\), \(\delta\), and \(\|\mG - \mP(0)\|_F\), such that:
      \begin{equation*}
            \begin{split}
                  \Delta(t) \leq \epsilon e^{c_1 t} \,, \quad \forall t \in \left[0, -\frac{\log \epsilon}{2 c_1}\right]\,.
            \end{split}
      \end{equation*}
\end{restatable}

The proof of \Cref{thm:boundedcoop} follows that of \Cref{thm:bounded}, and differs only in how the cooperative trajectory is bounded.
Since the loss is non-increasing, the prediction \(\mP\) stays bounded.
This does not bound \(V\) and \(V'\) individually, since their contributions to \(\mP\) along \(s_1^\star\) can cancel.
The assumption \(\delta > 0\) resolves this: outside \(s_1^\star\), only \(V'\) contributes to \(\mP\), with an attention bounded from below in terms of \(\delta\), which bounds \(V'\) and then \(V\), as we show in \Cref{lem:apriori}.
Then, \Cref{lem:shadowing} gives the claim, where the constant \(c_1\) grows as \(\delta\) decreases.

The dynamics in \Cref{eq:cooperative} with the initialization in \Cref{eq:initialization} are nontrivial: while \(V'\) grows in an orthogonal direction \(V_\perp\) to \(V_1^\star\), \(s'\) is still sparse around \(s_1^\star\).
This is due to the fact that \(\Pi(s') \approx 0\) at initialization as \(s'(0) \approx s_1^\star\) which leads to a scale separation between \(\dot{s}'\) and \(\dot{V}'\).
Consequently, when \(V'\) grows along some \(V_\perp\), the prediction is pushed to include the unnecessary term, \(V_\perp x_t\).
However, this is immediately canceled out by the progression of the ensemble, where \(V\) learns to offset this by learning \(-V_\perp\).
That is, the remaining heads become negatively aligned with the new direction.
This alignment is proportional to the attention \(s'\) still puts on \(s_1^\star\), and fades as this attention vanishes.

This collaborative behavior is best seen in \Cref{fig:simulations}.
We also verify that the same negative alignment of the non-offshooting heads' value matrices appears in the minimal transformer trained on the classification task in \Cref{fig:collaboration}, indicating that our theory captures non-trivial structural aspects of the dynamics beyond the existence of stages.
This comparison is qualitative: the two settings differ in loss, parameterization, and timescale, so we compare structural signatures rather than aligned numerical trajectories.

To simplify \Cref{eq:cooperative}, we compare \(V\) with its optimal value \(V^\star\) given \(V'\) and \(s'\), which depends on the attention of \(s'\) on \(s_1^\star\).
The following lemma shows that \(V\) tracks \(V^\star\) with a lag: the initial deviation decays exponentially, and the lag is small when \(V^\star\) moves slowly.
\begin{restatable}{lemma}{fastslow}\label{lem:optimality}
      Let \(\Delta(t) = V(t) - V^\star(t)\) where
      \begin{equation*}
            V^\star(t) = \dfrac{1}{\nheads - 1} \left(m_1^\star V_1^\star - \langle s_1^\star, s'(t)\rangle V'(t)\right).
      \end{equation*}
      If \(\delta \coloneqq \inf_{t \geq 0} \|s'(t) - s_1^\star\| > 0\), there exist constants \(c_1\) and \(c_2\), depending only on \(\nheads\), \(T\), \(\mG\), \(\delta\), and \(\|\mG - \mP(0)\|_F\), such that
      \begin{equation*}
            \|\Delta(t)\|_F \leq e^{-c_2 t} \|\Delta(0)\|_F + \dfrac{c_1(\delta)}{c_2}\,.
      \end{equation*}
\end{restatable}

Inspired by \Cref{lem:optimality} and numerical simulations in \Cref{fig:theory_vs_regression}, we use a two-scale approximation in which \(V\) is optimized faster:
\begin{equation}\label{eq:second_dynamics}
      \begin{split}
            \dot{V}'(t) &= \mG_{(1)} s_{(1)}'(t) - \|s_{(1)}'(t)\|^2 V'(t)\,, \\
            \dot{s}'(t) &= \Pi(s'(t))^2 \left(V'(t)^\top \mG_{(1)} - \|V'(t)\|_F^2 s_{(1)}'(t)\right)\,,
      \end{split}
\end{equation}
where we introduce \(\mG_{(i)} = \mG - \sum_{j=1}^i m_j^\star \left(V_j^\star \otimes s_j^\star\right)\) along with the following notation:
\begin{equation*}
      V_{(i)} = V - \sum_{j=1}^i \langle V_j^\star, V\rangle V_j^\star\,, \quad s_{(i)} = s - \sum_{j=1}^i \langle s_j^\star, s\rangle s_j^\star\,.
\end{equation*}

We show that, under the dynamics in \Cref{eq:second_dynamics}, the second head learns the second positional feature:
\begin{restatable}{theorem}{cooperative}
\label{thm:cooperative}
      Assume that the initialization satisfies the following for all \(k \in [2, \nheads]\) and some \(j \in [3, T]\):
      \begin{equation*}
            \langle V'(0), A_2^\star\rangle \geq \langle V'(0), A_k^\star \rangle\,, \quad \langle s'(0), s_2^\star\rangle \geq \langle s'(0), s_k^\star \rangle\,, \quad \langle s'(0), s_j^\star \rangle > 0\,.
      \end{equation*}
      Further, suppose that \(V'(0), s'(0)\) are such that
      \begin{equation}\label{eq:init_loss}
            \langle V'(0), \mG_{(1)} s_{(1)}'(0)\rangle > \dfrac{1}{2} \|V'(0)\|_F^2 \|s_{(1)}'(0)\|^2\,,
      \end{equation}
      and that \(\inf_{t \geq 0} \|s'(t) - s_1^\star\| > 0\).
      Then, the second head learns the second feature:
      \begin{equation*}
            \lim_{t \to \infty} V'(t) = m_2^\star V_2^\star\,, \quad \lim_{t \to \infty} s'(t) = s_2^\star\,.
      \end{equation*}
\end{restatable}
The proof of \Cref{thm:cooperative} has two steps.
The first step is similar to \Cref{thm:coupling}.
The ordering conditions are preserved along the flow, and the loss is non-increasing, so the trajectory converges to stationary points that respect the ordering.
Unlike in \Cref{thm:coupling}, these stationary points form a family, in which the second head splits its attention between the first two features and all members give the same prediction, as we show in \Cref{prop:family}.
The second step, which we show in \Cref{prop:attention}, is that the attention moves entirely to the second feature.
Here, the attention on the first feature decreases at a rate controlled by the attention on the remaining positions.
This is why we assume that \(s'(0)\) puts some attention outside the first two features: otherwise, this attention stays zero for all \(t\), and this mechanism is absent.
Both steps also use the assumption that \(s'(t)\) stays away from \(s_1^\star\), that is, the second head does not return to the first feature.
This keeps the trajectory bounded, which the convergence arguments in both steps rely on.

Notably, we require the initialization to satisfy \Cref{eq:init_loss}.
This ensures that the dynamics start with an initial decrease on the loss beyond the first saddle point characterized in \Cref{thm:coupling}: \Cref{thm:cooperative} proves that a potential characterizing the loss is monotonically minimized and this saddle is avoided.
In \Cref{rem:init_loss}, we discuss how a small perturbation toward the second positional feature, larger in the value than in the attention, together with an even smaller perturbation of the attention toward the remaining positions, satisfies the initialization conditions of \Cref{thm:cooperative}.

Finally, in \Cref{app:higher}, we extend the same argument to the specialization of an arbitrary head after the system has acquired the earlier features.
In later phases, the previously specialized heads stay fixed at their learned features, while the remaining coupled ensemble still represents the dominant feature and one free head offshoots toward the next feature.
Together with \Cref{thm:bounded,thm:boundedcoop}, the same shadowing argument shows that each phase remains close to the corresponding idealized dynamics for a finite, controlled time.

\section{Related Work and Discussion}

\paragraph{Incremental learning.} 
Plateau-like learning curves are a common feature in neural network training.
Early analyses, such as \citet{fukumizu2000local}, attribute these behaviors to critical points in supervised learning.
Subsequent studies have examined similar dynamics in a variety of simplified settings, including linear networks~\citep{gissin2020the,saxe2019mathematical,gidel2019implicit,arora2019implicit,jacot2021saddle,li2021towards,razin2021implicit,jiang2023algorithmic,berthier2023incremental,pesme2023saddle,jin2023understanding,varre2023on,varre2024sgd} and ReLU models~\citep{boursier2022gradient,abbe2023sgd}; recent work argues for their universality~\citep{ziyin2025parameter, kunin2025alternating, zhang2026saddle}.
In transformer training, plateaus followed by sudden capability gains \citep{chen2024sudden} are often observed in regression tasks~\citep{garg2022can,von2023transformers,ahn2023transformers} and formal language recognition~\citep{bhattamishra2024understanding,akyurek2024context,dangelo2025selective,cagnetta2025learning,dangelo2026transformers}.
Finally, \citet{cagnetta2024towards,cagnetta2025learning} show dataset size affects the order of the learned hierarchy in random probabilistic context-free grammars, mirroring our data-dependent stage progression and its connection to effective model complexity.

\paragraph{\(n\)-gram models.}
\(n\)-gram language models~\citep{jurafsky2026speech} serve as a toy setting to understand language models.
This perspective has motivated studies of optimization landscapes~\citep{makkuva2025attention}, expressivity over \(n\)-gram distributions~\citep{svete2024transformers}, and sample complexity~\citep{yuksel2025on,yuksel2025generalization}.
The learning of variable-order \(n\)-grams has been studied by \citet{zhou2026information}, while \citet{deora2025incontext} consider \(n\)-grams of different fixed orders \(n\).
Connections between in-context learning and the emergence of induction heads~\citep{elhage2021mathematical,olsson2022context}, together with their acquisition via gradient descent~\citep{nichani2024transformers}, are drawn by \citet{bietti2023birth}.
Training dynamics on \(n\)-gram tasks have also been shown to progress in stages: intermediate solutions approximate sub-\(n\)-grams \citep{edelman2024evolution}, which are formalized as near-stationary points by \citet{varre2025learning}.
Despite this rich phenomenology, standard \(n\)-gram models lack the hierarchical abstractions characteristic of natural language \citep{wu2022learning,wu2025building}.
Our task adds an importance hierarchy over positional features, while leaving richer forms of abstraction and composition outside its scope.

\paragraph{Attention dynamics.}
The dynamics of attention have recently been explored through various simplified tasks \citep{snell2021approximating,li2023transformers,yang2024training,huang2024incontext,goel2026training}, including tasks exhibiting sparsity \citep{jelassi2022vision,tian2023scan,wang2024transformers}.
Particularly relevant to our task, \citet{marion2025attention} study single-location regression, in which attention concentrates on a randomly selected informative position, a setting related to sequence multi-index models \citep{cui2024phase,troiani2025fundamental}.
Tractable analysis usually requires simplified attention parameterizations: diagonal \citep{abbe2023transformers}, position-only softmax \citep{jelassi2022vision}, position-free softmax \citep{tian2023scan}, component-wise \citep{marion2025attention}, or linear \citep{zhang2024trained,shen2025on} attention.
Closest to our setting, \citet{zucchet2025emergence} consider the same data model with position-only softmax and \(\nheads=1\), characterizing escape from the uniform-attention initialization via a local Taylor approximation.
The multi-head setting has been comparatively less developed.
\citet{chen2024training} and \citet{zhang2025training} study in-context linear regression with multi-head softmax and linear attention, respectively.
\citet{chen2024unveiling} analyze the training dynamics of softmax attention on in-context \(n\)-gram tasks, and \citet{varre2025incremental} study incremental learning on a related associative recall task.

\paragraph{Discussion.}
We study the saddle-to-saddle dynamics of multi-head softmax attention in a positional setting with an importance hierarchy over past positions.
This setting reveals a competitive-to-cooperative transition: heads first compete for a dominant sparse pattern, then specialize sequentially via offshooting, with non-offshooting heads compensating through their value matrices.
These structural signatures go beyond the mere presence of staged learning or head specialization.
Capturing these signatures requires jointly analyzing value and attention parameter dynamics in a multi-head system; they would be obscured by freezing parameter blocks to impose stage-wise training or by reductions that decouple head trajectories, e.g., via initialization or parameterization.
This competition is also distinct from circuit-level competition between algorithmic solutions \citep{reddy2024mechanistic,park2025competition}: here, the competing units are symmetric heads sharing the same input and output domain.
Our setting thus shows how the head specialization observed in trained transformers \citep{voita2019analyzing,clark2019what,wang2023interpretability} can \emph{emerge} through training dynamics, and gives a controlled example of the head-level interactions that \citet{chakrabarti2026multi} frame as a multi-player game.

\section{Conclusion}

In this work, we introduce a simple but rich task in which transformers need to implement multiple sparse attention patterns.
We show that transformers learn this task incrementally, with each stage acquiring a new sparse attention pattern in order of statistical importance.
This staged behavior is driven by a competitive-to-cooperative transition in the head dynamics, where all heads first focus on the most important pattern and later specialize via offshooting.
To analyze this mechanism, we use a simplified architecture and regression objective, deriving surrogate differential equations that capture the same qualitative stage-wise dynamics.
Together, these results provide a controlled account of how sparse attention patterns and head specialization can emerge through training dynamics.
Extending this analysis to more realistic architectures, objectives, and data distributions remains an important direction for future work.

\section*{Acknowledgments}

This work was partially funded by an unrestricted gift from Coefficient Giving, and the grant number 212111 from the Swiss National Science Foundation.
O{\u{g}}uz Kaan Y{\"u}ksel is supported by the SwissAI Fellowship.

\bibliography{main}
\bibliographystyle{tmlr}

\clearpage
\appendix
\begin{appendices}
    \section*{Organization of the Appendix}
\label{app:organization}
The appendix is organized as follows.
\begin{itemize}
      \item \Cref{app:identifiability} gives examples of the non-identifiability of the task.
      \item \Cref{app:proofs} provides proofs of the theoretical results.
      \item \Cref{app:initialization} discusses the initialization conditions of our main theorems and how they can be relaxed.
      \item \Cref{app:simulations} simulates the regression flow, verifies the theoretical predictions against it, and compares the resulting trajectories with the minimal transformer.
      \item \Cref{app:exp} provides the experimental details for the experiments in \Cref{sec:incremental,sec:minimal}.
      \item \Cref{app:additional_experiments} presents additional experiments and ablations of the training setup and the task.
\end{itemize}

\section{Identifiability}\label{app:identifiability}

The data generation process defined in \Cref{eq:markov_chain} is not identifiable: different choices of partitions \(I(k)\) and feature matrices \(A_k^\star\) can induce the same conditional distribution.
A trivial source of non-identifiability is permutation: jointly permuting the pairs \(\{(I(k), A_k^\star)\}_{k=1}^{\nheads}\) leaves the sum in \Cref{eq:markov_chain} unchanged.

Beyond permutations, there exist non-trivial decompositions in which the index sets \(I(k)\) overlap and the attention weights \(\alpha_i\) spread across
multiple positions.
Consider \(\nheads = 2\) with the sparse, disjoint decomposition
\[
I(1) = \{1\}\,, \quad I(2) = \{2\}\,, \quad \alpha_1^{(1)} = \alpha_2^{(2)} = 1\,,
\]
and feature matrices \((A_1^\star, A_2^\star)\). The resulting logits are
\[
A_1^\star x_{t-1} + A_2^\star x_{t-2}\,.
\]
Now take the alternative decomposition with overlapping partitions
\[
I(1) = I(2) = \{1, 2\}\,, \quad \alpha^{(1)} = (2/3, 1/3), \quad \alpha^{(2)} = (1/3, 2/3),
\]
and modified feature matrices
\[
B_1^\star = 2 A_1^\star - A_2^\star\,, \quad B_2^\star = -A_1^\star + 2 A_2^\star\,.
\]
The resulting logits are
\[
B_1^\star \left(\tfrac{2}{3} x_{t-1} + \tfrac{1}{3} x_{t-2}\right) + B_2^\star \left(\tfrac{1}{3} x_{t-1} + \tfrac{2}{3} x_{t-2}\right) = A_1^\star x_{t-1} + A_2^\star
x_{t-2}\,,
\]
matching the original distribution. In this alternative representation, each head attends to both positions, and the feature matrices are linear combinations
of the original. More generally, any invertible row-stochastic matrix \(
      \begin{pmatrix} a & 1-a \\ b & 1-b \end{pmatrix}\,,
\) with \(a \ne b\) yields an equivalent representation, and analogous constructions exist for \(\nheads > 2\).
As shown in \Cref{sec:incremental}, training consistently recovers the sparse, disjoint decomposition despite these non-sparse representations.

\section{Deferred Proofs}
\label{app:proofs}

We start with the proof of \Cref{prop:matrix} and some elementary results on the operation \(\Pi\).
Recall that we assume \(s_i^\star\) are one-hot in \Cref{sec:regression}.
That is, in the sequel, \(\|s_i^\star\|^2 = 1\).

\matrix*
\begin{proof}
      For any \(B, B' \in \sR^{d \times d}\) and \(v, v' \in \sR^T\), \Cref{assum:data} gives
      \begin{equation*}
            \E\left[\langle B X v, B' X v' \rangle\right] = \sum_{i=1}^T \sum_{j=1}^T v_i v'_j \Tr\left(B^\top B' \E\left[x_j x_i^\top\right]\right) = \langle B, B' \rangle \langle v, v' \rangle\,.
      \end{equation*}
      Expanding the square and applying this identity to each pair of terms gives, for any \(B_k\) and \(v_k\),
      \begin{equation*}
            \E\left[\Big\|\sum_{k} B_k X v_k\Big\|^2\right] = \sum_{k, l} \langle B_k, B_l \rangle \langle v_k, v_l \rangle = \Big\|\sum_{k} B_k \otimes v_k\Big\|_F^2\,.
      \end{equation*}
      Since \(A_k^\star = m_k^\star V_k^\star\), we have \(y_\theta(X) - y^\star(X, \xi) = \sum_{k} V_k X s_k - \sum_{k} m_k^\star V_k^\star X s_k^\star - \xi\).
      As \(\xi\) is independent of \(X\) and \(\E[\xi] = 0\), the cross terms involving \(\xi\) vanish, and therefore
      \begin{equation*}
            \gL(\theta) = \frac{1}{2} \|\mP - \mG\|_F^2 + \frac{1}{2} \E\left[\|\xi\|^2\right]\,.
      \end{equation*}
      The second term does not depend on \(\theta\), so the two objectives have the same gradient flow.
\end{proof}

\begin{lemma}\label{lem:kernel}
      Let \(s\) be a vector with non-negative entries, and let \(\Pi(s) = \mathrm{diag}(s) - ss^\top\).
      Then, \(u \in \ker(\Pi(s))\) if and only if
      \begin{equation*}
            u_j = \langle s, u\rangle \quad \text{for all } j \text{ with } s_j > 0\,.
      \end{equation*}
      In particular, if \(\|s\|_1 = 1\), then
      \begin{equation*}
            \ker\left(\Pi(s)\right) = \mathrm{span}\left(\left\{e_j : s_j = 0\right\} \cup \left\{\sum_{j: s_j > 0} e_j\right\}\right)\,,
      \end{equation*}
      and if \(\|s\|_1 < 1\), then \(\ker\left(\Pi(s)\right) = \mathrm{span}\left(\left\{e_j : s_j = 0\right\}\right)\).
\end{lemma}
\begin{proof}
      For any \(u\), we have
      \begin{equation*}
            \Pi(s) u = s \odot u - \langle s, u\rangle s = s \odot \left(u - \langle s, u\rangle 1_T\right)\,.
      \end{equation*}
      Thus, the \(j\)-th entry of \(\Pi(s) u\) is \(s_j (u_j - \langle s, u\rangle)\).
      This entry vanishes if and only if \(s_j = 0\) or \(u_j = \langle s, u\rangle\).
      Therefore, \(\Pi(s) u = 0\) if and only if \(u_j = \langle s, u\rangle\) for all \(j\) with \(s_j > 0\), which is the first claim.

      For the second claim, let \(u \in \ker(\Pi(s))\) and write \(c = \langle s, u\rangle\).
      By the first claim, \(u\) is equal to \(c\) on the support of \(s\).
      Since \(s_j = 0\) outside the support, we get
      \begin{equation*}
            c = \langle s, u\rangle = \sum_{j: s_j > 0} s_j u_j = c \sum_{j: s_j > 0} s_j = c \|s\|_1\,.
      \end{equation*}
      If \(\|s\|_1 = 1\), this identity holds for every \(c\).
      Hence, \(u\) is any vector that is constant on the support of \(s\) and arbitrary outside it, which gives the first span.
      If \(\|s\|_1 < 1\), the identity forces \(c = 0\).
      Hence, \(u\) vanishes on the support of \(s\) and is arbitrary outside it, which gives the second span.
\end{proof}

\subsection{Boundedness and Shadowing}

In this section, we first show in \Cref{lem:apriori} that the trajectories of the reference systems in \Cref{eq:coupled,eq:cooperative,eq:second_dynamics} remain bounded, and in \Cref{lem:shadowing} that solutions starting close to a bounded solution stay close for a finite time.
We then prove \Cref{thm:bounded,thm:boundedcoop}, establishing that the multi-head dynamics initialized close to the corresponding symmetric initializations shadow these trajectories.

\begin{lemma}\label{lem:apriori}
      The trajectories of the reference systems remain bounded.
      \begin{enumerate}[label=(\roman*)]
            \item Let \((V(t), s(t))\) be the solution of \Cref{eq:coupled}. Then, for all \(t \geq 0\),
            \begin{equation*}
                  \|V(t)\|_F \leq \max\left\{\|V(0)\|_F, \frac{m_1^\star \sqrt{T}}{\nheads}\right\}\,.
            \end{equation*}
            \item Let \((V(t), V'(t), s'(t))\) be the solution of \Cref{eq:cooperative}, and let \(\mP = (\nheads-1) V \otimes s_1^\star + V' \otimes s'\) and \(B \coloneqq \|\mG\|_F + \|\mG - \mP(0)\|_F\). Then, \(\|\mP(t)\|_F \leq B\) for all \(t \geq 0\). If, moreover, \(\delta \coloneqq \inf_{t \geq 0} \|s'(t) - s_1^\star\| > 0\), then for all \(t \geq 0\),
            \begin{equation*}
                  \|V'(t)\|_F \leq \dfrac{B \sqrt{T}}{\delta}\,, \quad \|V(t)\|_F \leq \dfrac{B}{\nheads-1} \left(1 + \dfrac{\sqrt{T}}{\delta}\right)\,.
            \end{equation*}
            \item Let \((V'(t), s'(t))\) be the solution of \Cref{eq:second_dynamics}. If \(\delta \coloneqq \inf_{t \geq 0} \|s'(t) - s_1^\star\| > 0\), then for all \(t \geq 0\),
            \begin{equation*}
                  \|s_{(1)}'(t)\| \geq \dfrac{\delta}{\sqrt{T}}\,, \quad \|V'(t)\|_F \leq \max\left\{\|V'(0)\|_F, \frac{m_2^\star \sqrt{T}}{\delta}\right\}\,.
            \end{equation*}
      \end{enumerate}
\end{lemma}
\begin{proof}
      For (i), since \(s\) lies on the simplex, \(1/\sqrt{T} \leq \|s\| \leq 1\).
      By \Cref{eq:coupled},
      \begin{equation*}
            \dfrac{d}{dt} \|V\|_F^2 = 2 \langle V, \mG s\rangle - 2 \nheads \|s\|^2 \|V\|_F^2 \leq 2 \|V\|_F \left(\|\mG s\|_F - \nheads \|s\|^2 \|V\|_F\right)\,,
      \end{equation*}
      which is negative whenever \(\|V\|_F > \|\mG s\|_F / (\nheads \|s\|^2)\), i.e., whenever \(V\) is larger in norm than its optimal value \(\mG s / (\nheads \|s\|^2)\) for the current \(s\).
      Since \(\|\mG s\|_F \leq m_1^\star \|s\|\) and \(\|s\| \geq 1/\sqrt{T}\), this threshold is at most \(m_1^\star \sqrt{T} / \nheads\), which gives the claim.

      For (ii), consider
      \begin{equation*}
            \begin{split}
                  \phi(V, V', s') &= (\nheads-1)m_1^\star \langle V, V_1^\star\rangle - \dfrac{(\nheads-1)^2}{2} \|V\|_F^2 \\
                  &- (\nheads-1) \langle s_1^\star, s'\rangle \langle V, V' \rangle + \langle V', \mG s'\rangle - \frac{1}{2} \|s'\|^2 \|V'\|_F^2\,,
            \end{split}
      \end{equation*}
      which satisfies \(\nabla_V \phi = (\nheads-1) \dot{V}\), \(\nabla_{V'} \phi = \dot{V}'\), and \(\dot{s}' = \Pi(s')^2 \nabla_{s'} \phi\). Hence,
      \begin{equation*}
            \dot{\phi} = (\nheads-1) \|\dot{V}\|_F^2 + \|\dot{V}'\|_F^2 + \|\Pi(s') \nabla_{s'} \phi\|^2 \geq 0\,.
      \end{equation*}
      Since \(\phi = \frac{1}{2}\|\mG\|_F^2 - \frac{1}{2}\|\mG - \mP\|_F^2\), the distance \(\|\mG - \mP(t)\|_F\) is non-increasing, and by the triangle inequality, \(\|\mP(t)\|_F \leq \|\mG\|_F + \|\mG - \mP(0)\|_F = B\).

      To bound \(V\) and \(V'\) individually, note that since \(s_{(1)}' \perp s_1^\star\), the decomposition
      \begin{equation*}
            \mP = \left((\nheads-1) V + \langle s_1^\star, s'\rangle V'\right) \otimes s_1^\star + V' \otimes s_{(1)}'
      \end{equation*}
      is orthogonal, which gives \(\|V'\|_F \|s_{(1)}'\| \leq B\) and \((\nheads-1)\|V\|_F \leq B + \|V'\|_F\).
      Moreover, \(\|s' - s_1^\star\|^2 = (1 - \langle s_1^\star, s'\rangle)^2 + \|s_{(1)}'\|^2\) and \(1 - \langle s_1^\star, s'\rangle = \langle 1_T, s_{(1)}'\rangle \leq \sqrt{T-1} \|s_{(1)}'\|\), as \(s_{(1)}'\) has at most \(T-1\) non-zero entries, so \(\|s_{(1)}'\| \geq \delta / \sqrt{T}\), which gives the claim.

      For (iii), the same argument as in (ii) gives \(\|s_{(1)}'(t)\| \geq \delta / \sqrt{T}\).
      By \Cref{eq:second_dynamics}, as in (i),
      \begin{equation*}
            \dfrac{d}{dt} \|V'\|_F^2 = 2 \langle V', \mG_{(1)} s_{(1)}'\rangle - 2 \|s_{(1)}'\|^2 \|V'\|_F^2 \leq 2 \|V'\|_F \left(\|\mG_{(1)} s_{(1)}'\|_F - \|s_{(1)}'\|^2 \|V'\|_F\right)\,,
      \end{equation*}
      which is negative whenever \(\|V'\|_F > \|\mG_{(1)} s_{(1)}'\|_F / \|s_{(1)}'\|^2\).
      Since \(\|\mG_{(1)} s_{(1)}'\|_F \leq m_2^\star \|s_{(1)}'\|\) and \(\|s_{(1)}'\| \geq \delta / \sqrt{T}\), this threshold is at most \(m_2^\star \sqrt{T} / \delta\), which gives the claim.
\end{proof}

\begin{lemma}\label{lem:shadowing}
      For a state \(x = (V_1, \ldots, V_{\nheads}, s_1, \ldots, s_{\nheads})\) of \Cref{eq:dynamics}, let \(\|x\| = \max_{k \in [\nheads]} \max\{\|V_k\|_F, \|s_k\|\}\).
      Let \(\bar{x}(t) = (\bar{V}_1(t), \ldots, \bar{V}_{\nheads}(t), \bar{s}_1(t), \ldots, \bar{s}_{\nheads}(t))\) be a solution of \Cref{eq:dynamics} with \(\|\bar{V}_k(t)\|_F \leq R\) for all \(k \in [\nheads]\) and \(t \geq 0\).
      Then, there exists a constant \(c_1 = \gO((R+1)^2)\), depending only on \(\nheads\), \(T\), and \(\mG\), such that any solution \(x(t)\) of \Cref{eq:dynamics} with each \(s_k(0)\) on the simplex and \(\|x(0) - \bar{x}(0)\| \leq \epsilon \leq 1\) satisfies
      \begin{equation*}
            \|x(t) - \bar{x}(t)\| \leq \epsilon e^{c_1 t}\,, \quad \forall t \in \left[0, -\frac{\log \epsilon}{2 c_1}\right]\,.
      \end{equation*}
\end{lemma}
\begin{proof}
      Let \(K\) be the set of states with \(\|V_k\|_F \leq R + 1\) and \(s_k\) on the simplex for all \(k \in [\nheads]\), and let \(F\) denote the vector field of \Cref{eq:dynamics}.

      \(F\) is a polynomial of degree at most two in \(V_1, \ldots, V_{\nheads}\), and \(\|\Pi(s)\| \leq 1\) on the simplex.
      Hence, \(F\) is Lipschitz on \(K\) with respect to \(\|\cdot\|\), with a constant \(c_1 = \gO((R+1)^2)\).
      The leading term comes from \(\Pi(s_k)^2 V_k^\top \mP\), whose derivative with respect to \(s_j\) contains \(\langle V_k, V_j \rangle\).

      Let \(e(t) = x(t) - \bar{x}(t)\) and \(\Delta(t) = \|e(t)\|\).
      Since \(x\) and \(\bar{x}\) solve the same equation,
      \begin{equation*}
            e(t) = e(0) + \int_0^t \left(F(x(\tau)) - F(\bar{x}(\tau))\right) d\tau\,.
      \end{equation*}
      Let \(t^\star = \sup\{t \geq 0 : \Delta(\tau) \leq 1 \text{ for all } \tau \in [0, t]\}\).
      For \(t < t^\star\), \(\bar{x}(t) \in K\) by assumption, and \(x(t) \in K\) since \(\|V_k(t) - \bar{V}_k(t)\|_F \leq \Delta(t) \leq 1\) and \(s_k(t)\) lies on the simplex.
      Taking norms and using the Lipschitz bound,
      \begin{equation*}
            \Delta(t) \leq \Delta(0) + c_1 \int_0^t \Delta(\tau)\, d\tau\,, \quad \forall t < t^\star\,.
      \end{equation*}
      By Grönwall's inequality, \(\Delta(t) \leq \Delta(0) e^{c_1 t} \leq \epsilon e^{c_1 t}\) for all \(t < t^\star\).

      For \(t \leq -\log \epsilon / (2 c_1)\), we have \(\epsilon e^{c_1 t} \leq \sqrt{\epsilon}\).
      If \(\epsilon < 1\), then \(\sqrt{\epsilon} < 1\), so by continuity of \(\Delta\), \(t^\star > -\log \epsilon / (2 c_1)\), and the bound holds on the whole window.
      If \(\epsilon = 1\), the window is \(\{0\}\) and the claim is trivial.
\end{proof}

\bounded*
\begin{proof}
      The state with \(V_k = V(t)\) and \(s_k = s(t)\) for all \(k \in [\nheads]\) is a solution of \Cref{eq:dynamics}, since for equal heads \(\mP = \nheads V \otimes s\) and \Cref{eq:dynamics} reduces to \Cref{eq:coupled}.
      By \Cref{lem:apriori}~(i), \(\|V(t)\|_F \leq R = \max\{\|V(0)\|_F, m_1^\star \sqrt{T} / \nheads\}\).
      The claim follows from \Cref{lem:shadowing}, with \(c_1\) depending only on \(\nheads\), \(T\), \(\mG\), and \(\|V(0)\|_F\).
\end{proof}

\boundedcoop*
\begin{proof}
      The state with \(V_k = V(t)\) and \(s_k = s_1^\star\) for \(k \in [\nheads-1]\), \(V_{\nheads} = V'(t)\), and \(s_{\nheads} = s'(t)\) is a solution of \Cref{eq:dynamics}, since on this state \Cref{eq:dynamics} reduces to \Cref{eq:cooperative}, using \(\Pi(s_1^\star) = 0\).
      By \Cref{lem:apriori}~(ii), \(\|V(t)\|_F\) and \(\|V'(t)\|_F\) are at most \(R = B (1 + \sqrt{T} / \delta)\).
      Since \(\delta \leq \|s'(0) - s_1^\star\| \leq \sqrt{2}\), we have \(R = \gO(B \delta^{-1})\), and \Cref{lem:shadowing} gives the claim with \(c_1 = \gO(B^2 \delta^{-2})\).
\end{proof}

\subsection{Competitive Phase}

In this section, we prove the main result of \Cref{sec:coupled}.

\coupling*
\begin{proof}
      The vector field of \Cref{eq:coupled} is polynomial in \((V, s)\).
      Hence, it is locally Lipschitz, and the solution is unique.
      Moreover, \(V(t)\) is bounded by \Cref{lem:apriori}~(i), and \(s(t)\) lies on the simplex.
      Thus, the solution is bounded and exists for all \(t \geq 0\).

      Let \(\gR\) be the following set:
      \begin{equation*}
            \gR = \left\{ (V, s) \mid s \text{ on the simplex}, \forall k \in [T], \langle V, A_1^\star - A_k^\star\rangle \geq 0\,, \langle s, s_1^\star - s_k^\star\rangle \geq 0\right\}\,.
      \end{equation*}
      We prove that the flow is forward-invariant on \(\gR\).

      Fix any \(j \in [T]\).
      Let \(w_j = \langle V, A_1^\star - A_j^\star \rangle\), \(z_j = \langle s, s_1^\star - s_j^\star \rangle\), \(r_j = \langle s \odot s, s_1^\star - s_j^\star \rangle\), \(t_j = \langle s \odot s \odot s, s_1^\star - s_j^\star \rangle\).
      The flow of \(w_j\) and \(z_j\) are as follows:
      \begin{equation*}
            \begin{split}
                  \dot{w}_j &= (m_1^\star)^2 \langle s, s_1^\star\rangle - (m_j^\star)^2 \langle s, s_j^\star\rangle - \nheads \|s\|^2 w_j\,, \\
                  \dot{z}_j &= (s_1^\star - s_j^\star)^\top \Pi(s)^2 \left(V^\top \mG - \nheads \|V\|_F^2 s\right)\,.
            \end{split}
      \end{equation*}
      Rewriting \(\dot{z}_j\):
      \begin{equation*}
            \begin{split}
                  \dot{z}_j &= \left((s_1^\star - s_j^\star)^\top \diag(s) - z_j s^\top\right) \Pi(s) \left(V^\top \mG - \nheads \|V\|_F^2 s\right) \\
                  &= (s_1^\star - s_j^\star)^\top \diag(s)^2 \left(V^\top \mG - \nheads \|V\|_F^2 s\right) - z_j s^\top \diag(s) \left(V^\top \mG - \nheads \|V\|_F^2 s\right) \\
                  &\quad\quad + \left(\|s\|^2 z_j - r_j\right) \left(V^\top \mG s - \nheads \|V\|_F^2 \|s\|^2\right) \\
                  &= \langle s_1^\star, s\rangle^2 \langle V, A_1^\star \rangle - \langle s_j^\star, s\rangle^2 \langle V, A_j^\star \rangle - \nheads \|V\|_F^2 t_j \\
                  &\quad\quad - z_j s^\top \diag(s) \left(V^\top \mG - \nheads \|V\|_F^2 s\right) + \left(\|s\|^2 z_j - r_j \right) \left(V^\top \mG s - \nheads \|V\|_F^2 \|s\|^2\right)\,.
            \end{split}
      \end{equation*}

      On the boundary of \(\gR\), we have \(w_j = 0\) or \(z_j = 0\).
      If \(w_j = 0\), then \(\dot{w}_j \geq 0\) as \(\langle s, s_1^\star\rangle \geq \langle s, s_j^\star\rangle \geq 0\) and \(m_1^\star > m_j^\star \geq 0\).
      If \(z_j = 0\), then \(r_j = t_j = 0\), and we get \(\dot{z}_j = \langle s, s_1^\star\rangle^2 w_j \geq 0\).

      Since \(\gR\) is defined by finitely many linear equalities and inequalities, by Nagumo's theorem it suffices to check that \(\dot{w}_j \geq 0\) whenever \(w_j = 0\), \(\dot{z}_j \geq 0\) whenever \(z_j = 0\), and \(\dot{s}_i \geq 0\) whenever \(s_i = 0\), and that the flow preserves \(\sum_i s_i = 1\).
      The first two conditions are shown above, and the last two hold since \(\dot{s}_i\) is proportional to \(s_i\) and \(\dot{s}\) is orthogonal to \(1_T\).
      Therefore, a flow that has started in \(\gR\) will remain in \(\gR\) for all time.

      Now, consider the following Lyapunov function:
      \begin{equation}\label{eq:lyapunov}
            \phi(V, s) = \langle V, \mG s\rangle - \dfrac{\nheads}{2} \|V\|_F^2 \|s\|^2\,.
      \end{equation}
      The derivative of \(\phi(V, s)\) is as follows:
      \begin{equation*}
            \begin{split}
                  \nabla_V \phi(V, s) &= \mG s - \nheads \|s\|^2 V\,, \\
                  \nabla_s \phi(V, s) &= V^\top \mG - \nheads \|V\|_F^2 s\,.                  
            \end{split}
      \end{equation*}
      Therefore, the time derivative of \(\phi\):
      \begin{equation*}
            \dot{\phi}(V, s) = \|\dot{V}\|_F^2 + \|\Pi(s) \nabla_s \phi(V, s)\|^2\, \geq 0\,.
      \end{equation*}
      
      Since the solution is bounded, its \(\omega\)-limit set \(\Omega\) is non-empty, compact, and invariant, and \((V(t), s(t))\) converges to \(\Omega\).
      By LaSalle's invariance principle, \(\Omega\) is contained in the largest invariant subset of \(\{\dot{\phi} = 0\}\).
      Since \(\dot{\phi} = \|\dot{V}\|_F^2 + \|\Pi(s) \nabla_s \phi\|^2\), \(\dot{\phi} = 0\) implies \(\dot{V} = 0\) and \(\Pi(s) \nabla_s \phi = 0\), hence \(\dot{s} = \Pi(s)^2 \nabla_s \phi = 0\).
      Thus, every point of \(\{\dot{\phi} = 0\}\) is a stationary point of \Cref{eq:coupled}, and so is every point of \(\Omega\).
      Since \(\gR\) is closed and forward-invariant, \(\Omega \subseteq \gR\).

      Let \((V_\infty, s_\infty) \in \Omega\).
      Then, it satisfies
      \begin{equation*}
            \mG s_\infty - \nheads \|s_\infty\|^2 V_\infty = 0\,, \quad V_\infty^\top \mG - \nheads \|V_\infty\|_F^2 s_\infty \in \mathrm{ker}(\Pi(s_\infty))\,.
      \end{equation*}
      Note that we have the following equality:
      \begin{equation*}
            \begin{split}
                  (\mG s_\infty)^\top \mG &= \sum_{j=1}^{\nheads} m_j^\star \left\langle V_j^\star, \sum_{k=1}^{\nheads} m_k^\star V_k^\star \langle s_k^\star, s_\infty\rangle \right\rangle s_j^\star = \sum_{j=1}^{\nheads} (m_j^\star)^2 \langle s_j^\star, s_\infty\rangle s_j^\star\,.
            \end{split}
      \end{equation*}
      Then, substituting \(V_\infty = \mG s_\infty / (\nheads \|s_\infty\|^2)\) into the second condition and multiplying by \(\nheads \|s_\infty\|^2\), the stationary point \((V_\infty, s_\infty)\) satisfies
      \begin{equation}\label{eq:kernel_condition}
            u \coloneqq \sum_{j=1}^{\nheads} (m_j^\star)^2 \langle s_j^\star, s_\infty\rangle s_j^\star - \lambda \, s_\infty \in \mathrm{ker}(\Pi(s_\infty))\,, \quad \text{where} \quad \lambda = \nheads^2 \|s_\infty\|^2 \|V_\infty\|_F^2 = \dfrac{\|\mG s_\infty\|_F^2}{\|s_\infty\|^2}\,.
      \end{equation}

      For \(i \in [T]\), let \(\mu_i = m_k^\star\) where \(k \in [T]\) is such that \(s_k^\star = e_i\), so that \(u_i = (\mu_i^2 - \lambda) (s_\infty)_i\).
      By \Cref{lem:kernel}, \(u_i = c\) for all \(i\) with \((s_\infty)_i > 0\), where
      \begin{equation*}
            c = \langle s_\infty, u \rangle = \sum_{j=1}^{\nheads} (m_j^\star)^2 \langle s_j^\star, s_\infty\rangle^2 - \lambda \|s_\infty\|^2 = \|\mG s_\infty\|_F^2 - \|\mG s_\infty\|_F^2 = 0\,.
      \end{equation*}
      Hence, \(\mu_i^2 = \lambda\) for every \(i\) in the support of \(s_\infty\).
      Since \((V_\infty, s_\infty) \in \gR\), we have \(\langle s_\infty, s_1^\star\rangle \geq \langle s_\infty, s_k^\star\rangle\) for all \(k \in [T]\), i.e., the entry of \(s_\infty\) at the position of \(s_1^\star\) is its largest entry, and in particular it is positive.
      Thus, this position is in the support of \(s_\infty\), and \(\lambda = (m_1^\star)^2\).
      By \Cref{assum:distinct}, no other position satisfies \(\mu_i^2 = (m_1^\star)^2\), so the support of \(s_\infty\) consists of this single position, and \(s_\infty = s_1^\star\) and \(V_\infty = \mG s_1^\star / \nheads = \frac{m_1^\star}{\nheads} V_1^\star\).
      Therefore, \(\Omega = \{(\frac{m_1^\star}{\nheads} V_1^\star, s_1^\star)\}\), and the flow converges to this point.
\end{proof}

\subsection{Cooperation Phase}
\label{app:cooperation}

In this section, we prove the remaining results in \Cref{sec:cooperative}.
First, we show convergence of the second head starting with the system in \Cref{eq:cooperative}.
We then extend the analysis to arbitrary phases of the dynamics.

\subsubsection{Convergence of the Second Head}
\label{app:second}

Following \Cref{eq:initialization}, we consider the following initialization scheme:
\begin{equation}\label{eq:fast_init}
      V(0) = \dfrac{1}{\nheads-1} \left(m_1^\star V_1^\star - \langle s_1^\star, s'(0)\rangle V'(0)\right)\,, \quad V'(0) \approx V(0)\,, \quad s'(0) \approx s_1^\star\,.
\end{equation}
Here, we note that \(\dot{V}(0) = 0\).
That is, \(V(0)\) is at its optimal value given \(V'(0)\) and \(s'(0)\).
The following lemma shows that \(V\) stays close to its optimum throughout the trajectory:

\fastslow*
\begin{proof}
      We compute the derivative of \(\Delta\):
      \begin{equation*}
            \dot{\Delta} = - (\nheads-1) \Delta + \frac{1}{\nheads-1} \langle s_1^\star, s'\rangle \dot{V}' + \frac{1}{\nheads-1} \langle s_1^\star, \dot{s}'\rangle V'\,.
      \end{equation*}
      Setting \(c_2 = \nheads-1\) and \(c(t) = \dfrac{1}{\nheads-1} \langle s_1^\star, s'(t) \rangle V'(t)\), this reads \(\dot{\Delta} = - c_2 \Delta + \dot{c}\).
      Hence, \(\frac{d}{dt} \left(e^{c_2 t} \Delta(t)\right) = e^{c_2 t} \dot{c}(t)\), and integrating from \(0\) to \(t\) gives
      \begin{equation*}
            \Delta(t) = e^{-c_2 t} \Delta(0) + \int_0^t e^{-c_2 (t - \tau)} \dot{c}(\tau) \, d\tau\,.
      \end{equation*}
      Moreover, by \Cref{lem:apriori}~(ii), \(V\) and \(V'\) are bounded, and hence so are \(\dot{V}'\) and \(\dot{s}'\) by \Cref{eq:cooperative}.
      Therefore, \(\|\dot{c}(t)\|_F \leq c_1\) for some constant \(c_1\) depending only on \(\nheads\), \(T\), \(\mG\), \(\delta\), and \(\|\mG - \mP(0)\|_F\), and
      \begin{equation*}
            \|\Delta(t)\|_F \leq e^{-c_2 t} \|\Delta(0)\|_F + c_1 \int_0^t e^{-c_2 (t - \tau)} \, d\tau = e^{-c_2 t} \|\Delta(0)\|_F + \frac{c_1}{c_2} \left(1 - e^{-c_2 t}\right)\,.
      \end{equation*}
\end{proof}

Based on \Cref{lem:optimality} and evidence from our numerical simulations, we approximate the full dynamics by a two-scale analysis where \(V\) is optimized faster than \(V'\) and \(s'\), leading to \Cref{eq:second_dynamics}.
Expanding \(\Pi(s')\) gives
\begin{equation*}
      \Pi(s') = \Pi(s_{(1)}') + \langle s_1^\star, s'\rangle \left(\left(1 - \langle s_1^\star, s'\rangle\right) s_1^\star \left(s_1^\star\right)^\top - s_1^\star (s_{(1)}')^\top - s_{(1)}' \left(s_1^\star\right)^\top\right)\,.
\end{equation*}
Since \(V'^\top \mG_{(1)} - \|V'\|_F^2 s_{(1)}'\) is perpendicular to the direction \(s_1^\star\), we obtain:
\begin{equation*}
      \dot{s}' = \Pi(s') \left(\Pi(s_{(1)}') - \langle s_1^\star, s' \rangle s_1^\star (s_{(1)}')^\top \right) \left(V'^\top \mG_{(1)} - \|V'\|_F^2 s_{(1)}'\right)\,.
\end{equation*}
Writing out the update along the direction of \(s_1^\star\):
\begin{equation*}
      \begin{split}
            \langle s_1^\star, \dot{s}' \rangle &= \langle s_1^\star, s' \rangle \left(\left(1 - \langle s_1^\star, s' \rangle\right) s_1^\star - s_{(1)}'\right)^\top \left(\Pi(s_{(1)}') - \langle s_1^\star, s' \rangle s_1^\star (s_{(1)}')^\top \right) \left(V'^\top \mG_{(1)} - \|V'\|_F^2 s_{(1)}'\right) \\
            &= - \langle s_1^\star, s'\rangle (s_{(1)}')^\top \left(\Pi(s_{(1)}') + \langle s_1^\star, s'\rangle \left(1 - \langle s_1^\star, s' \rangle\right) I\right) \left(V'^\top \mG_{(1)} - \|V'\|_F^2 s_{(1)}'\right) \\
            &= - \langle s_1^\star, s'\rangle (s_{(1)}')^\top \left(\diag(s_{(1)}') + \left[\langle s_1^\star, s'\rangle \left(1 - \langle s_1^\star, s' \rangle\right) - \|s_{(1)}'\|^2\right] I\right) \left(V'^\top \mG_{(1)} - \|V'\|_F^2 s_{(1)}'\right)\,.
      \end{split}
\end{equation*}
The rest of the update follows:
\begin{equation*}
      \begin{split}
            \dot{s}'_{(1)} &= \left(\Pi(s_{(1)}') - \langle s_1^\star, s'\rangle s_{(1)}' (s_1^\star)^\top\right) \left(\Pi(s_{(1)}') - \langle s_1^\star, s' \rangle s_1^\star (s_{(1)}')^\top \right) \left(V'^\top \mG_{(1)} - \|V'\|_F^2 s_{(1)}'\right) \\
            &= \left(\Pi(s_{(1)}')^2 + \langle s_1^\star, s'\rangle^2 s_{(1)}' (s_{(1)}')^\top\right) \left(V'^\top \mG_{(1)} - \|V'\|_F^2 s_{(1)}'\right)\,.
      \end{split}
\end{equation*}
Similarly, writing the update in the direction of \(V_1^\star\) and the update for \(V'_{(1)}\):
\begin{equation*}
      \begin{split}
            \langle V_1^\star, \dot{V}'\rangle &= - \|s_{(1)}'\|^2 \langle V_1^\star, V'\rangle\,, \\
            \dot{V}'_{(1)} &= \mG_{(1)} s_{(1)}' - \|s_{(1)}'\|^2 V'_{(1)}\,.
      \end{split}
\end{equation*}

We first show that the trajectory converges to the family of stationary points of \Cref{eq:second_dynamics}.
\begin{proposition}\label{prop:family}
      Under the assumptions of \Cref{thm:cooperative}, the \(\omega\)-limit set \(\Omega\) of \((V'(t), s'(t))\) satisfies
      \begin{equation*}
            \Omega \subseteq \left\{\left(\frac{m_2^\star}{\rho} V_2^\star,\ (1 - \rho)\, s_1^\star + \rho\, s_2^\star\right) : \rho \in (0, 1]\right\}\,,
      \end{equation*}
      and \(\lim_{t \to \infty} V'(t) \otimes s'_{(1)}(t) = m_2^\star V_2^\star \otimes s_2^\star\).
\end{proposition}
\begin{proof}
      We follow the same strategy as in \Cref{thm:coupling}.
      The vector field of \Cref{eq:second_dynamics} is polynomial in \((V', s')\).
      Hence, it is locally Lipschitz, and the solution is unique.
      Moreover, \(V'(t)\) is bounded by \Cref{lem:apriori}~(iii), and \(s'(t)\) lies on the simplex.
      Thus, the solution is bounded and exists for all \(t \geq 0\).

      Let \(\gR\) be the following set:
      \begin{equation*}
            \gR = \left\{ (V', s') \mid s' \text{ on the simplex}, \forall k \in [2, \nheads], \langle V', A_2^\star\rangle \geq \langle V', A_k^\star \rangle \enspace \text{and} \enspace \langle s', s_2^\star\rangle \geq \langle s', s_k^\star \rangle \right\}\,.
      \end{equation*}
      We prove that the flow is forward-invariant on \(\gR\).

      Fix any \(j \in [2, \nheads]\).
      Let \(w_j = \langle V', A_2^\star - A_j^\star \rangle\) and \(z_j = \langle s_{(1)}', s_2^\star - s_j^\star \rangle\).
      The flow of \(w_j\) and \(z_j\) are as follows:
      \begin{equation*}
            \begin{split}
                  \dot{w}_j &= (m_2^\star)^2 \langle s_{(1)}', s_2^\star\rangle - (m_j^\star)^2 \langle s_{(1)}', s_j^\star\rangle - \|s_{(1)}'\|^2 w_j\,, \\
                  \dot{z}_j &= \left(s_2^\star - s_j^\star\right)^\top \left(\Pi(s_{(1)}')^2 + \langle s_1^\star, s'\rangle^2 s_{(1)}' (s_{(1)}')^\top\right) \left(V'^\top \mG_{(1)} - \|V'\|_F^2 s_{(1)}'\right)\,.
            \end{split}
      \end{equation*}
      Rewriting \(\dot{z}_j\):
      \begin{equation*}
            \begin{split}
                  \dot{z}_j &= (s_2^\star - s_j^\star)^\top \Pi(s_{(1)}')^2 \left(V'^\top \mG_{(1)} - \|V'\|_F^2 s_{(1)}'\right) + c z_j \\
                  &= (s_2^\star - s_j^\star)^\top \diag(s_{(1)}')^2 \left(V'^\top \mG_{(1)} - \|V'\|_F^2 s_{(1)}'\right) \\
                  &\quad\quad - (s_2^\star - s_j^\star)^\top \diag(s_{(1)}') s_{(1)}' (s_{(1)}')^\top \left(V'^\top \mG_{(1)} - \|V'\|_F^2 s_{(1)}'\right) + c z_j \\
                  &= (s_2^\star - s_j^\star)^\top \diag(s_{(1)}')^2 \left(V'^\top \mG_{(1)} - \|V'\|_F^2 s_{(1)}'\right) \\
                  &\quad\quad - \left(\langle s_{(1)}', s_2^\star - s_j^\star \rangle s_2^\star + \langle s_{(1)}', s_j^\star \rangle \left(s_2^\star - s_j^\star\right)\right)^\top s_{(1)}' (s_{(1)}')^\top \left(V'^\top \mG_{(1)} - \|V'\|_F^2 s_{(1)}'\right) + c z_j \\
                  &= m_2^\star \langle s_2^\star, s_{(1)}'\rangle^2 \langle V', V_2^\star \rangle - m_j^\star \langle s_j^\star, s_{(1)}'\rangle^2 \langle V', V_j^\star \rangle + c z_j\,,
            \end{split}
      \end{equation*}
      where \(c\) is some arbitrary time-dependent function that changes from line to line.
      On the boundary of \(\gR\), we have \(w_j = 0\) or \(z_j = 0\).
      If \(w_j = 0\), then \(\dot{w}_j \geq 0\) as \(\langle s_{(1)}', s_2^\star\rangle \geq \langle s_{(1)}', s_j^\star\rangle \geq 0\) and \(m_2^\star > m_j^\star\).
      If \(z_j = 0\), then \(\dot{z}_j = \langle s_{(1)}', s_2^\star\rangle^2 w_j \geq 0\).

      Since \(\gR\) is defined by finitely many linear equalities and inequalities, by Nagumo's theorem it suffices to check that \(\dot{w}_j \geq 0\) whenever \(w_j = 0\), \(\dot{z}_j \geq 0\) whenever \(z_j = 0\), and \(\dot{s}'_i \geq 0\) whenever \(s'_i = 0\), and that the flow preserves \(\sum_i s'_i = 1\).
      The first two conditions are shown above, and the last two hold since \(\dot{s}'_i\) is proportional to \(s'_i\) and \(\dot{s}'\) is orthogonal to \(1_T\).
      Therefore, a flow that has started in \(\gR\) will remain in \(\gR\) for all time.
      
      Now, consider the following Lyapunov function:
      \begin{equation*}
            \phi(V', s_{(1)}') = \langle V', \mG_{(1)} s_{(1)}'\rangle - \dfrac{1}{2} \|V'\|_F^2 \|s_{(1)}'\|^2\,.
      \end{equation*}
      The derivative of \(\phi(V', s_{(1)}')\) is as follows:
      \begin{equation*}
            \begin{split}
                  \nabla_{V'} \phi(V', s_{(1)}') &= \mG_{(1)} s_{(1)}' - \|s_{(1)}'\|^2 V'\,, \\
                  \nabla_{s_{(1)}'} \phi(V', s_{(1)}') &= V'^\top \mG_{(1)} - \|V'\|_F^2 s_{(1)}'\,.                  
            \end{split}
      \end{equation*}
      Therefore, the time derivative of \(\phi\):
      \begin{equation*}
            \dot{\phi} = \|\dot{V}'\|_F^2 + \|\tilde{\Pi}(s') \nabla_{s_{(1)}'} \phi(V', s')\|^2\, \geq 0\,,
      \end{equation*}
      where \(\tilde{\Pi}(s')\) is a positive semi-definite matrix that satisfies:
      \begin{equation*}
            \tilde{\Pi}(s')^2 = \left(\Pi(s_{(1)}')^2 + \langle s_1^\star, s'\rangle^2 s_{(1)}' (s_{(1)}')^\top\right)\,, \quad \ker(\tilde{\Pi}(s')) \subseteq \ker(\Pi(s_{(1)}'))\,.
      \end{equation*}

      Since the solution is bounded, its \(\omega\)-limit set \(\Omega\) is non-empty, compact, and invariant, and \((V'(t), s'(t))\) converges to \(\Omega\).
      By LaSalle's invariance principle, \(\Omega\) is contained in the largest invariant subset of \(\{\dot{\phi} = 0\}\).
      Since \(\dot{\phi} = \|\dot{V}'\|_F^2 + \|\tilde{\Pi}(s') \nabla_{s_{(1)}'} \phi\|^2\), \(\dot{\phi} = 0\) implies \(\dot{V}' = 0\) and \(\tilde{\Pi}(s') \nabla_{s_{(1)}'} \phi = 0\), hence \(\dot{s}_{(1)}' = \tilde{\Pi}(s')^2 \nabla_{s_{(1)}'} \phi = 0\) and \(\langle s_1^\star, \dot{s}' \rangle = -\langle 1_T, \dot{s}_{(1)}' \rangle = 0\).
      Thus, every point of \(\{\dot{\phi} = 0\}\) is a stationary point of \Cref{eq:second_dynamics}, and so is every point of \(\Omega\).
      Since \(\gR\) is closed and forward-invariant, \(\Omega \subseteq \gR\).

      Let \((V'_\infty, s'_\infty) \in \Omega\) and \(r_\infty \coloneqq (s'_\infty)_{(1)}\), which is non-zero by \Cref{lem:apriori}~(iii).
      Then, it satisfies
      \begin{equation*}
            \mG_{(1)} r_\infty - \|r_\infty\|^2 V'_\infty = 0\,, \quad V_\infty'^\top \mG_{(1)} - \|V_\infty'\|_F^2 r_\infty \in \ker(\tilde{\Pi}(s'_\infty)) \subseteq \ker(\Pi(r_\infty))\,.
      \end{equation*}

      Note that we have the following equality:
      \begin{equation*}
            \begin{split}
                  (\mG_{(1)} r_\infty)^\top \mG_{(1)} &= \sum_{j=2}^{\nheads} m_j^\star \left\langle V_j^\star, \sum_{k=2}^{\nheads} m_k^\star V_k^\star \langle s_k^\star, r_\infty\rangle \right\rangle s_j^\star = \sum_{j=2}^{\nheads} (m_j^\star)^2 \langle s_j^\star, r_\infty\rangle s_j^\star\,.
            \end{split}
      \end{equation*}
      Then, substituting \(V'_\infty = \mG_{(1)} r_\infty / \|r_\infty\|^2\) into the second condition and multiplying by \(\|r_\infty\|^2\), the stationary point \((V'_\infty, s'_\infty)\) satisfies
      \begin{equation*}
            u \coloneqq \sum_{j=2}^{\nheads} (m_j^\star)^2 \langle s_j^\star, r_\infty\rangle s_j^\star - \lambda \, r_\infty \in \mathrm{ker}(\Pi(r_\infty))\,, \quad \text{where} \quad \lambda = \|r_\infty\|^2 \|V'_\infty\|_F^2 = \dfrac{\|\mG_{(1)} r_\infty\|_F^2}{\|r_\infty\|^2}\,.
      \end{equation*}

      For \(i \in [T]\), let \(\mu_i = m_k^\star\) where \(k \in [T]\) is such that \(s_k^\star = e_i\), so that \(u_i = (\mu_i^2 - \lambda) (r_\infty)_i\).
      By \Cref{lem:kernel}, \(u_i = c\) for all \(i\) with \((r_\infty)_i > 0\), where
      \begin{equation*}
            c = \langle r_\infty, u \rangle = \sum_{j=2}^{\nheads} (m_j^\star)^2 \langle s_j^\star, r_\infty\rangle^2 - \lambda \|r_\infty\|^2 = \|\mG_{(1)} r_\infty\|_F^2 - \|\mG_{(1)} r_\infty\|_F^2 = 0\,.
      \end{equation*}
      Hence, \(\mu_i^2 = \lambda\) for every \(i\) in the support of \(r_\infty\).

      We now distinguish two cases.
      Suppose first that \(\langle s_k^\star, r_\infty\rangle > 0\) for some \(k \in [2, \nheads]\).
      Since \((V'_\infty, s'_\infty) \in \gR\) and \(\langle s'_\infty, s_k^\star\rangle = \langle r_\infty, s_k^\star\rangle\) for \(k \geq 2\), we have \(\langle s_2^\star, r_\infty\rangle \geq \langle s_k^\star, r_\infty\rangle > 0\).
      Thus, the position of \(s_2^\star\) is in the support of \(r_\infty\), and \(\lambda = (m_2^\star)^2\).
      By \Cref{assum:distinct}, no other position satisfies \(\mu_i^2 = (m_2^\star)^2\), so the support of \(r_\infty\) consists of this single position.
      Hence, \(r_\infty = \rho\, s_2^\star\) with \(\rho = \|r_\infty\| \in (0, 1]\), \(s'_\infty = (1 - \rho)\, s_1^\star + \rho\, s_2^\star\), and \(V'_\infty = \mG_{(1)} r_\infty / \rho^2 = \frac{m_2^\star}{\rho} V_2^\star\).

      Suppose otherwise that \(\langle s_k^\star, r_\infty\rangle = 0\) for all \(k \in [2, \nheads]\), i.e., \(r_\infty\) is supported on positions with \(m_k^\star = 0\).
      Then, \(\mG_{(1)} r_\infty = 0\), so \(V'_\infty = 0\) and \(\phi(V'_\infty, r_\infty) = 0\).
      This is impossible: \(\phi\) is non-decreasing along the flow, so \(\phi(V'_\infty, r_\infty) = \lim_{t \to \infty} \phi(V'(t), s_{(1)}'(t)) \geq \phi(0) > 0\) by \Cref{eq:init_loss}.

      Therefore,
      \begin{equation*}
            \Omega \subseteq \left\{\left(\frac{m_2^\star}{\rho} V_2^\star,\ (1 - \rho)\, s_1^\star + \rho\, s_2^\star\right) : \rho \in (0, 1]\right\}\,,
      \end{equation*}
      and \(V' \otimes s'_{(1)} = m_2^\star V_2^\star \otimes s_2^\star\) on this set.
      Since the solution is bounded and converges to \(\Omega\), and \((V', s') \mapsto V' \otimes s'_{(1)}\) is continuous and equal to \(m_2^\star V_2^\star \otimes s_2^\star\) on \(\Omega\), we conclude that \(V'(t) \otimes s'_{(1)}(t) \to m_2^\star V_2^\star \otimes s_2^\star\).
\end{proof}

We now discuss when the initialization conditions of \Cref{thm:cooperative} hold.
\Cref{thm:coupling,thm:bounded} demonstrate that a wide range of symmetric initializations converge toward the configuration defined in \Cref{eq:fast_init}.
Note that \Cref{eq:init_loss} requires stronger alignment than \Cref{eq:fast_init}, specifically that the tensor factorization loss is strictly lower than the value attained at the first saddle point characterized by \Cref{thm:coupling}.
In practice, the initialization conditions of \Cref{thm:cooperative} are satisfied by small perturbations toward the second positional feature, as made precise in \Cref{rem:init_loss}.

\subsubsection{Convergence of the Attention to the Second Feature}
\label{app:attention}

\Cref{prop:family} shows that \(\Omega\) lies in the family \(\{(\frac{m_2^\star}{\rho} V_2^\star, (1 - \rho) s_1^\star + \rho s_2^\star) : \rho \in (0, 1]\}\), but not that the trajectory converges to a single point of it.
In this section, we show that it converges to the point with \(\rho = 1\), that is, the attention of the second head moves entirely to the second feature.
The argument gives no rate of convergence, and the attention on the first feature can take a long time to decay.
This is because it decreases at a rate controlled by the attention on the remaining positions, which itself vanishes over time and is small from the start for the initializations in \Cref{rem:init_loss}.

The proof proceeds as follows.
By \Cref{lem:mass}, \(\dot{a} = - a \Psi\), so it suffices to show that \(\int_0^\infty \Psi \, dt = \infty\).
We first compare \(V'\) with its optimal value given the attention, and show that, near the family of stationary points, the value of \(\Psi\) at this optimum is bounded from below by a multiple of a gap \(\Gamma\), which measures the attention outside the two features (\Cref{lem:alignment}).
The deviation of \(V'\) from its optimum contributes error terms, all of which are integrable except a single term that is linear in the component \(\Delta_2\) of this deviation along \(V_2^\star\) (\Cref{lem:a_equation,cor:psi_tilde}).
Two differential inequalities, for \(\ln a\) and \(\ln \Gamma\), then reduce the claim to the integrability of \(\Delta_2\) (\Cref{lem:ln_a,lem:ln_gamma}).
Finally, we prove this integrability (\Cref{prop:delta2}).
Here, \(a\) and \(\Delta_2\) drive each other: \(\Delta_2\) is driven by the decrease of \(a\), and feeds back into the equation for \(a\).
The key point is that this feedback returns only a fraction \(\theta_{\max} < 1\) of what it receives.

Throughout this section, \((V', s')\) is the solution of \Cref{eq:second_dynamics} under the assumptions of \Cref{thm:cooperative}, and \(\delta \coloneqq \inf_{t \geq 0} \|s'(t) - s_1^\star\| > 0\).
We use the following notation:
\begin{equation*}
      a \coloneqq \langle s_1^\star, s'\rangle\,, \quad b \coloneqq \langle s_2^\star, s'\rangle\,, \quad r \coloneqq s'_{(1)}\,, \quad D \coloneqq \|r\|^2\,, \quad \nu \coloneqq \|V'\|_F^2\,.
\end{equation*}
We also write
\begin{equation*}
      g \coloneqq V'^\top \mG_{(1)} - \nu \, r\,, \quad \beta \coloneqq \langle r, g\rangle\,,
\end{equation*}
so that \(\dot{s}' = \Pi(s')^2 g\).
Since \(s_1^\star\) is one-hot, \(r\) is \(s'\) with its entry on \(s_1^\star\) set to zero, and \(\|r\|_1 = 1 - a\).
Moreover, \(\langle s_1^\star, g\rangle = 0\), since \(\mG_{(1)}\) and \(r\) both vanish along \(s_1^\star\).
Finally, we let
\begin{equation*}
      v_{\max} \coloneqq \max\left\{\|V'(0)\|_F, \frac{m_2^\star \sqrt{T}}{\delta}\right\}\,,
\end{equation*}
so that \(\|V'(t)\|_F \leq v_{\max}\) for all \(t \geq 0\) by \Cref{lem:apriori}~(iii).

The first lemma gives the dynamics of the individual entries of \(s'\), and in particular of \(a\).
\begin{lemma}\label{lem:mass}
      Let \(\tilde{g} \coloneqq g - \beta 1_T\) and \(M \coloneqq \sum_{i=1}^T (s'_i)^2 \tilde{g}_i\).
      Then, for all \(i \in [T]\),
      \begin{equation*}
            \dot{s}'_i = s'_i \left(s'_i \tilde{g}_i - M\right)\,.
      \end{equation*}
      In particular, \(\dot{a} = - a \Psi\), where
      \begin{equation*}
            \Psi \coloneqq M + a \beta = \sum_{i=1}^T r_i^2 g_i + \left(a (1 - a) - D\right) \beta\,.
      \end{equation*}
\end{lemma}
\begin{proof}
      For any vector \(x\), the \(i\)-th entry of \(\Pi(s') x\) is \(s'_i (x_i - \langle s', x\rangle)\).
      Since \(\langle s_1^\star, g\rangle = 0\), we have \(\langle s', g\rangle = \langle r, g\rangle = \beta\).
      Hence, the \(i\)-th entry of \(\Pi(s') g\) is \(s'_i (g_i - \beta) = s'_i \tilde{g}_i\).
      Applying \(\Pi(s')\) once more, and using \(\langle s', \Pi(s') g\rangle = \sum_j (s'_j)^2 \tilde{g}_j = M\), we get \(\dot{s}'_i = s'_i (s'_i \tilde{g}_i - M)\).

      For the entry on \(s_1^\star\), we have \(\tilde{g}_i = - \beta\), so \(\dot{a} = a (- a \beta - M) = - a \Psi\).
      Splitting off this entry in \(M\),
      \begin{equation*}
            M = - a^2 \beta + \sum_{i=1}^T r_i^2 (g_i - \beta) = - a^2 \beta + \sum_{i=1}^T r_i^2 g_i - D \beta\,.
      \end{equation*}
      Adding \(a \beta\) gives the second expression for \(\Psi\).
\end{proof}

Next, we compare \(V'\) with its optimal value given the attention, \(V'_{\mathrm{opt}} \coloneqq \mG_{(1)} r / D\), which maximizes \(\phi(\cdot, r)\).
The following lemma shows that the tracking error is square-integrable.
\begin{lemma}\label{lem:tracking}
      Let \(\Delta \coloneqq V' - V'_{\mathrm{opt}}\).
      Then, \(\dot{V}' = - D \Delta\) and \(\beta = \langle V', \dot{V}'\rangle\).
      Moreover,
      \begin{equation*}
            \int_0^\infty \|\Delta(t)\|_F^2 \, dt \leq L_\Delta \coloneqq \left(\frac{T}{\delta^2}\right)^2 \left(\frac{(m_2^\star)^2}{2} - \phi(0)\right)\,.
      \end{equation*}
\end{lemma}
\begin{proof}
      By \Cref{eq:second_dynamics}, \(\dot{V}' = \mG_{(1)} r - D V' = - D \Delta\).
      Moreover, \(\beta = \langle r, V'^\top \mG_{(1)}\rangle - \nu D = \langle V', \mG_{(1)} r - D V'\rangle = \langle V', \dot{V}'\rangle\).

      For the integral, we first show that \(\phi\) is bounded from above.
      Completing the square in \(V'\), we get
      \begin{equation*}
            \phi = \langle V', \mG_{(1)} r\rangle - \frac{D}{2} \|V'\|_F^2 = \frac{\|\mG_{(1)} r\|_F^2}{2D} - \frac{D}{2} \left\|V' - \frac{\mG_{(1)} r}{D}\right\|_F^2 = \frac{\|\mG_{(1)} r\|_F^2}{2D} - \frac{D}{2} \|\Delta\|_F^2\,.
      \end{equation*}
      The last term is non-positive, so \(\phi \leq \|\mG_{(1)} r\|_F^2 / (2D)\).
      Since \(\mG_{(1)} r = \sum_{k=2}^{\nheads} m_k^\star \langle s_k^\star, r\rangle V_k^\star\) with orthonormal \(V_k^\star\), and \(m_k^\star \leq m_2^\star\) for \(k \geq 2\),
      \begin{equation*}
            \|\mG_{(1)} r\|_F^2 = \sum_{k=2}^{\nheads} (m_k^\star)^2 \langle s_k^\star, r\rangle^2 \leq (m_2^\star)^2 \sum_{k=2}^{\nheads} \langle s_k^\star, r\rangle^2 \leq (m_2^\star)^2 D\,.
      \end{equation*}
      Hence, \(\phi \leq (m_2^\star)^2 / 2\) at all times.

      Next, recall from the proof of \Cref{prop:family} that \(\dot{\phi} \geq \|\dot{V}'\|_F^2\).
      Integrating from \(0\) to \(t\), we get
      \begin{equation*}
            \int_0^t \|\dot{V}'(\tau)\|_F^2 \, d\tau \leq \phi(t) - \phi(0) \leq \frac{(m_2^\star)^2}{2} - \phi(0)\,.
      \end{equation*}
      Letting \(t \to \infty\), \(\int_0^\infty \|\dot{V}'\|_F^2 \, dt \leq (m_2^\star)^2 / 2 - \phi(0)\).
      By \Cref{lem:apriori}~(iii), \(D \geq \delta^2 / T\), so \(\|\Delta\|_F = \|\dot{V}'\|_F / D \leq (T / \delta^2) \|\dot{V}'\|_F\).
      Squaring and integrating gives the claim.
\end{proof}

We now split \(g\) into its value on the slow manifold \(V' = V'_{\mathrm{opt}}\) and an error controlled by \(\Delta\).
For \(i \in [T]\), let \(\mu_i \coloneqq \|\mG_{(1)} e_i\|_F\).
Thus, \(\mu_i = m_k^\star\) if \(s_k^\star = e_i\) for some \(k \in [2, \nheads]\), and \(\mu_i = 0\) otherwise.
By \Cref{assum:orthogonal}, the columns \(\mG_{(1)} e_i\) are pairwise orthogonal.
We also define the Rayleigh quotient
\begin{equation*}
      R \coloneqq \frac{\|\mG_{(1)} r\|_F^2}{D} = \frac{1}{D} \sum_{i=1}^T \mu_i^2 r_i^2\,.
\end{equation*}
\begin{lemma}\label{lem:split}
      Let \(g_{\mathrm{opt}}\) be the value of \(g\) at \(V' = V'_{\mathrm{opt}}\), and let \(g_\Delta \coloneqq g - g_{\mathrm{opt}}\).
      Then, for all \(i \in [T]\),
      \begin{equation*}
            (g_{\mathrm{opt}})_i = \frac{r_i \left(\mu_i^2 - R\right)}{D}\,, \quad \text{and} \quad \langle r, g_{\mathrm{opt}}\rangle = 0\,.
      \end{equation*}
      Moreover, with \(K_\Delta \coloneqq m_2^\star (1 + 3\sqrt{T} / \delta) + v_{\max}\),
      \begin{equation*}
            \|g_\Delta\| \leq K_\Delta \|\Delta\|_F\,, \quad |\beta| \leq K_\Delta \|\Delta\|_F\,.
      \end{equation*}
\end{lemma}
\begin{proof}
      For any matrix \(W\), the \(i\)-th entry of \(W^\top \mG_{(1)}\) is \(\langle W, \mG_{(1)} e_i\rangle\).

      We first compute \(g_{\mathrm{opt}}\).
      Since the columns of \(\mG_{(1)}\) are pairwise orthogonal with norms \(\mu_i\),
      \begin{equation*}
            \left(V_{\mathrm{opt}}'^\top \mG_{(1)}\right)_i = \frac{1}{D} \sum_{j=1}^T r_j \langle \mG_{(1)} e_j, \mG_{(1)} e_i\rangle = \frac{\mu_i^2 r_i}{D}\,.
      \end{equation*}
      Moreover, \(\|V'_{\mathrm{opt}}\|_F^2 = \|\mG_{(1)} r\|_F^2 / D^2 = R / D\).
      Hence, \((g_{\mathrm{opt}})_i = \mu_i^2 r_i / D - (R / D) r_i\), which is the first claim.
      Then, \(\langle r, g_{\mathrm{opt}}\rangle = \frac{1}{D} \left(\sum_i \mu_i^2 r_i^2 - R D\right) = 0\) by the definition of \(R\).

      Next, we bound the error \(g_\Delta\).
      By definition,
      \begin{equation*}
            g_\Delta = \Delta^\top \mG_{(1)} - \left(\nu - \|V'_{\mathrm{opt}}\|_F^2\right) r\,.
      \end{equation*}
      For the first term, the columns \(\mG_{(1)} e_i / \mu_i\) with \(\mu_i > 0\) are orthonormal and \(\mu_i \leq m_2^\star\), so \(\|\Delta^\top \mG_{(1)}\| \leq m_2^\star \|\Delta\|_F\).
      For the second term, we expand \(\nu = \|V'_{\mathrm{opt}} + \Delta\|_F^2\) to get
      \begin{equation*}
            \nu - \|V'_{\mathrm{opt}}\|_F^2 = 2 \langle V'_{\mathrm{opt}}, \Delta\rangle + \|\Delta\|_F^2\,.
      \end{equation*}
      To bound \(V'_{\mathrm{opt}}\), recall that \(\|V'_{\mathrm{opt}}\|_F^2 = R / D\).
      The proof of \Cref{lem:tracking} shows that \(R \leq (m_2^\star)^2\), and \Cref{lem:apriori}~(iii) gives \(D \geq \delta^2 / T\).
      Hence,
      \begin{equation*}
            \|V'_{\mathrm{opt}}\|_F \leq \frac{m_2^\star \sqrt{T}}{\delta}\,.
      \end{equation*}
      In particular, \(\|\Delta\|_F \leq \|V'\|_F + \|V'_{\mathrm{opt}}\|_F \leq v_{\max} + m_2^\star \sqrt{T} / \delta\).
      Since \(\|r\| \leq 1\), we get
      \begin{equation*}
            \|g_\Delta\| \leq m_2^\star \|\Delta\|_F + \left(\frac{2 m_2^\star \sqrt{T}}{\delta} + \|\Delta\|_F\right) \|\Delta\|_F \leq \left(m_2^\star + \frac{3 m_2^\star \sqrt{T}}{\delta} + v_{\max}\right) \|\Delta\|_F = K_\Delta \|\Delta\|_F\,.
      \end{equation*}
      Finally, \(\beta = \langle r, g\rangle = \langle r, g_\Delta\rangle\) since \(\langle r, g_{\mathrm{opt}}\rangle = 0\), so \(|\beta| \leq \|r\| \|g_\Delta\| \leq K_\Delta \|\Delta\|_F\).
\end{proof}

The gap between \(R\) and its largest possible value \((m_2^\star)^2\) measures how much of \(r\) lies outside the second feature.
Let \(O \coloneqq \{i \in [T] : e_i \neq s_1^\star, s_2^\star\}\) be the remaining positions, and let
\begin{equation*}
      q \coloneqq \sum_{i \in O} r_i^2\,, \quad \gamma \coloneqq (m_2^\star)^2 - (m_3^\star)^2\,.
\end{equation*}
By \Cref{assum:distinct}, \(\gamma > 0\), where \(m_3^\star = 0\) if \(\nheads = 2\).
Note that \(D = b^2 + q\).
\begin{lemma}\label{lem:gap}
      Let \(\Gamma \coloneqq (m_2^\star)^2 - R\).
      Then,
      \begin{equation*}
            \Gamma = \frac{1}{D} \sum_{i \in O} \left((m_2^\star)^2 - \mu_i^2\right) r_i^2\,, \quad \text{and} \quad \frac{\gamma}{D} q \leq \Gamma \leq \frac{(m_2^\star)^2}{D} q\,.
      \end{equation*}
      Moreover, if \(r_i(0) > 0\) for some \(i \in O\), then \(\Gamma(t) > 0\) for all \(t \geq 0\).
\end{lemma}
\begin{proof}
      Since \(r\) vanishes on \(s_1^\star\) and \(\mu_i = m_2^\star\) on \(s_2^\star\),
      \begin{equation*}
            D \, \Gamma = (m_2^\star)^2 \sum_{i=1}^T r_i^2 - \sum_{i=1}^T \mu_i^2 r_i^2 = \sum_{i \in O} \left((m_2^\star)^2 - \mu_i^2\right) r_i^2\,.
      \end{equation*}
      Dividing by \(D\) gives the first claim.
      For \(i \in O\), we have \(\mu_i \leq m_3^\star\), so the weights satisfy \(\gamma \leq (m_2^\star)^2 - \mu_i^2 \leq (m_2^\star)^2\), which gives the bounds.

      For the last claim, \Cref{lem:mass} gives \(s'_i(t) = s'_i(0) \exp\left(\int_0^t (s'_i \tilde{g}_i - M) \, d\tau\right)\) for all \(i \in [T]\).
      Hence, for \(i \in O\) with \(r_i(0) > 0\), we have \(r_i(t) > 0\) for all \(t \geq 0\), so \(q > 0\) and \(\Gamma \geq \gamma q / D > 0\).
\end{proof}

\Cref{prop:family} implies that the trajectory eventually stays close to the family of stationary points.
\begin{lemma}\label{lem:localization}
      As \(t \to \infty\), we have \(q(t) \to 0\), \(\Gamma(t) \to 0\), and \(\Delta(t) \to 0\).
      Moreover, \(\liminf_{t \to \infty} b(t) \geq \delta / \sqrt{2}\).
\end{lemma}
\begin{proof}
      By \Cref{prop:family}, every point of \(\Omega\) has the form \((\frac{m_2^\star}{\rho} V_2^\star, (1 - \rho) s_1^\star + \rho s_2^\star)\) for some \(\rho \in (0, 1]\).
      At such a point, \(r = \rho s_2^\star\), so \(q = 0\), \(R = (m_2^\star)^2\), and \(\Gamma = 0\).
      Moreover, \(V'_{\mathrm{opt}} = \mG_{(1)} r / D = \frac{m_2^\star}{\rho} V_2^\star = V'\), so \(\Delta = 0\).
      Finally, \(b = \rho\), and \(\|s' - s_1^\star\| = \sqrt{2} \rho \geq \delta\), since every point of \(\Omega\) is a limit of points of the trajectory.
      Hence, \(b \geq \delta / \sqrt{2}\) on \(\Omega\).

      Let \(f\) be any of the functions \(q, \Gamma\), or \(\|\Delta\|_F\), viewed as functions of the state \((V', s')\).
      We write \(f(t)\) for \(f(V'(t), s'(t))\).
      By the first paragraph, \(f\) vanishes on \(\Omega\).
      Moreover, \(f\) is continuous in \((V', s')\) wherever \(D > 0\), since \(D\) appears only in denominators.
      By \Cref{lem:apriori}~(iii), \(D \geq \delta^2 / T\) along the trajectory.
      Since \(D\) is continuous and every point of \(\Omega\) is a limit of points of the trajectory, \(D \geq \delta^2 / T\) on \(\Omega\) as well.
      Thus, \(f\) is continuous at every point of \(\Omega\).
      Suppose that \(f(t) \not\to 0\).
      Then, there exist \(\epsilon > 0\) and \(t_n \to \infty\) with \(|f(t_n)| \geq \epsilon\).
      Since the trajectory is bounded, a subsequence of \((V'(t_n), s'(t_n))\) converges to a point of \(\Omega\), where \(f\) vanishes.
      This contradicts the continuity of \(f\) at this point, which proves the first three claims.

      For the last claim, choose \(t_n \to \infty\) with \(b(t_n) \to \liminf_{t \to \infty} b(t)\).
      Such a sequence exists since \(b \in [0, 1]\): for each \(n\), pick \(t_n \geq n\) with \(b(t_n) \leq \inf_{t \geq n} b(t) + 1/n\).
      As above, a subsequence of \((V'(t_n), s'(t_n))\) converges to a point of \(\Omega\).
      Since \(b = \langle s_2^\star, s'\rangle\) is continuous and \(b \geq \delta / \sqrt{2}\) on \(\Omega\), we get \(\liminf_{t \to \infty} b(t) \geq \delta / \sqrt{2}\).
\end{proof}

We fix \(t_0 \geq 0\) such that, for all \(t \geq t_0\),
\begin{equation*}
      \text{(R1)} \enspace b \geq b_{\min} \coloneqq \frac{\delta}{2}\,, \qquad \text{(R2)} \enspace r_i \leq \frac{b}{2} \enspace \forall i \in O\,, \qquad \text{(R3)} \enspace \Gamma \leq \frac{\gamma}{4}\,.
\end{equation*}
This is possible by \Cref{lem:localization}: (R1) holds eventually since \(\delta / 2 < \delta / \sqrt{2}\), (R2) since \(r_i \leq \sqrt{q} \to 0\) while \(b \geq b_{\min}\), and (R3) since \(\Gamma \to 0\).
All three conditions are of the form ``for all \(t \geq t_0\)'', so they remain valid if \(t_0\) is replaced by any larger time.
Later lemmas require further conditions of this form, each holding for all sufficiently large \(t\), and we then increase \(t_0\) accordingly.

We now show that, near the family of stationary points, the rate \(\Psi\) in \Cref{lem:mass} is bounded from below by the gap \(\Gamma\) on the slow manifold.
Let
\begin{equation*}
      \Psi_{\mathrm{opt}} \coloneqq \sum_{i=1}^T r_i^2 (g_{\mathrm{opt}})_i = \frac{1}{D} \sum_{i=1}^T r_i^3 \left(\mu_i^2 - R\right)\,,
\end{equation*}
where the second equality is \Cref{lem:split}.
This is the value of \(\Psi\) at \(V' = V'_{\mathrm{opt}}\), since \(\beta = \langle r, g_{\mathrm{opt}}\rangle = 0\) there.
\begin{lemma}\label{lem:alignment}
      For all \(t \geq t_0\),
      \begin{equation*}
            \Psi_{\mathrm{opt}} \geq \frac{b}{4} \Gamma \geq c_0 \Gamma\,, \quad \text{where} \quad c_0 \coloneqq \frac{b_{\min}}{4}\,.
      \end{equation*}
\end{lemma}
\begin{proof}
      Since \(\sum_i r_i^2 (\mu_i^2 - R) = D \langle r, g_{\mathrm{opt}}\rangle = 0\) by \Cref{lem:split}, we can subtract \(b\) from \(r_i\) inside the sum:
      \begin{equation*}
            D \, \Psi_{\mathrm{opt}} = \sum_{i=1}^T r_i^2 (r_i - b) \left(\mu_i^2 - R\right)\,.
      \end{equation*}
      The term on \(s_2^\star\) vanishes since \(r_i = b\) there, and the term on \(s_1^\star\) vanishes since \(r_i = 0\) there.
      For \(i \in O\), we write \(\mu_i^2 - R = - \left((m_2^\star)^2 - \mu_i^2\right) + \Gamma\).
      Then,
      \begin{equation*}
            D \, \Psi_{\mathrm{opt}} = \sum_{i \in O} r_i^2 (b - r_i) \left((m_2^\star)^2 - \mu_i^2\right) - \Gamma \sum_{i \in O} r_i^2 (b - r_i)\,.
      \end{equation*}

      We bound the two sums separately.
      By (R2), \(b - r_i \geq b / 2\) for \(i \in O\).
      Hence, by \Cref{lem:gap}, the first sum is at least
      \begin{equation*}
            \frac{b}{2} \sum_{i \in O} r_i^2 \left((m_2^\star)^2 - \mu_i^2\right) = \frac{b}{2} D \, \Gamma\,.
      \end{equation*}
      For the second sum, \(0 \leq b - r_i \leq b\), and \(q \leq D \Gamma / \gamma\) by \Cref{lem:gap}, so
      \begin{equation*}
            \sum_{i \in O} r_i^2 (b - r_i) \leq b \, q \leq \frac{b D \Gamma}{\gamma}\,.
      \end{equation*}

      Combining the two bounds and dividing by \(D\),
      \begin{equation*}
            \Psi_{\mathrm{opt}} \geq b \, \Gamma \left(\frac{1}{2} - \frac{\Gamma}{\gamma}\right) \geq \frac{b}{4} \Gamma\,,
      \end{equation*}
      where the last step uses \(\Gamma \leq \gamma / 4\) from (R3).
      Finally, \(b \geq b_{\min}\) by (R1), which gives \(\Psi_{\mathrm{opt}} \geq c_0 \Gamma\).
\end{proof}

To control \(\Psi - \Psi_{\mathrm{opt}}\), we decompose \(V'\) and \(\Delta\) along the feature directions.
Let \(v_k \coloneqq \langle V', V_k^\star\rangle\) for \(k \in [2, \nheads]\), and let \(V'_\perp \coloneqq V' - \sum_{k=2}^{\nheads} v_k V_k^\star\) be the remainder, which is orthogonal to \(V_2^\star, \ldots, V_{\nheads}^\star\).
Since \(V'_{\mathrm{opt}} = \mG_{(1)} r / D = \sum_{k=2}^{\nheads} \frac{m_k^\star \langle s_k^\star, r\rangle}{D} V_k^\star\), the tracking error decomposes as
\begin{equation*}
      \Delta = \sum_{k=2}^{\nheads} \Delta_k V_k^\star + V'_\perp\,, \quad \text{where} \quad \Delta_k \coloneqq v_k - \frac{m_k^\star \langle s_k^\star, r\rangle}{D}\,,
\end{equation*}
and \(\|\Delta\|_F^2 = \sum_{k=2}^{\nheads} \Delta_k^2 + \|V'_\perp\|_F^2\).
In particular, the component along the second feature is \(\Delta_2 = v_2 - m_2^\star b / D\).
Projecting \(\dot{V}' = - D \Delta\) from \Cref{lem:tracking} onto \(V_k^\star\) gives \(\dot{v}_k = - D \Delta_k\).

The next lemma shows that, to first order in \(\Delta\), the energy gap \(\nu - \|V'_{\mathrm{opt}}\|_F^2\) and \(\beta\) depend only on \(\Delta_2\): all other terms carry an extra factor \(\sqrt{q}\) or \(\|\Delta\|_F\).
\begin{lemma}\label{lem:energy}
      The following decompositions hold.
      \begin{enumerate}[label=(\roman*)]
            \item \(\nu - \|V'_{\mathrm{opt}}\|_F^2 = \dfrac{2 m_2^\star b}{D} \Delta_2 + \varepsilon_\nu\), where
            \begin{equation*}
                  \varepsilon_\nu \coloneqq \frac{2}{D} \sum_{k=3}^{\nheads} m_k^\star \langle s_k^\star, r\rangle \Delta_k + \|\Delta\|_F^2\,, \quad |\varepsilon_\nu| \leq \frac{2 m_2^\star}{D} \sqrt{q} \, \|\Delta\|_F + \|\Delta\|_F^2\,.
            \end{equation*}
            \item \(\beta = - D v_2 \Delta_2 + \varepsilon_\beta\), where
            \begin{equation*}
                  \varepsilon_\beta \coloneqq - D \left(\sum_{k=3}^{\nheads} v_k \Delta_k + \|V'_\perp\|_F^2\right)\,, \quad |\varepsilon_\beta| \leq m_2^\star \sqrt{q} \, \|\Delta\|_F + \|\Delta\|_F^2\,.
            \end{equation*}
      \end{enumerate}
\end{lemma}
\begin{proof}
      For \(k \geq 3\), the position of \(s_k^\star\) lies in \(O\), so \(\sum_{k=3}^{\nheads} \langle s_k^\star, r\rangle^2 \leq q\).
      Together with \(m_k^\star \leq m_2^\star\) and the Cauchy--Schwarz inequality, this gives
      \begin{equation*}
            \left|\sum_{k=3}^{\nheads} m_k^\star \langle s_k^\star, r\rangle \Delta_k\right| \leq m_2^\star \sqrt{q} \, \|\Delta\|_F\,.
      \end{equation*}

      For (i), expanding \(\nu = \|V'_{\mathrm{opt}} + \Delta\|_F^2\) gives \(\nu - \|V'_{\mathrm{opt}}\|_F^2 = 2 \langle V'_{\mathrm{opt}}, \Delta\rangle + \|\Delta\|_F^2\).
      By the coordinates of \(V'_{\mathrm{opt}}\) and \(\Delta\),
      \begin{equation*}
            \langle V'_{\mathrm{opt}}, \Delta\rangle = \sum_{k=2}^{\nheads} \frac{m_k^\star \langle s_k^\star, r\rangle}{D} \Delta_k = \frac{m_2^\star b}{D} \Delta_2 + \frac{1}{D} \sum_{k=3}^{\nheads} m_k^\star \langle s_k^\star, r\rangle \Delta_k\,.
      \end{equation*}
      This gives the decomposition, and the bound on \(\varepsilon_\nu\) follows from the inequality above.

      For (ii), \Cref{lem:tracking} gives \(\beta = \langle V', \dot{V}'\rangle = - D \langle V', \Delta\rangle\).
      By the coordinates,
      \begin{equation*}
            \langle V', \Delta\rangle = v_2 \Delta_2 + \sum_{k=3}^{\nheads} v_k \Delta_k + \|V'_\perp\|_F^2\,,
      \end{equation*}
      which gives the decomposition.
      For the bound, we write \(v_k = m_k^\star \langle s_k^\star, r\rangle / D + \Delta_k\), so that
      \begin{equation*}
            \left|\sum_{k=3}^{\nheads} v_k \Delta_k\right| \leq \frac{m_2^\star}{D} \sqrt{q} \, \|\Delta\|_F + \sum_{k=3}^{\nheads} \Delta_k^2\,.
      \end{equation*}
      Multiplying by \(D\), adding \(D \|V'_\perp\|_F^2\), and using \(D \leq \|r\|_1^2 \leq 1\), we get \(|\varepsilon_\beta| \leq m_2^\star \sqrt{q} \, \|\Delta\|_F + \|\Delta\|_F^2\).
\end{proof}

We now apply the same reasoning to \(\Psi\), which gives an exact equation for \(a\) with a single term linear in \(\Delta_2\).
Let \(\kappa \coloneqq a (1 - a) - D\), so that \(\Psi = \sum_i r_i^2 g_i + \kappa \beta\) by \Cref{lem:mass}.
\begin{lemma}\label{lem:a_equation}
      The following decompositions hold.
      \begin{enumerate}[label=(\roman*)]
            \item \(\Psi = \Psi_{\mathrm{opt}} + (b + \kappa) \beta + \varepsilon_\Psi\), where
            \begin{equation*}
                  \varepsilon_\Psi \coloneqq \sum_{i \in O} r_i (r_i - b) (g_\Delta)_i\,, \quad |\varepsilon_\Psi| \leq K_\Delta \sqrt{q} \, \|\Delta\|_F\,.
            \end{equation*}
            \item Let \(\tilde{\Psi} \coloneqq \Psi_{\mathrm{opt}} + (b + \kappa) \varepsilon_\beta + \varepsilon_\Psi\) and \(\Lambda \coloneqq a (b + \kappa) D v_2\).
            Then,
            \begin{equation*}
                  \Psi = \tilde{\Psi} - (b + \kappa) D v_2 \Delta_2\,, \quad \dot{a} = - a \tilde{\Psi} + \Lambda \Delta_2\,.
            \end{equation*}
      \end{enumerate}
\end{lemma}
\begin{proof}
      For (i), we split \(g = g_{\mathrm{opt}} + g_\Delta\) in \(\sum_i r_i^2 g_i\).
      The part with \(g_{\mathrm{opt}}\) is \(\Psi_{\mathrm{opt}}\) by definition.
      In the part with \(g_\Delta\), the entry on \(s_1^\star\) vanishes since \(r\) vanishes there, so
      \begin{equation*}
            \sum_{i=1}^T r_i^2 (g_\Delta)_i = b^2 (g_\Delta)_{s_2^\star} + \sum_{i \in O} r_i^2 (g_\Delta)_i\,,
      \end{equation*}
      where \((g_\Delta)_{s_2^\star}\) denotes the entry on \(s_2^\star\).
      Since \(\beta = \langle r, g_\Delta\rangle\) by \Cref{lem:split}, we have \(\beta = b (g_\Delta)_{s_2^\star} + \sum_{i \in O} r_i (g_\Delta)_i\).
      Multiplying by \(b\) and substituting \(b^2 (g_\Delta)_{s_2^\star}\),
      \begin{equation*}
            \sum_{i=1}^T r_i^2 (g_\Delta)_i = b \beta + \sum_{i \in O} r_i (r_i - b) (g_\Delta)_i = b \beta + \varepsilon_\Psi\,.
      \end{equation*}
      Adding \(\kappa \beta\) gives the decomposition.
      For the bound, \(|r_i - b| \leq 1\), so the Cauchy--Schwarz inequality and \Cref{lem:split} give \(|\varepsilon_\Psi| \leq \sqrt{q} \, \|g_\Delta\| \leq K_\Delta \sqrt{q} \, \|\Delta\|_F\).

      For (ii), we substitute \(\beta = - D v_2 \Delta_2 + \varepsilon_\beta\) from \Cref{lem:energy}~(ii) into (i).
      This gives \(\Psi = \tilde{\Psi} - (b + \kappa) D v_2 \Delta_2\).
      Then, \(\dot{a} = - a \Psi\) from \Cref{lem:mass} gives \(\dot{a} = - a \tilde{\Psi} + \Lambda \Delta_2\).
\end{proof}

Combining \Cref{lem:alignment} with the bounds on the error terms, we get a lower bound on \(\tilde{\Psi}\).
\begin{corollary}\label{cor:psi_tilde}
      For all \(t \geq t_0\),
      \begin{equation*}
            \tilde{\Psi} \geq \frac{c_0}{2} \Gamma - K_1 \|\Delta\|_F^2\,, \quad \text{where} \quad K_1 \coloneqq 2 + \frac{(2 m_2^\star + K_\Delta)^2}{2 c_0 \gamma}\,.
      \end{equation*}
      In particular, the negative part \(\tilde{\Psi}^- \coloneqq \max\{-\tilde{\Psi}, 0\}\) satisfies \(\tilde{\Psi}^- \leq K_1 \|\Delta\|_F^2\).
\end{corollary}
\begin{proof}
      First, \(|b + \kappa| \leq 2\), since \(0 \leq b \leq 1\) and \(-1 \leq -D \leq \kappa \leq a(1 - a) \leq 1/4\).
      By \Cref{lem:energy}~(ii) and \Cref{lem:a_equation}~(i), the error terms in \(\tilde{\Psi}\) satisfy
      \begin{equation*}
            \left|(b + \kappa) \varepsilon_\beta + \varepsilon_\Psi\right| \leq 2 \left(m_2^\star \sqrt{q} \, \|\Delta\|_F + \|\Delta\|_F^2\right) + K_\Delta \sqrt{q} \, \|\Delta\|_F = (2 m_2^\star + K_\Delta) \sqrt{q} \, \|\Delta\|_F + 2 \|\Delta\|_F^2\,.
      \end{equation*}

      Next, we bound the first term by a fraction of \(\Gamma\) plus a multiple of \(\|\Delta\|_F^2\).
      By \Cref{lem:gap} and \(D \leq 1\), we have \(q \leq D \Gamma / \gamma \leq \Gamma / \gamma\).
      Then, by the inequality \(xy \leq \frac{c_0}{2} x^2 + \frac{1}{2 c_0} y^2\) with \(x = \sqrt{\Gamma}\) and \(y = (2 m_2^\star + K_\Delta) \|\Delta\|_F / \sqrt{\gamma}\),
      \begin{equation*}
            (2 m_2^\star + K_\Delta) \sqrt{q} \, \|\Delta\|_F \leq \frac{c_0}{2} \Gamma + \frac{(2 m_2^\star + K_\Delta)^2}{2 c_0 \gamma} \|\Delta\|_F^2\,.
      \end{equation*}
      Hence, \(\tilde{\Psi} \geq \Psi_{\mathrm{opt}} - \frac{c_0}{2} \Gamma - K_1 \|\Delta\|_F^2\).
      By \Cref{lem:alignment}, \(\Psi_{\mathrm{opt}} \geq c_0 \Gamma\), which gives the first claim.
      The second claim follows since \(\Gamma \geq 0\).
\end{proof}

The gap \(\Gamma\) is square-integrable, by a differential inequality that follows from \Cref{lem:mass}.
\begin{lemma}\label{lem:gap_l2}
      For all \(t \geq t_0\),
      \begin{equation*}
            \dot{\Gamma} \leq - \frac{c_\Gamma}{2} \Gamma^2 + K_3 \|\Delta\|_F^2\,, \quad \text{where} \quad c_\Gamma \coloneqq \frac{9 \gamma^2}{8 (m_2^\star)^4 (T - 2)}\,, \quad K_2 \coloneqq 4 K_\Delta \left(\frac{(m_2^\star)^2}{\gamma} + 1\right)\,, \quad K_3 \coloneqq \frac{K_2^2}{2 c_\Gamma}\,.
      \end{equation*}
      Consequently,
      \begin{equation*}
            \int_{t_0}^\infty \Gamma(t)^2 \, dt \leq L_\Gamma \coloneqq \frac{2}{c_\Gamma} \left(\frac{\gamma}{4} + K_3 L_\Delta\right)\,.
      \end{equation*}
\end{lemma}
\begin{proof}
      We first derive a formula for \(\dot{\Gamma}\).
      Since \(R = \sum_i \mu_i^2 r_i^2 / D\) and \(D = \sum_i r_i^2\),
      \begin{equation*}
            \dot{R} = \frac{2}{D} \sum_{i=1}^T \mu_i^2 r_i \dot{r}_i - \frac{R}{D} \cdot 2 \sum_{i=1}^T r_i \dot{r}_i = \frac{2}{D} \sum_{i=1}^T \left(\mu_i^2 - R\right) r_i \dot{r}_i\,.
      \end{equation*}
      By \Cref{lem:mass}, \(\dot{r}_i = r_i (r_i \tilde{g}_i - M)\) for every entry off \(s_1^\star\), and \(r\) vanishes on \(s_1^\star\).
      The part with \(M\) is multiplied by \(\sum_i (\mu_i^2 - R) r_i^2 = 0\), so it drops out.
      Hence,
      \begin{equation*}
            \dot{\Gamma} = - \dot{R} = - \frac{2}{D} \sum_{i=1}^T \left(\mu_i^2 - R\right) r_i^3 \tilde{g}_i\,, \quad \text{where} \quad \tilde{g}_i = (g_{\mathrm{opt}})_i + (g_\Delta)_i - \beta\,.
      \end{equation*}
      We bound the terms on \(O\) and on \(s_2^\star\) separately.
      For this, note that \(|(g_\Delta)_i| \leq \|g_\Delta\|\), so \Cref{lem:split} gives
      \begin{equation*}
            \left|(g_\Delta)_i - \beta\right| \leq |(g_\Delta)_i| + |\beta| \leq 2 K_\Delta \|\Delta\|_F \quad \text{for all } i \in [T]\,.
      \end{equation*}

      We start with the terms on \(O\).
      Inserting \((g_{\mathrm{opt}})_i = r_i (\mu_i^2 - R) / D\) from \Cref{lem:split}, the term for \(i \in O\) is
      \begin{equation*}
            - \frac{2}{D^2} \left(\mu_i^2 - R\right)^2 r_i^4 - \frac{2}{D} \left(\mu_i^2 - R\right) r_i^3 \left((g_\Delta)_i - \beta\right)\,.
      \end{equation*}
      For the first part, \(\mu_i^2 \leq (m_3^\star)^2 = (m_2^\star)^2 - \gamma\), and \(R = (m_2^\star)^2 - \Gamma \geq (m_2^\star)^2 - \gamma / 4\) by (R3).
      Hence, \(\mu_i^2 - R \leq - 3\gamma / 4\), so \((\mu_i^2 - R)^2 \geq 9 \gamma^2 / 16\).
      Moreover, \(\sum_{i \in O} r_i^4 \geq q^2 / (T - 2)\) by the Cauchy--Schwarz inequality, and \(q \geq D \Gamma / (m_2^\star)^2\) by \Cref{lem:gap}.
      Thus, summing the first parts over \(O\),
      \begin{equation*}
            - \frac{2}{D^2} \sum_{i \in O} \left(\mu_i^2 - R\right)^2 r_i^4 \leq - \frac{9 \gamma^2}{8 D^2} \cdot \frac{q^2}{T - 2} \leq - c_\Gamma \Gamma^2\,.
      \end{equation*}
      For the second part, \(|\mu_i^2 - R| \leq (m_2^\star)^2\) and \(r_i^3 \leq r_i^2\).
      Summing over \(O\) and using \(q \leq D \Gamma / \gamma\) from \Cref{lem:gap}, its absolute value is at most
      \begin{equation*}
            \frac{2 (m_2^\star)^2}{D} \, q \cdot 2 K_\Delta \|\Delta\|_F \leq \frac{4 (m_2^\star)^2 K_\Delta}{\gamma} \Gamma \|\Delta\|_F\,.
      \end{equation*}

      Next, we bound the term on \(s_2^\star\).
      Here, \(\mu_i^2 - R = \Gamma\), \(r_i = b\), and \((g_{\mathrm{opt}})_i = b \Gamma / D\).
      Hence, the term is
      \begin{equation*}
            - \frac{2 b^4}{D^2} \Gamma^2 - \frac{2 b^3}{D} \Gamma \left((g_\Delta)_{s_2^\star} - \beta\right) \leq 4 K_\Delta \Gamma \|\Delta\|_F\,,
      \end{equation*}
      where we dropped the first part, which is non-positive, and used \(b^2 \leq D\) and \(b \leq 1\).

      Collecting the terms, \(\dot{\Gamma} \leq - c_\Gamma \Gamma^2 + K_2 \Gamma \|\Delta\|_F\).
      By the inequality \(K_2 \Gamma \|\Delta\|_F \leq \frac{c_\Gamma}{2} \Gamma^2 + \frac{K_2^2}{2 c_\Gamma} \|\Delta\|_F^2\), we get the differential inequality.
      Integrating it over \([t_0, t_1]\) for any \(t_1 > t_0\),
      \begin{equation*}
            \frac{c_\Gamma}{2} \int_{t_0}^{t_1} \Gamma^2 \, dt \leq \Gamma(t_0) - \Gamma(t_1) + K_3 \int_{t_0}^{t_1} \|\Delta\|_F^2 \, dt \leq \frac{\gamma}{4} + K_3 L_\Delta\,,
      \end{equation*}
      by (R3), \(\Gamma \geq 0\), and \Cref{lem:tracking}.
      Letting \(t_1 \to \infty\) gives the claim.
\end{proof}

We now derive a differential inequality for \(\ln a\).
If \(a(0) = 0\), then \(a(t) = 0\) for all \(t \geq 0\) by \Cref{lem:mass}, and there is nothing to prove.
Otherwise, \(a(t) = a(0) \exp(- \int_0^t \Psi \, d\tau) > 0\) for all \(t \geq 0\) by \Cref{lem:mass}, so \(\ln a\) is well defined.
\begin{lemma}\label{lem:ln_a}
      For all \(t \geq t_0\),
      \begin{equation*}
            \frac{d}{dt} \ln a \leq - \frac{c_0}{2} \Gamma + K_1 \|\Delta\|_F^2 + 2 v_{\max} |\Delta_2|\,.
      \end{equation*}
\end{lemma}
\begin{proof}
      By \Cref{lem:a_equation}~(ii), \(\frac{d}{dt} \ln a = \dot{a} / a = - \tilde{\Psi} + (b + \kappa) D v_2 \Delta_2\).
      By \Cref{cor:psi_tilde}, \(- \tilde{\Psi} \leq - \frac{c_0}{2} \Gamma + K_1 \|\Delta\|_F^2\).
      For the last term, \(|b + \kappa| \leq 2\) as in the proof of \Cref{cor:psi_tilde}, \(D \leq 1\), and \(|v_2| \leq \|V'\|_F \leq v_{\max}\).
      Hence, \(|(b + \kappa) D v_2 \Delta_2| \leq 2 v_{\max} |\Delta_2|\), which gives the claim.
\end{proof}

Conversely, the gap \(\Gamma\) cannot decay faster than exponentially at a rate controlled by \(\Gamma\), \(|\Delta_2|\), and \(\|\Delta\|_F^2\).
\begin{lemma}\label{lem:ln_gamma}
      Assume that \(r_i(0) > 0\) for some \(i \in O\).
      Then, for all \(t \geq t_0\),
      \begin{equation*}
            \frac{d}{dt} \ln \Gamma \geq - K_5 \left(\Gamma + |\Delta_2| + \|\Delta\|_F^2\right)\,,
      \end{equation*}
      where \(K_4 \coloneqq \max\{3 m_2^\star + v_{\max}, m_2^\star + 2 m_2^\star T / \delta^2, 2\}\) and
      \begin{equation*}
            K_5 \coloneqq \frac{2 (m_2^\star)^4}{\gamma^2} + 2 + 2 K_4 + \frac{2 K_4}{\sqrt{\gamma}} + \frac{4 (m_2^\star)^2 K_\Delta}{\gamma^{3/2}}\,.
      \end{equation*}
\end{lemma}
\begin{proof}
      By \Cref{lem:gap}, \(\Gamma > 0\), so \(\ln \Gamma\) is well defined.
      We bound each term of the formula \(\dot{\Gamma} = - \frac{2}{D} \sum_i (\mu_i^2 - R) r_i^3 \tilde{g}_i\) from the proof of \Cref{lem:gap_l2} from below.

      We start with the terms on \(O\).
      For \(i \in O\), we split \(\tilde{g}_i = (g_{\mathrm{opt}})_i + ((g_\Delta)_i - \beta)\) and insert \((g_{\mathrm{opt}})_i = r_i (\mu_i^2 - R) / D\) from \Cref{lem:split}.
      The term for \(i \in O\) becomes
      \begin{equation*}
            - \frac{2}{D^2} \left(\mu_i^2 - R\right)^2 r_i^4 - \frac{2}{D} \left(\mu_i^2 - R\right) r_i^3 \left((g_\Delta)_i - \beta\right)\,.
      \end{equation*}
      In \Cref{lem:gap_l2}, we bounded these two parts from above; here, we bound them from below.
      For the first part, \((\mu_i^2 - R)^2 \leq (m_2^\star)^4\), and \(\sum_{i \in O} r_i^4 \leq q^2 \leq (D \Gamma / \gamma)^2\) by \Cref{lem:gap}.
      Hence, summing over \(O\),
      \begin{equation*}
            - \frac{2}{D^2} \sum_{i \in O} \left(\mu_i^2 - R\right)^2 r_i^4 \geq - \frac{2 (m_2^\star)^4}{\gamma^2} \Gamma^2\,.
      \end{equation*}
      For the second part, \(\sum_{i \in O} r_i^3 \leq \left(\max_{i \in O} r_i\right) q \leq q^{3/2} \leq (D \Gamma / \gamma)^{3/2}\).
      Using \(|(g_\Delta)_i - \beta| \leq 2 K_\Delta \|\Delta\|_F\) from the proof of \Cref{lem:gap_l2} and \(D^{1/2} \leq 1\), its absolute value is at most
      \begin{equation*}
            \frac{2 (m_2^\star)^2}{D} \sum_{i \in O} r_i^3 \cdot 2 K_\Delta \|\Delta\|_F \leq \frac{4 (m_2^\star)^2 K_\Delta}{\gamma^{3/2}} \Gamma^{3/2} \|\Delta\|_F\,.
      \end{equation*}

      Next, we bound the term on \(s_2^\star\).
      As in the proof of \Cref{lem:gap_l2}, and using \(b^4 / D^2 \leq 1\) and \(b^3 / D \leq b \leq 1\), the term is at least
      \begin{equation*}
            - 2 \Gamma^2 - 2 \Gamma \left|(g_\Delta)_{s_2^\star} - \beta\right|\,.
      \end{equation*}
      The key point is that \((g_\Delta)_{s_2^\star} - \beta\) is controlled by \(|\Delta_2|\) up to small terms.
      By the formula for \(g_\Delta\) in the proof of \Cref{lem:split}, its entry on \(s_2^\star\) is
      \begin{equation*}
            (g_\Delta)_{s_2^\star} = \langle \Delta, \mG_{(1)} s_2^\star\rangle - b \left(\nu - \|V'_{\mathrm{opt}}\|_F^2\right) = m_2^\star \Delta_2 - b \left(\nu - \|V'_{\mathrm{opt}}\|_F^2\right)\,.
      \end{equation*}
      Substituting \Cref{lem:energy}~(i) and~(ii),
      \begin{equation*}
            (g_\Delta)_{s_2^\star} - \beta = \left(m_2^\star - \frac{2 m_2^\star b^2}{D} + D v_2\right) \Delta_2 - b \, \varepsilon_\nu - \varepsilon_\beta\,.
      \end{equation*}
      By the triangle inequality, the coefficient of \(\Delta_2\) is at most \(m_2^\star + 2 m_2^\star + v_{\max} = 3 m_2^\star + v_{\max}\) in absolute value, since \(b^2 \leq D \leq 1\) and \(|v_2| \leq \|V'\|_F \leq v_{\max}\).
      By the bounds in \Cref{lem:energy} and \(D \geq \delta^2 / T\) from \Cref{lem:apriori}~(iii),
      \begin{equation*}
            |b \, \varepsilon_\nu| + |\varepsilon_\beta| \leq \left(\frac{2 m_2^\star T}{\delta^2} + m_2^\star\right) \sqrt{q} \, \|\Delta\|_F + 2 \|\Delta\|_F^2\,.
      \end{equation*}
      Hence,
      \begin{equation*}
            \left|(g_\Delta)_{s_2^\star} - \beta\right| \leq K_4 \left(|\Delta_2| + \sqrt{q} \, \|\Delta\|_F + \|\Delta\|_F^2\right)\,.
      \end{equation*}

      Collecting the terms and dividing by \(\Gamma > 0\),
      \begin{equation*}
            \frac{d}{dt} \ln \Gamma \geq - \left(\frac{2 (m_2^\star)^4}{\gamma^2} + 2\right) \Gamma - 2 K_4 \left(|\Delta_2| + \sqrt{q} \, \|\Delta\|_F + \|\Delta\|_F^2\right) - \frac{4 (m_2^\star)^2 K_\Delta}{\gamma^{3/2}} \sqrt{\Gamma} \, \|\Delta\|_F\,.
      \end{equation*}
      It remains to bound the two terms with \(\sqrt{q} \, \|\Delta\|_F\) and \(\sqrt{\Gamma} \, \|\Delta\|_F\).
      We use \(\sqrt{q} \leq \sqrt{\Gamma / \gamma}\) from \Cref{lem:gap} and \(\sqrt{\Gamma} \, \|\Delta\|_F \leq \Gamma + \|\Delta\|_F^2\), so that
      \begin{equation*}
            2 K_4 \sqrt{q} \, \|\Delta\|_F \leq \frac{2 K_4}{\sqrt{\gamma}} \left(\Gamma + \|\Delta\|_F^2\right)\,, \quad \frac{4 (m_2^\star)^2 K_\Delta}{\gamma^{3/2}} \sqrt{\Gamma} \, \|\Delta\|_F \leq \frac{4 (m_2^\star)^2 K_\Delta}{\gamma^{3/2}} \left(\Gamma + \|\Delta\|_F^2\right)\,.
      \end{equation*}
      Substituting these bounds, the coefficients of \(\Gamma\), \(|\Delta_2|\), and \(\|\Delta\|_F^2\) are
      \begin{equation*}
            \frac{2 (m_2^\star)^4}{\gamma^2} + 2 + \frac{2 K_4}{\sqrt{\gamma}} + \frac{4 (m_2^\star)^2 K_\Delta}{\gamma^{3/2}}\,, \quad 2 K_4\,, \quad 2 K_4 + \frac{2 K_4}{\sqrt{\gamma}} + \frac{4 (m_2^\star)^2 K_\Delta}{\gamma^{3/2}}\,,
      \end{equation*}
      respectively, and each of them is at most \(K_5\).
      This gives the claim.
\end{proof}

By \Cref{lem:ln_a,lem:ln_gamma}, it remains to show that \(\Delta_2\) is integrable.
For this, we derive an equation for \(\Delta_2\).
Let \(f \coloneqq \sum_{i \in O} r_i\) be the attention on the remaining positions, so that \(a + b + f = 1\).
We enlarge \(t_0\) so that, in addition to (R1)--(R3), for all \(t \geq t_0\),
\begin{equation*}
      \text{(R4)} \enspace f \leq \frac{b}{2}\,, \qquad \text{(R5)} \enspace v_2 \geq \frac{m_2^\star}{2}\,.
\end{equation*}
This is possible by \Cref{lem:localization}.
For (R4), \(f \leq \sqrt{T - 2} \sqrt{q} \to 0\) by the Cauchy--Schwarz inequality, while \(b \geq b_{\min}\).
For (R5), \(v_2 = \Delta_2 + m_2^\star b / D\), where \(\Delta_2 \to 0\) and \(b / D = b / (b^2 + q) \geq b / (b + q) \to 1\), since \(b \leq 1\), \(b \geq b_{\min}\), and \(q \to 0\).
\begin{lemma}\label{lem:delta2}
      Let \(w \coloneqq (b^2 - q) / D^2\) and \(\lambda \coloneqq D + m_2^\star w \Lambda\).
      Then,
      \begin{equation*}
            \dot{\Delta}_2 = - \lambda \Delta_2 + m_2^\star w \, a \tilde{\Psi} - m_2^\star w \dot{f} + \frac{m_2^\star b}{D^2} \dot{q}\,.
      \end{equation*}
      Moreover, for all \(t \geq t_0\),
      \begin{equation*}
            b + \kappa \geq 0\,, \quad 0 < w \leq w_{\max} \coloneqq \frac{1}{b_{\min}^2}\,, \quad 0 \leq \Lambda \leq \Lambda_{\max} \coloneqq 2 v_{\max}\,, \quad \lambda \geq D \geq \frac{\delta^2}{T}\,.
      \end{equation*}
\end{lemma}
\begin{proof}
      We first derive the equation.
      By definition, \(\Delta_2 = v_2 - m_2^\star b / D\), and \(\dot{v}_2 = - D \Delta_2\).
      Since \(D = b^2 + q\), we have \(\dot{D} = 2 b \dot{b} + \dot{q}\), so
      \begin{equation*}
            \frac{d}{dt} \frac{b}{D} = \frac{\dot{b}}{D} - \frac{b \dot{D}}{D^2} = \dot{b} \, \frac{D - 2 b^2}{D^2} - \frac{b \dot{q}}{D^2} = - w \dot{b} - \frac{b \dot{q}}{D^2}\,,
      \end{equation*}
      where we used \(D - 2b^2 = q - b^2\).
      Hence,
      \begin{equation*}
            \dot{\Delta}_2 = - D \Delta_2 + m_2^\star w \dot{b} + \frac{m_2^\star b}{D^2} \dot{q}\,.
      \end{equation*}
      Since \(a + b + f = 1\), we have \(\dot{b} = - \dot{a} - \dot{f}\), and \(\dot{a} = - a \tilde{\Psi} + \Lambda \Delta_2\) by \Cref{lem:a_equation}~(ii).
      Substituting \(\dot{b} = a \tilde{\Psi} - \Lambda \Delta_2 - \dot{f}\) gives the equation.

      Next, we prove the bounds for \(t \geq t_0\).
      Since \(q = \sum_{i \in O} r_i^2 \leq f^2\), (R4) gives \(q \leq b^2 / 4 < b^2\), so \(w > 0\).
      Since \(D \geq b^2\), we have \(w \leq b^2 / D^2 \leq 1 / b^2 \leq w_{\max}\) by (R1).
      For \(\Lambda = a (b + \kappa) D v_2\), we first show \(b + \kappa \geq 0\).
      Substituting \(1 - a = b + f\) and \(D = b^2 + q\) into \(\kappa = a(1 - a) - D\), and then using \(1 - b = a + f\),
      \begin{equation*}
            b + \kappa = b (1 - b) + a b + a f - q = 2ab + f (a + b) - q \geq 2ab + f (a + b - f) \geq 0\,,
      \end{equation*}
      where we used \(q \leq f^2\) and \(f \leq b\) from (R4).
      Moreover, \(b + \kappa \leq 2\) as in the proof of \Cref{cor:psi_tilde}, \(0 \leq a \leq 1\), \(D \leq 1\), and \(0 < v_2 \leq v_{\max}\) by (R5).
      Hence, \(0 \leq \Lambda \leq 2 v_{\max}\).
      Finally, \(\lambda = D + m_2^\star w \Lambda \geq D \geq \delta^2 / T\) by \Cref{lem:apriori}~(iii).
\end{proof}

The forcing terms \(\dot{f}\) and \(\dot{q}\) in \Cref{lem:delta2} have bounded total variation.
To show this, we enlarge \(t_0\) once more so that, for all \(t \geq t_0\),
\begin{equation*}
      \text{(R6)} \enspace \|\Delta\|_F \leq 1\,,
\end{equation*}
which is possible since \(\Delta \to 0\) by \Cref{lem:localization}.
We also define the smallness parameter
\begin{equation*}
      \epsilon_0 \coloneqq \sup_{t \geq t_0} \max\left\{f(t), |\Delta_2(t)|\right\}\,,
\end{equation*}
which tends to \(0\) as \(t_0 \to \infty\), since \(f \to 0\) and \(\Delta_2 \to 0\) by \Cref{lem:localization}.
For \(t_1 > t_0\) and a function \(y\), we write \(\mathrm{TV}(y) \coloneqq \int_{t_0}^{t_1} |\dot{y}| \, dt\) for its total variation on \([t_0, t_1]\), and \(X \coloneqq \int_{t_0}^{t_1} |\Delta_2| \, dt\).
\begin{lemma}\label{lem:total_variation}
      Let
      \begin{equation*}
            K_6 \coloneqq 2 \left(2 K_\Delta + 3 v_{\max} + m_2^\star + K_1 + 1\right)\,, \quad K_7 \coloneqq \frac{2 K_\Delta + m_2^\star \sqrt{T - 2}}{2 \gamma} + K_1 + 1\,.
      \end{equation*}
      Then, for all \(t_1 > t_0\),
      \begin{equation*}
            \mathrm{TV}(q) \leq C_q \coloneqq 1 + \frac{K_6}{\gamma} (L_\Gamma + L_\Delta)\,, \quad \mathrm{TV}(f) \leq C_f + 6 v_{\max} \epsilon_0 X\,, \quad \text{where} \quad C_f \coloneqq 1 + 2 K_7 (L_\Gamma + L_\Delta)\,.
      \end{equation*}
\end{lemma}
\begin{proof}
      We first reduce the total variation to the increasing part.
      For a non-negative function \(y\), we have \(\int_{t_0}^{t_1} \dot{y}^- = \int_{t_0}^{t_1} \dot{y}^+ - (y(t_1) - y(t_0))\), where \(\dot{y}^\pm = \max\{\pm \dot{y}, 0\}\).
      Hence,
      \begin{equation*}
            \mathrm{TV}(y) = 2 \int_{t_0}^{t_1} \dot{y}^+ \, dt - y(t_1) + y(t_0) \leq 2 \int_{t_0}^{t_1} \dot{y}^+ \, dt + y(t_0)\,.
      \end{equation*}
      We apply this to \(q\) and \(f\), with \(q(t_0), f(t_0) \leq 1\).

      Next, we bound the increasing parts of \(q\) and \(f\).
      By \Cref{lem:mass}, \(\dot{r}_i = r_i (r_i \tilde{g}_i - M)\) for \(i \in O\), so
      \begin{equation*}
            \dot{q} = 2 \sum_{i \in O} r_i^2 \left(r_i \tilde{g}_i - M\right)\,, \quad \dot{f} = \sum_{i \in O} r_i \left(r_i \tilde{g}_i - M\right)\,.
      \end{equation*}
      Since \(r_i \tilde{g}_i - M \leq r_i \tilde{g}_i^+ + M^-\) and \(r_i \leq 1\), we get
      \begin{equation*}
            \dot{q}^+ \leq 2 q \left(\max_{i \in O} \tilde{g}_i^+ + M^-\right)\,, \quad \dot{f}^+ \leq q \max_{i \in O} \tilde{g}_i^+ + f M^-\,.
      \end{equation*}

      We now bound \(\tilde{g}_i^+\) for \(i \in O\).
      We have \(\tilde{g}_i = (g_{\mathrm{opt}})_i + (g_\Delta)_i - \beta\), and \((g_{\mathrm{opt}})_i = r_i (\mu_i^2 - R) / D \leq 0\), since \(\mu_i^2 - R \leq - 3\gamma / 4\) as in the proof of \Cref{lem:gap_l2}.
      Hence, \(\tilde{g}_i^+ \leq |(g_\Delta)_i - \beta| \leq 2 K_\Delta \|\Delta\|_F\), where the last inequality is from the proof of \Cref{lem:gap_l2}.

      Next, we bound \(M^-\).
      By \Cref{lem:mass}, \(M = \Psi - a \beta\), so \(M^- \leq \Psi^- + |\beta|\).
      By \Cref{lem:a_equation}~(ii) and \Cref{cor:psi_tilde}, \(\Psi^- \leq \tilde{\Psi}^- + |(b + \kappa) D v_2 \Delta_2| \leq K_1 \|\Delta\|_F^2 + 2 v_{\max} |\Delta_2|\).
      By \Cref{lem:energy}~(ii), \(|\beta| \leq v_{\max} |\Delta_2| + m_2^\star \sqrt{q} \, \|\Delta\|_F + \|\Delta\|_F^2\).
      Hence,
      \begin{equation*}
            M^- \leq (K_1 + 1) \|\Delta\|_F^2 + 3 v_{\max} |\Delta_2| + m_2^\star \sqrt{q} \, \|\Delta\|_F\,.
      \end{equation*}

      We now bound the total variation of \(q\).
      Using \(|\Delta_2| \leq \|\Delta\|_F \leq 1\) by (R6) and \(\sqrt{q} \leq 1\), we get \(\max_{i \in O} \tilde{g}_i^+ + M^- \leq \frac{K_6}{2} \|\Delta\|_F\).
      Hence, using \(q \leq \Gamma / \gamma\) from \Cref{lem:gap} and \(\Gamma \|\Delta\|_F \leq \frac{1}{2}(\Gamma^2 + \|\Delta\|_F^2)\),
      \begin{equation*}
            \dot{q}^+ \leq K_6 \, q \, \|\Delta\|_F \leq \frac{K_6}{\gamma} \Gamma \|\Delta\|_F \leq \frac{K_6}{2\gamma} \left(\Gamma^2 + \|\Delta\|_F^2\right)\,.
      \end{equation*}
      Integrating and using \Cref{lem:tracking,lem:gap_l2} gives \(\int_{t_0}^{t_1} \dot{q}^+ \leq \frac{K_6}{2 \gamma} (L_\Gamma + L_\Delta)\), and the reduction above gives \(\mathrm{TV}(q) \leq C_q\).

      Finally, we bound the total variation of \(f\).
      For the first term, \(q \max_{i \in O} \tilde{g}_i^+ \leq 2 K_\Delta \, q \, \|\Delta\|_F \leq \frac{2 K_\Delta}{\gamma} \Gamma \|\Delta\|_F\).
      For \(f M^-\), we bound \(f\) differently in each term of \(M^-\).
      In the term with \(\|\Delta\|_F^2\), we use \(f \leq 1\).
      In the term with \(|\Delta_2|\), we use \(f \leq \epsilon_0\).
      In the term with \(\sqrt{q} \, \|\Delta\|_F\), we use \(f \leq \sqrt{T - 2} \sqrt{q}\), so that \(f \sqrt{q} \leq \sqrt{T - 2} \, q \leq \sqrt{T - 2} \, \Gamma / \gamma\).
      Hence,
      \begin{equation*}
            \dot{f}^+ \leq \frac{2 K_\Delta + m_2^\star \sqrt{T - 2}}{\gamma} \Gamma \|\Delta\|_F + (K_1 + 1) \|\Delta\|_F^2 + 3 v_{\max} \epsilon_0 |\Delta_2| \leq K_7 \left(\Gamma^2 + \|\Delta\|_F^2\right) + 3 v_{\max} \epsilon_0 |\Delta_2|\,.
      \end{equation*}
      Integrating, using \Cref{lem:tracking,lem:gap_l2}, and the reduction above gives \(\mathrm{TV}(f) \leq C_f + 6 v_{\max} \epsilon_0 X\).
\end{proof}

We now separate the forcing in \Cref{lem:delta2} into a main part and a small part.
Writing \(a \tilde{\Psi} = a \tilde{\Psi}^+ - a \tilde{\Psi}^-\), where \(\tilde{\Psi}^+ \coloneqq \max\{\tilde{\Psi}, 0\}\), we get
\begin{equation*}
      \dot{\Delta}_2 = - \lambda \Delta_2 + m_2^\star w \, a \tilde{\Psi}^+ + F\,, \quad \text{where} \quad F \coloneqq - m_2^\star w \, a \tilde{\Psi}^- - m_2^\star w \dot{f} + \frac{m_2^\star b}{D^2} \dot{q}\,.
\end{equation*}
We also write \(D_{\min} \coloneqq \delta^2 / T\), so that \(D \geq D_{\min}\) by \Cref{lem:apriori}~(iii), and
\begin{equation*}
      P \coloneqq \int_{t_0}^{t_1} a \tilde{\Psi}^+ \, dt
\end{equation*}
for the main forcing.
Note that \(\epsilon_0 \leq 1\), since \(f \leq 1\) and \(|\Delta_2| \leq \|\Delta\|_F \leq 1\) by (R6).
\begin{lemma}\label{lem:forcing}
      For all \(t_1 > t_0\),
      \begin{equation*}
            \int_{t_0}^{t_1} |F| \, dt \leq B_3 + B_4 \epsilon_0 X\,, \quad \text{where} \quad B_3 \coloneqq m_2^\star w_{\max} \left(K_1 L_\Delta + C_f\right) + \frac{m_2^\star C_q}{D_{\min}^2}\,, \quad B_4 \coloneqq 6 m_2^\star w_{\max} v_{\max}\,.
      \end{equation*}
\end{lemma}
\begin{proof}
      We bound the three terms of \(F\) separately, using \(0 < w \leq w_{\max}\) from \Cref{lem:delta2}.
      For the first, \(a \leq 1\) and \(\tilde{\Psi}^- \leq K_1 \|\Delta\|_F^2\) by \Cref{cor:psi_tilde}, so its integral is at most \(m_2^\star w_{\max} K_1 L_\Delta\) by \Cref{lem:tracking}.
      For the second, its integral is at most \(m_2^\star w_{\max} \mathrm{TV}(f) \leq m_2^\star w_{\max} (C_f + 6 v_{\max} \epsilon_0 X)\) by \Cref{lem:total_variation}.
      For the third, \(b \leq 1\) and \(D \geq D_{\min}\), so its integral is at most \(m_2^\star \mathrm{TV}(q) / D_{\min}^2 \leq m_2^\star C_q / D_{\min}^2\) by \Cref{lem:total_variation}.
      Adding the three bounds gives the claim.
\end{proof}

In the proof of the main proposition, we multiply the equation for \(\Delta_2\) in \Cref{lem:delta2} by \(\varrho \coloneqq \Lambda / \lambda\).
Since \(\varrho \lambda = \Lambda\), this gives
\begin{equation}\label{eq:weighted}
      \Lambda \Delta_2 = - \varrho \dot{\Delta}_2 + \theta \, a \tilde{\Psi}^+ + \varrho F\,, \quad \text{where} \quad \theta \coloneqq m_2^\star w \varrho\,.
\end{equation}
The left-hand side is the term linear in \(\Delta_2\) in the equation for \(a\) from \Cref{lem:a_equation}.
Thus, \(\theta\) is the fraction of \(a \tilde{\Psi}^+\) that returns to the equation for \(a\) through \(\Delta_2\).
The next lemma shows that \(\theta \leq \theta_{\max} < 1\), and bounds the total variation of \(\varrho\).
\begin{lemma}\label{lem:weight}
      For all \(t \geq t_0\),
      \begin{equation*}
            0 \leq \varrho \leq \varrho_{\max} \coloneqq \frac{\Lambda_{\max}}{D_{\min}}\,, \quad 0 \leq \theta \leq \theta_{\max} \coloneqq \frac{m_2^\star w_{\max} \Lambda_{\max}}{D_{\min} + m_2^\star w_{\max} \Lambda_{\max}} < 1\,.
      \end{equation*}
      Moreover, there exists a constant \(L_\varrho\), depending only on \(m_2^\star\), \(\delta\), \(T\), and \(\|V'(0)\|_F\), such that, for all \(t_1 > t_0\),
      \begin{equation*}
            \mathrm{TV}(\varrho) \leq L_\varrho \left(2P + B_1 + B_2 X\right)\,, \quad \text{where} \quad B_1 \coloneqq 2 K_1 L_\Delta + C_f + C_q\,, \quad B_2 \coloneqq 2 \Lambda_{\max} + 6 v_{\max} + 1\,.
      \end{equation*}
\end{lemma}
\begin{proof}
      By \Cref{lem:delta2}, \(0 \leq \Lambda \leq \Lambda_{\max}\) and \(\lambda \geq D \geq D_{\min}\), so \(0 \leq \varrho \leq \varrho_{\max}\).
      Moreover, by the definitions of \(\varrho\) and \(\lambda = D + m_2^\star w \Lambda\),
      \begin{equation*}
            \theta = \frac{m_2^\star w \Lambda}{\lambda} = \frac{m_2^\star w \Lambda}{D + m_2^\star w \Lambda}\,.
      \end{equation*}
      This is non-decreasing in \(m_2^\star w \Lambda \geq 0\) and non-increasing in \(D\).
      Since \(m_2^\star w \Lambda \leq m_2^\star w_{\max} \Lambda_{\max}\) and \(D \geq D_{\min}\), we get \(0 \leq \theta \leq \theta_{\max} < 1\).

      For the total variation, note that \(\varrho\) is a function of \(x \coloneqq (a, b, q, v_2)\).
      Indeed, \(D = b^2 + q\), \(\kappa = a(1 - a) - D\), \(\Lambda = a (b + \kappa) D v_2\), \(w = (b^2 - q) / D^2\), and \(\lambda = D + m_2^\star w \Lambda\).
      Consider the set
      \begin{equation*}
            \mathcal{K} \coloneqq \left\{x : a \in [0, 1],\ b \in [b_{\min}, 1],\ q \in [0, b^2 / 4],\ v_2 \in [m_2^\star / 2, v_{\max}],\ b + \kappa \geq 0\right\}\,.
      \end{equation*}
      It is compact, since it is closed and bounded.
      For \(t \geq t_0\), the trajectory \(x(t)\) lies in \(\mathcal{K}\): we have \(a, b \leq 1\) since \(s'\) lies on the simplex, \(b \geq b_{\min}\) by (R1), \(q \leq f^2 \leq b^2 / 4\) by (R4), \(m_2^\star / 2 \leq v_2 \leq v_{\max}\) by (R5), and \(b + \kappa \geq 0\) by \Cref{lem:delta2}.
      On \(\mathcal{K}\), we have \(D \geq b_{\min}^2 > 0\), \(w \geq 0\), and \(\Lambda \geq 0\), so \(\lambda \geq D \geq b_{\min}^2 > 0\).
      Hence, \(\varrho\) is continuously differentiable on a neighborhood of \(\mathcal{K}\), and its partial derivatives are bounded on \(\mathcal{K}\) by some \(L_\varrho > 0\).
      Then, \(|\dot{\varrho}| \leq L_\varrho (|\dot{a}| + |\dot{b}| + |\dot{q}| + |\dot{v}_2|)\) along the trajectory, so
      \begin{equation*}
            \mathrm{TV}(\varrho) \leq L_\varrho \left(\mathrm{TV}(a) + \mathrm{TV}(b) + \mathrm{TV}(q) + \mathrm{TV}(v_2)\right)\,.
      \end{equation*}

      We bound the four total variations.
      By \Cref{lem:a_equation}~(ii), \(|\dot{a}| \leq a \tilde{\Psi}^+ + a \tilde{\Psi}^- + \Lambda |\Delta_2|\), so \(\mathrm{TV}(a) \leq P + K_1 L_\Delta + \Lambda_{\max} X\) by \Cref{cor:psi_tilde} and \Cref{lem:tracking}.
      Since \(b = 1 - a - f\), we have \(\mathrm{TV}(b) \leq \mathrm{TV}(a) + \mathrm{TV}(f)\).
      By \Cref{lem:total_variation}, \(\mathrm{TV}(q) \leq C_q\) and \(\mathrm{TV}(f) \leq C_f + 6 v_{\max} \epsilon_0 X\).
      Since \(\dot{v}_2 = - D \Delta_2\) and \(D \leq 1\), we have \(\mathrm{TV}(v_2) \leq X\).
      Adding these bounds and using \(\epsilon_0 \leq 1\),
      \begin{equation*}
            \mathrm{TV}(a) + \mathrm{TV}(b) + \mathrm{TV}(q) + \mathrm{TV}(v_2) \leq 2P + 2 K_1 L_\Delta + C_f + C_q + (2 \Lambda_{\max} + 6 v_{\max} + 1) X\,,
      \end{equation*}
      which gives the claim.
\end{proof}

We can now show that \(\Delta_2\) is integrable.
For every finite \(t_1\), both \(X\) and \(P\) are finite, since their integrands are continuous on \([t_0, t_1]\).
The point of the following proposition is that its bound on \(X\) does not depend on \(t_1\).
\begin{proposition}\label{prop:delta2}
      There exists \(\epsilon_\star > 0\), which does not depend on \(t_0\), such that the following holds.
      If \(\epsilon_0 \leq \epsilon_\star\), then
      \begin{equation*}
            \int_{t_0}^\infty |\Delta_2(t)| \, dt \leq X_\star < \infty\,,
      \end{equation*}
      where \(X_\star\) is given in the proof and does not depend on \(t_0\) either.
\end{proposition}
\begin{proof}
      All constants in this section so far depend only on \(m_2^\star\), \(\gamma\), \(\delta\), \(T\), \(\phi(0)\), and \(\|V'(0)\|_F\), and not on \(t_0\) or \(t_1\).
      Fix \(t_1 > t_0\).
      The proof combines two inequalities.
      The first, from the equation for \(a\), bounds \(P\) by \(\int_{t_0}^{t_1} \Lambda \Delta_2 \, dt\).
      The second, from the equation for \(\Delta_2\), bounds \(X\) by \(P\).
      Bounding \(\int_{t_0}^{t_1} \Lambda \Delta_2 \, dt\) by \(\Lambda_{\max} X\) would close this loop only if \(m_2^\star w_{\max} \Lambda_{\max} < D_{\min}\), which need not hold.
      Instead, we bound it using \Cref{eq:weighted}, which returns only a fraction \(\theta_{\max} < 1\) of \(P\).

      We start with the equation for \(a\).
      Integrating \(\dot{a} = - a \tilde{\Psi}^+ + a \tilde{\Psi}^- + \Lambda \Delta_2\) from \Cref{lem:a_equation}~(ii) over \([t_0, t_1]\), and using \(a \leq 1\), \Cref{cor:psi_tilde}, and \Cref{lem:tracking},
      \begin{equation}\label{eq:p1}
            P = a(t_0) - a(t_1) + \int_{t_0}^{t_1} a \tilde{\Psi}^- \, dt + \int_{t_0}^{t_1} \Lambda \Delta_2 \, dt \leq A_1 + \int_{t_0}^{t_1} \Lambda \Delta_2 \, dt\,, \quad \text{where} \quad A_1 \coloneqq 1 + K_1 L_\Delta\,.
      \end{equation}

      Next, we bound this integral by integrating \Cref{eq:weighted} over \([t_0, t_1]\):
      \begin{equation*}
            \int_{t_0}^{t_1} \Lambda \Delta_2 \, dt = - \int_{t_0}^{t_1} \varrho \dot{\Delta}_2 \, dt + \int_{t_0}^{t_1} \theta \, a \tilde{\Psi}^+ \, dt + \int_{t_0}^{t_1} \varrho F \, dt\,.
      \end{equation*}
      We bound the three integrals on the right-hand side in turn.
      For the first, integration by parts gives
      \begin{equation*}
            - \int_{t_0}^{t_1} \varrho \dot{\Delta}_2 \, dt = - \left[\varrho \Delta_2\right]_{t_0}^{t_1} + \int_{t_0}^{t_1} \Delta_2 \dot{\varrho} \, dt \leq 2 \varrho_{\max} \epsilon_0 + \epsilon_0 \mathrm{TV}(\varrho)\,,
      \end{equation*}
      using \(|\Delta_2| \leq \epsilon_0\) and \(0 \leq \varrho \leq \varrho_{\max}\) on \([t_0, t_1]\).
      For the second, \(\theta \leq \theta_{\max}\) gives \(\int_{t_0}^{t_1} \theta a \tilde{\Psi}^+ \, dt \leq \theta_{\max} P\).
      For the third, \(\int_{t_0}^{t_1} \varrho F \, dt \leq \varrho_{\max} (B_3 + B_4 \epsilon_0 X)\) by \Cref{lem:forcing}.
      Together with \Cref{lem:weight} for \(\mathrm{TV}(\varrho)\),
      \begin{equation*}
            \int_{t_0}^{t_1} \Lambda \Delta_2 \, dt \leq \theta_{\max} P + 2 \varrho_{\max} \epsilon_0 + \epsilon_0 L_\varrho \left(2P + B_1 + B_2 X\right) + \varrho_{\max} \left(B_3 + B_4 \epsilon_0 X\right)\,.
      \end{equation*}

      We now close the loop.
      Inserting this bound into \Cref{eq:p1} and using \(\epsilon_0 \leq 1\),
      \begin{equation*}
            \left(1 - \theta_{\max} - 2 \epsilon_0 L_\varrho\right) P \leq A_2 + \epsilon_0 B_5 X\,, \quad \text{where} \quad A_2 \coloneqq A_1 + 2 \varrho_{\max} + L_\varrho B_1 + \varrho_{\max} B_3\,, \quad B_5 \coloneqq L_\varrho B_2 + \varrho_{\max} B_4\,.
      \end{equation*}
      Here, we used that \(P\) is finite to move \(\theta_{\max} P\) and \(2 \epsilon_0 L_\varrho P\) to the left-hand side.
      This is where \(\theta_{\max} < 1\) is essential: the feedback through \(\Delta_2\) returns at most a fraction \(\theta_{\max}\) of \(P\).
      If \(\epsilon_0 \leq (1 - \theta_{\max}) / (4 L_\varrho)\), then \(1 - \theta_{\max} - 2 \epsilon_0 L_\varrho \geq (1 - \theta_{\max}) / 2\), so
      \begin{equation}\label{eq:p2}
            P \leq \frac{2}{1 - \theta_{\max}} \left(A_2 + \epsilon_0 B_5 X\right)\,.
      \end{equation}

      Next, we bound \(X\) by \(P\), using the damping term \(- \lambda \Delta_2\) in \Cref{lem:delta2}, with \(\lambda \geq D_{\min}\).
      Since \((V', s')\) solves \Cref{eq:second_dynamics}, whose vector field is polynomial, and \(D \geq D_{\min} > 0\), the function \(\Delta_2 = v_2 - m_2^\star b / D\) is continuously differentiable in \(t\).
      Hence, \(|\Delta_2|\) is Lipschitz, and thus absolutely continuous, on \([t_0, t_1]\), with derivative \(\mathrm{sign}(\Delta_2) \dot{\Delta}_2\) almost everywhere on \(\{\Delta_2 \neq 0\}\) and \(0\) almost everywhere on \(\{\Delta_2 = 0\}\).

      On \(\{\Delta_2 \neq 0\}\), we multiply \(\dot{\Delta}_2 = - \lambda \Delta_2 + m_2^\star w \, a \tilde{\Psi}^+ + F\) by \(\mathrm{sign}(\Delta_2)\), and bound the last two terms by their absolute values.
      On \(\{\Delta_2 = 0\}\), the derivative is \(0\), and the right-hand side below is non-negative.
      Thus, almost everywhere,
      \begin{equation*}
            \frac{d}{dt} |\Delta_2| \leq - \lambda |\Delta_2| + m_2^\star w \, a \tilde{\Psi}^+ + |F|\,.
      \end{equation*}
      Integrating over \([t_0, t_1]\), and using \(\lambda \geq D_{\min}\), \(w \leq w_{\max}\), \(|\Delta_2(t_0)| \leq 1\), and \Cref{lem:forcing},
      \begin{equation*}
            D_{\min} X \leq |\Delta_2(t_0)| - |\Delta_2(t_1)| + m_2^\star w_{\max} P + \int_{t_0}^{t_1} |F| \, dt \leq 1 + m_2^\star w_{\max} P + B_3 + B_4 \epsilon_0 X\,.
      \end{equation*}

      Finally, inserting \Cref{eq:p2} gives
      \begin{equation*}
            D_{\min} X \leq A_3 + \epsilon_0 B_6 X\,, \quad \text{where} \quad A_3 \coloneqq 1 + B_3 + \frac{2 m_2^\star w_{\max}}{1 - \theta_{\max}} A_2\,, \quad B_6 \coloneqq B_4 + \frac{2 m_2^\star w_{\max}}{1 - \theta_{\max}} B_5\,.
      \end{equation*}
      Let
      \begin{equation*}
            \epsilon_\star \coloneqq \min\left\{\frac{1 - \theta_{\max}}{4 L_\varrho}, \frac{D_{\min}}{2 B_6}\right\}\,, \quad X_\star \coloneqq \frac{2 A_3}{D_{\min}}\,.
      \end{equation*}
      These are built from the constants above, so neither \(\epsilon_\star\) nor \(X_\star\) depends on \(t_0\) or \(t_1\).
      If \(\epsilon_0 \leq \epsilon_\star\), then, since \(X\) is finite, \(\frac{D_{\min}}{2} X \leq A_3\), that is, \(X \leq X_\star\).
      Since \(X_\star\) does not depend on \(t_1\), letting \(t_1 \to \infty\) gives the claim.
\end{proof}

Since \(\epsilon_0 \to 0\) as \(t_0 \to \infty\) and \(\epsilon_\star\) does not depend on \(t_0\), we can enlarge \(t_0\) so that \(\epsilon_0 \leq \epsilon_\star\).
Hence, \(\Delta_2\) is integrable on \([t_0, \infty)\).

We can now conclude that the attention of the second head moves entirely to the second feature.
\begin{proposition}\label{prop:attention}
      Under the assumptions of \Cref{thm:cooperative},
      \begin{equation*}
            \lim_{t \to \infty} a(t) = 0\,, \quad \lim_{t \to \infty} V'(t) = m_2^\star V_2^\star\,, \quad \lim_{t \to \infty} s'(t) = s_2^\star\,.
      \end{equation*}
\end{proposition}
\begin{proof}
      The assumption \(\langle s'(0), s_j^\star\rangle > 0\) for some \(j \in [3, T]\) means that \(r_i(0) > 0\) for some \(i \in O\), as required by \Cref{lem:gap,lem:ln_gamma}.
      We fix \(t_0\) such that (R1)--(R6) hold and \(\epsilon_0 \leq \epsilon_\star\).
      Then, \(\int_{t_0}^\infty |\Delta_2| \, dt \leq X_\star\) by \Cref{prop:delta2}, and \(\int_{t_0}^\infty \|\Delta\|_F^2 \, dt \leq L_\Delta\) by \Cref{lem:tracking}.

      First, we show that \(\int_{t_0}^\infty \Gamma \, dt = \infty\).
      Suppose, for contradiction, that this integral is finite.
      Integrating \Cref{lem:ln_gamma} over \([t_0, t]\) gives, for all \(t \geq t_0\),
      \begin{equation*}
            \ln \Gamma(t) \geq \ln \Gamma(t_0) - K_5 \left(\int_{t_0}^\infty \Gamma \, d\tau + X_\star + L_\Delta\right)\,.
      \end{equation*}
      The right-hand side is finite and does not depend on \(t\), since \(\Gamma(t_0) > 0\) by \Cref{lem:gap}.
      Hence, \(\Gamma\) is bounded from below by a positive constant on \([t_0, \infty)\).
      This contradicts \(\Gamma(t) \to 0\) from \Cref{lem:localization}.

      Next, we show that \(a(t) \to 0\).
      If \(a(0) = 0\), then \(a(t) = 0\) for all \(t \geq 0\).
      Otherwise, integrating \Cref{lem:ln_a} over \([t_0, t]\) gives
      \begin{equation*}
            \ln a(t) \leq \ln a(t_0) + K_1 L_\Delta + 2 v_{\max} X_\star - \frac{c_0}{2} \int_{t_0}^t \Gamma \, d\tau\,.
      \end{equation*}
      The right-hand side tends to \(- \infty\) as \(t \to \infty\) by the first step, so \(a(t) \to 0\).

      Finally, we derive the limits of \(s'\) and \(V'\).
      By \Cref{lem:localization}, \(q \to 0\), so \(f \leq \sqrt{T - 2} \sqrt{q} \to 0\).
      Since \(a + b + f = 1\) and \(a \to 0\), we get \(b = \langle s_2^\star, s' \rangle \to 1\), that is, \(s' \to s_2^\star\).
      In particular, \(r = s' - a s_1^\star \to s_2^\star\) and \(D \to 1\), so \(V'_{\mathrm{opt}} = \mG_{(1)} r / D \to \mG_{(1)} s_2^\star = m_2^\star V_2^\star\).
      Since \(\Delta = V' - V'_{\mathrm{opt}} \to 0\) by \Cref{lem:localization}, we get \(V' \to m_2^\star V_2^\star\).
\end{proof}

We can now prove \Cref{thm:cooperative}.
\cooperative*
\begin{proof}
      By \Cref{prop:family}, the trajectory converges to the family of stationary points, and \(V'(t) \otimes s'_{(1)}(t) \to m_2^\star V_2^\star \otimes s_2^\star\).
      This identifies the limit of the product, but not of \(V'\) and \(s'\) individually.
      The limits \(V'(t) \to m_2^\star V_2^\star\) and \(s'(t) \to s_2^\star\) follow from \Cref{prop:attention}, which uses the assumption \(\langle s'(0), s_j^\star\rangle > 0\) for some \(j \in [3, T]\).
\end{proof}

\subsection{Extension to Higher-order Heads}
\label{app:higher}

As in \Cref{app:cooperation}, we study the offshoot of an arbitrary head \(n > 2\) after the system has learned the first \(n-1\) features.
Each feature \(k \in [2, n-1]\) is learned by its own head \(k\), head \(n\) is the offshooting head, and the remaining heads \(1, n+1, \ldots, \nheads\) are still aligned with the first feature.
We consider the following configuration:
\begin{equation*}
      \begin{split}
            V_1 &= V_{n+1} = \ldots = V_{\nheads} = V\,, \\
            s_1 &= s_{n+1} = \ldots = s_{\nheads} = s_1^\star\,, \\
            s_k &= s_k^\star\,, \quad \forall k \in [2, n-1]\,.
      \end{split}
\end{equation*}
That is, the attention of each head \(k \in [2, n-1]\) has moved entirely to its feature, as in the limit of \Cref{thm:cooperative}.
We assume an analog of the initialization in \Cref{eq:fast_init}:
\begin{equation*}
\begin{split}
      V_k(0) &= m_k^\star V_k^\star\,, \quad \langle s_k^\star, s_{n}(0)\rangle = 0\,, \quad \forall k \in [2, n-1]\,, \\
      V(0) &= \dfrac{1}{\nheads-n+1} \left(m_1^\star V_1^\star - \langle s_1^\star, s_{n}(0)\rangle V_{n}(0)\right)\,, \\
      V_{n}(0) &\approx V(0)\,, \quad s_{n}(0) \approx s_1^\star\,.
\end{split}
\end{equation*}
Here, we note that \(\dot{V}(0) = 0\) and \(\dot{V}_k(0) = 0\) for all \(k \in [2, n-1]\).
That is, \(V(0)\) and \(V_k(0)\) are at their optimal values given \(V_n(0)\) and \(s_n(0)\).

The heads \(k \in [2, n-1]\) stay fixed for all \(t \geq 0\).
Indeed, consider the states with
\begin{equation*}
      V_k = m_k^\star V_k^\star\,, \quad s_k = s_k^\star\,, \quad \langle s_k^\star, s_n\rangle = 0\,, \quad \forall k \in [2, n-1]\,.
\end{equation*}
We show that the time derivatives of all three quantities vanish on these states.
First, \(\Pi(s_k^\star) = 0\) since \(s_k^\star\) is one-hot, so \(\dot{s}_k = 0\).
Second, the \(i\)-th entry of \(\Pi(s) x\) is \(s_i (x_i - \langle s, x\rangle)\) for any \(x\), so each entry of \(\dot{s}_n\) is proportional to the corresponding entry of \(s_n\).
Hence, \(\frac{d}{dt} \langle s_k^\star, s_n\rangle = 0\).
Third, \(\dot{V}_k = (\mG - \mP) s_k^\star\).
Among all heads, only head \(k\) attends to \(s_k^\star\), so \((\mG - \mP) s_k^\star = m_k^\star V_k^\star - V_k = 0\).
By uniqueness of solutions, the trajectory then stays on these states.

As in \Cref{app:second}, we assume that \(V\) has fast dynamics, by an argument similar to \Cref{lem:optimality}, so that it stays at its optimal value throughout.
Then, the columns of \(\mG - \mP\) along \(s_1^\star, \ldots, s_{n-1}^\star\) vanish, and at every other position, only head \(n\) contributes to \(\mP\).
Writing \(V' = V_n\) and \(s' = s_n\) for brevity, the dynamics of head \(n\) become:
\begin{equation*}
      \begin{split}
            \dot{V}'(t) &= \mG_{(n-1)} s_{(n-1)}'(t) - \|s_{(n-1)}'(t)\|^2 V'(t)\,, \\
            \dot{s}'(t) &= \Pi(s'(t))^2 \left(V'(t)^\top \mG_{(n-1)} - \|V'(t)\|_F^2 s_{(n-1)}'(t)\right)\,.
      \end{split}
\end{equation*}
Computing the update in the relevant directions of \(s'\), we obtain:
\begin{equation*}
      \dot{s}'_{(n-1)} = \left(\Pi(s_{(n-1)}')^2 + \langle s_1^\star, s'\rangle^2 s_{(n-1)}' (s_{(n-1)}')^\top\right) \left(V'^\top \mG_{(n-1)} - \|V'\|_F^2 s_{(n-1)}'\right)\,.
\end{equation*}
These are the dynamics in \Cref{app:second} with \(\mG_{(1)}\) replaced by \(\mG_{(n-1)}\), so the analysis in \Cref{app:second,app:attention} leads to the following analog of \Cref{thm:cooperative}:
\begin{theorem}\label{thm:higher}
      Assume that the initialization satisfies the following for all \(k \in [n, \nheads]\) and some \(j \in [n+1, T]\):
      \begin{equation*}
            \langle V_{n}(0), A_n^\star\rangle \geq \langle V_{n}(0), A_k^\star \rangle\,, \quad \langle s_{n}(0), s_{n}^\star\rangle \geq \langle s_{n}(0), s_k^\star \rangle\,, \quad \langle s_{n}(0), s_j^\star \rangle > 0\,.
      \end{equation*}
      Further, suppose that \(V_{n}(0), s_{n}(0)\) are such that
      \begin{equation*}
            \langle V_{n}(0), \mG_{(n-1)} (s_{n}(0))_{(n-1)}\rangle > \dfrac{1}{2} \|V_{n}(0)\|_F^2 \|(s_{n}(0))_{(n-1)}\|^2\,,
      \end{equation*}
      and that \(\inf_{t \geq 0} \|s_n(t) - s_1^\star\| > 0\).
      Then, the \(n\)-th head learns the \(n\)-th feature:
      \begin{equation*}
            \lim_{t \to \infty} V_n(t) = m_n^\star V_n^\star\,, \quad \lim_{t \to \infty} s_n(t) = s_n^\star\,.
      \end{equation*}
\end{theorem}
\begin{proof}
      The proof proceeds mutatis mutandis to that of \Cref{thm:cooperative}.
\end{proof}

For initializations where head \(n\) has attention at most \(\epsilon\) on \(s_2^\star, \ldots, s_{n-1}^\star\), the stability argument of \Cref{thm:boundedcoop} applies.
By \Cref{lem:shadowing}, the solution of \Cref{eq:dynamics} stays close to the one from the initialization above on a time window of length \(\gO(\log(1/\epsilon))\), although this does not extend \Cref{thm:higher} to such initializations.
We expect that the proof of \Cref{thm:higher} can be extended to them if \(V_2, \ldots, V_{n-1}\) are also treated as fast variables.
Then, for head \(n\), the positions of \(s_2^\star, \ldots, s_{n-1}^\star\) behave like the position of \(s_1^\star\), and the main additional step is to bound the total variation of how the attention of head \(n\) is distributed among these positions.

\section{Initialization Conditions of the Main Theorems}
\label{app:initialization}

In this section, we discuss the initialization conditions of \Cref{thm:coupling,thm:cooperative} in detail.

\paragraph{Competitive phase.}
We first recall \Cref{rmk:initialization}.

\initialization*

Consider the symmetric initialization \(V_k(0) = 0\) and \(s_k(0) = u \coloneqq \frac{1}{T} 1_T\) for all \(k \in [\nheads]\), i.e., \(q_k(0) = 0\).
Then all heads follow the coupled dynamics of \Cref{sec:coupled}, and we compute the Taylor expansion of \((V, s)\) at \(t = 0\).
Let \(C = I_T - \frac{1}{T} 1_T 1_T^\top\) and \(w = \sum_{j=1}^{\nheads} (m_j^\star)^2 s_j^\star\), so that \(\|u\|^2 = \frac{1}{T}\) and \(\Pi(u) = \frac{1}{T} C\) with \(C^2 = C\).

At first order, since \(V(0) = 0\),
\begin{equation*}
      \dot{V}(0) = \mG u\,, \quad \dot{s}(0) = \Pi(u)^2 \left(V(0)^\top \mG - \nheads \|V(0)\|_F^2 u\right) = 0\,.
\end{equation*}
Therefore, the values move before the attention scores.
At second order, differentiating the coupled dynamics and using \(V(0) = 0\) and \(\dot{s}(0) = 0\),
\begin{equation*}
      \begin{split}
            \ddot{V}(0) &= \mG \dot{s}(0) - 2 \nheads \langle u, \dot{s}(0)\rangle V(0) - \nheads \|u\|^2 \dot{V}(0) = -\dfrac{\nheads}{T} \mG u\,, \\
            \ddot{s}(0) &= \Pi(u)^2 \dot{V}(0)^\top \mG = \dfrac{1}{T^2} C \left(\mG u\right)^\top \mG = \dfrac{1}{T^3} C w\,,
      \end{split}
\end{equation*}
where the term involving the derivative of \(\Pi(s)^2\) vanishes since \(V(0)^\top \mG - \nheads \|V(0)\|_F^2 u = 0\), and \((\mG s)^\top \mG = \sum_{j} (m_j^\star)^2 \langle s_j^\star, s\rangle s_j^\star\) by the orthonormality of the \(V_j^\star\).
Hence,
\begin{equation*}
      V(t) = t \left(1 - \dfrac{\nheads t}{2T}\right) \mG u + \gO(t^3)\,, \quad s(t) = u + \dfrac{t^2}{2 T^3} C w + \gO(t^3)\,.
\end{equation*}
Since \(\langle \mG u, A_k^\star\rangle = \frac{1}{T} (m_k^\star)^2\) and \(\langle C w, s_k^\star\rangle = (m_k^\star)^2 - \frac{1}{T} \sum_{j} (m_j^\star)^2\), and \(m_1^\star > m_k^\star\) for all \(k \neq 1\), both orderings in \Cref{eq:ordering} hold strictly for all sufficiently small \(t > 0\).
That is, the orderings in \Cref{eq:ordering}, which hold with equality at \(t = 0\), become strict for all sufficiently small \(t > 0\).
Since they hold strictly, they also hold for each head of an \(\epsilon\)-perturbation of this initialization with \(\epsilon\) small enough, as each head stays close to the unperturbed trajectory for times \(t \lesssim -\log \epsilon\) by \Cref{thm:bounded}.
By \Cref{thm:coupling}, each head then comes arbitrarily close to the fixed point in \Cref{eq:fixed_point} within this time as \(\epsilon\) becomes small.

\paragraph{Cooperative phase.}
The following remark gives a family of initializations satisfying the initialization assumptions of \Cref{thm:cooperative}.
\begin{remark}\label{rem:init_loss}
      The initialization conditions of \Cref{thm:cooperative}, that is, all of its assumptions except that \(s'(t)\) stays away from \(s_1^\star\), are satisfied by the following initialization
      \begin{equation*}
            V'(0) = \dfrac{m_1^\star}{\nheads} V_1^\star + \eta V_2^\star\,, \quad s'(0) = (1 - \epsilon - \epsilon') s_1^\star + \epsilon s_2^\star + \epsilon' w\,,
      \end{equation*}
      where \(w\) is any probability vector supported on the positions of \(s_3^\star, \ldots, s_T^\star\), for any \(\eta > 0\) and
      \begin{equation*}
            0 < \epsilon < \min\left\{\frac{1}{2}, \frac{2 m_2^\star \eta}{(m_1^\star/\nheads)^2 + \eta^2}\right\}\,, \quad 0 < \epsilon' \leq \epsilon\,, \quad (\epsilon')^2 < \epsilon \left(\frac{2 m_2^\star \eta}{(m_1^\star/\nheads)^2 + \eta^2} - \epsilon\right)\,.
      \end{equation*}
\end{remark}

We now argue, by a similar Taylor argument, that the dynamics in \Cref{eq:second_dynamics} reach the regime of \Cref{rem:init_loss}.
Consider an initialization near \Cref{eq:fast_init} with perturbations of both the value and the attention:
\begin{equation*}
      V'(0) = \dfrac{m_1^\star}{\nheads} V_1^\star + \eta_0 V_2^\star\,, \quad s'(0) = (1 - \epsilon - \epsilon') s_1^\star + \epsilon s_2^\star + \epsilon' w\,,
\end{equation*}
with \(w\) as in \Cref{rem:init_loss}, \(\eta_0 = \gO(\epsilon)\), and \(\epsilon' = \gO(\epsilon^2)\), which need not satisfy \Cref{eq:init_loss}.
Since \(s_{(1)}'(0) = \epsilon s_2^\star + \epsilon' w\), the value moves at first order in \(\epsilon\):
\begin{equation*}
      \dot{V}'(0) = \mG_{(1)} s_{(1)}'(0) - \|s_{(1)}'(0)\|^2 V'(0) = \epsilon m_2^\star V_2^\star + \gO(\epsilon^2)\,.
\end{equation*}
In contrast, \(\Pi(s')\) vanishes at \(s_1^\star\), so \(\Pi(s'(0)) = \gO(\epsilon)\), and the term \(V'^\top \mG_{(1)} - \|V'\|_F^2 s_{(1)}'\) is also \(\gO(\epsilon)\).
Hence, \(\dot{s}'(0) = \gO(\epsilon^3)\).
For times \(t = \gO(1)\), we thus have \(s'(t) = s'(0) + \gO(\epsilon^3)\) and \(V'(t) \approx \frac{m_1^\star}{\nheads} V_1^\star + (\eta_0 + \epsilon m_2^\star t) V_2^\star\), i.e., the initialization in \Cref{rem:init_loss} with \(\eta \approx \eta_0 + \epsilon m_2^\star t\).
Since \(\epsilon' = \gO(\epsilon^2)\), the conditions on \(\epsilon'\) in \Cref{rem:init_loss} hold for small \(\epsilon\), and only the condition on \(\epsilon\) matters to leading order.
This condition holds as soon as
\begin{equation*}
      t > \dfrac{(m_1^\star/\nheads)^2}{2 (m_2^\star)^2} - \dfrac{\eta_0}{\epsilon m_2^\star}\,,
\end{equation*}
which, for \(\eta_0 \geq 0\), is bounded by \((m_1^\star/\nheads)^2 / (2 (m_2^\star)^2)\) independently of \(\epsilon\).
In other words, the scale separation between \(\dot{V}'\) and \(\dot{s}'\) near \(s_1^\star\) drives the dynamics into the regime of \Cref{rem:init_loss}.
A similar analysis applies to the initialization of \Cref{thm:higher}.

Note that the initialization regimes in our theorems are not limits toward a particular point, such as a vanishing initialization, but large sets that satisfy some ordering.
Together with the analysis above, this suggests that the conditions of both theorems are reached by the dynamics from a wider range of initializations than those stated in the theorems.

\section{Simulations and Comparison with Transformers}
\label{app:simulations}

In this section, we numerically simulate the regression gradient flow analyzed in \Cref{sec:regression}, verify that our theoretical predictions match these simulations, and compare the resulting trajectories with those of the minimal transformer trained on the classification task from \Cref{sec:incremental}.

\subsection{Numerical Simulations of the Regression Flow}
\label{app:simulations:regression}

We present numerical simulations of the gradient flow dynamics of the loss in \Cref{eq:loss} with the following parameters: \(d = 50\), \(T = 40\), \(h = 3\), \(|I(k)| = 1\) for all \(k \in [h]\), and \(m_k^\star = m^{h-k}\) with \(m = 1.7\).
We initialize the value parameters \(V_i\) to 0, and the attention patterns \(s_i\) by normalizing \(\frac{1}{T} 1_T + u_i\) to the simplex, where the entries of \(u_i\) are sampled uniformly from \([0, \epsilon)\) with \(\epsilon = 10^{-3}\).
\Cref{fig:simulations} shows the evolution of the attention patterns \(s_k\), the value parameters \(V_k\), and the squared error \(\|\mG - \mP\|_F^2\) over time.

The results align with the transformer experiments in \Cref{sec:incremental}.
Similar to the transformer experiments, the heads first learn from position \((1)\), then position \((2)\), and finally position \((3)\).
The time scales of these stages are clearly separated where the first stage is the fastest and the third stage is the slowest.
Notably, at first, all heads try to learn from the position \((1)\), as it is related to the most important feature.
After this competitive phase, the heads start to learn from the position \((2)\) and then the position \((3)\) where they specialize in different patterns.
Here, they cooperate to learn from the position \((3)\).
In particular, the first head becomes negatively aligned with feature \((3)\), to cancel the cross term from the third head's residual attention on position \((1)\).
This alignment fades as the residual attention decays.

\begin{figure}[!ht]
      \begin{center}
            \includegraphics[width=\linewidth]{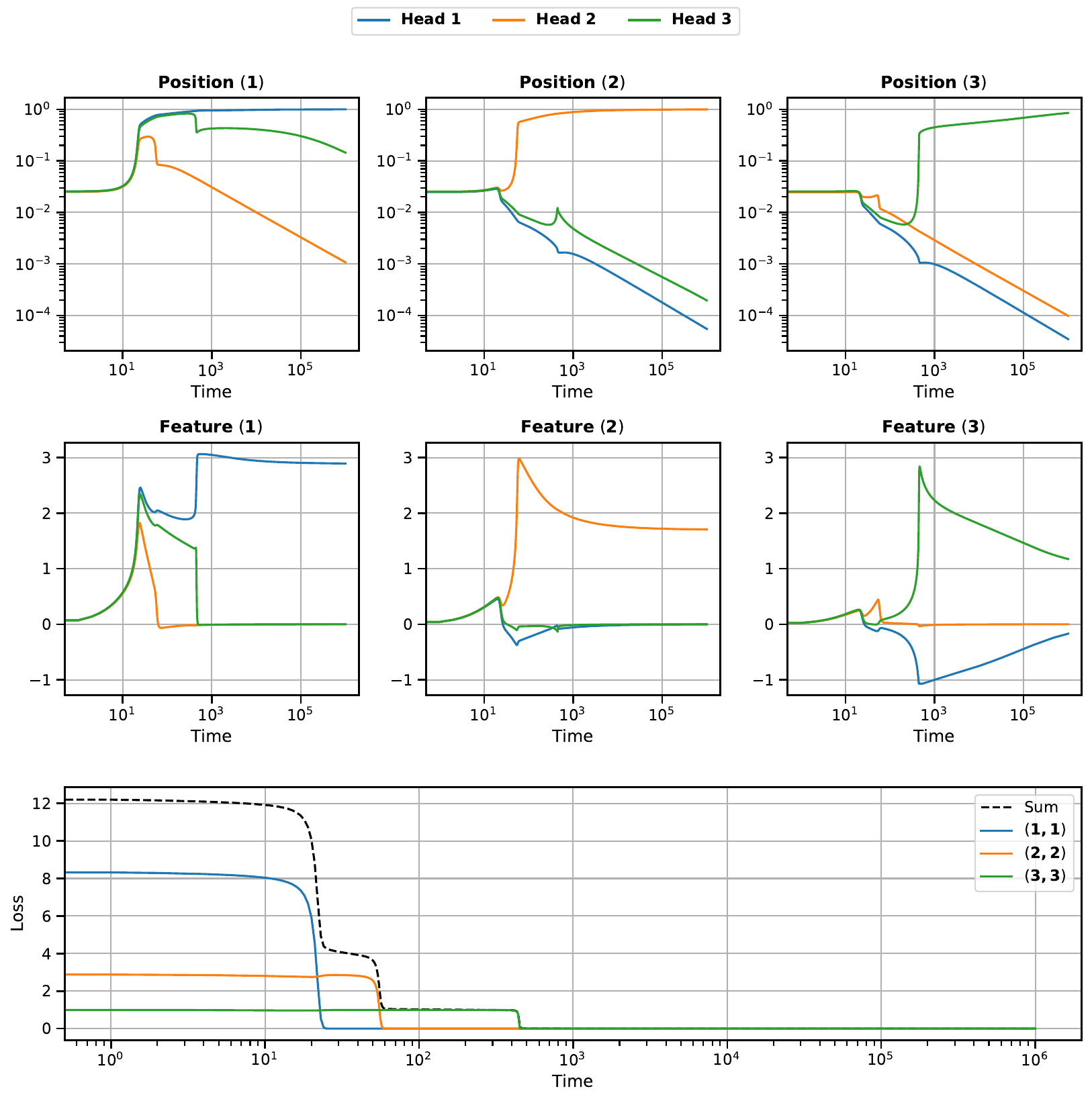}
      \end{center}
      \caption{
            (Top) Per-head attention mass \(s_k\) on positions \((1), (2), (3)\) over time.
            (Middle) Per-head value-matrix alignments \(\langle V_k, V_j^\star\rangle\) with the three feature directions.
            (Bottom) The evolution of the squared error \(\|\mG - \mP\|_F^2\) over time.
            We only plot the relevant coordinates of \(s_k\) and \(V_k\) for clarity.
            We decompose the squared error into the (feature, position) contributions which are plotted in the color of the heads that learn these contributions.
      }
      \label{fig:simulations}
\end{figure}

\Cref{fig:theory_vs_regression} overlays the predictions of our simplified analysis on the full regression simulation.
We pick \(t = 25\) as the start of the second (cooperative) stage and \(t = 300\) as the start of the third stage; at each boundary, we initialize the simplified dynamics of \Cref{sec:cooperative,app:higher} from the theoretical initialization for that stage and integrate them forward, plotting the resulting trajectories as dashed curves.
The simplified dynamics quantitatively trace the full ODE across both attention positions \(s_k\) and value parameters \(V_k\), including the transient negative alignments of the non-offshooting heads during the cooperative phase.
This indicates that our theoretical simplifications preserve the quantitative behavior of the underlying dynamics, not just the qualitative stage structure.

\begin{figure}[!ht]
      \begin{center}
            \includegraphics[width=\linewidth]{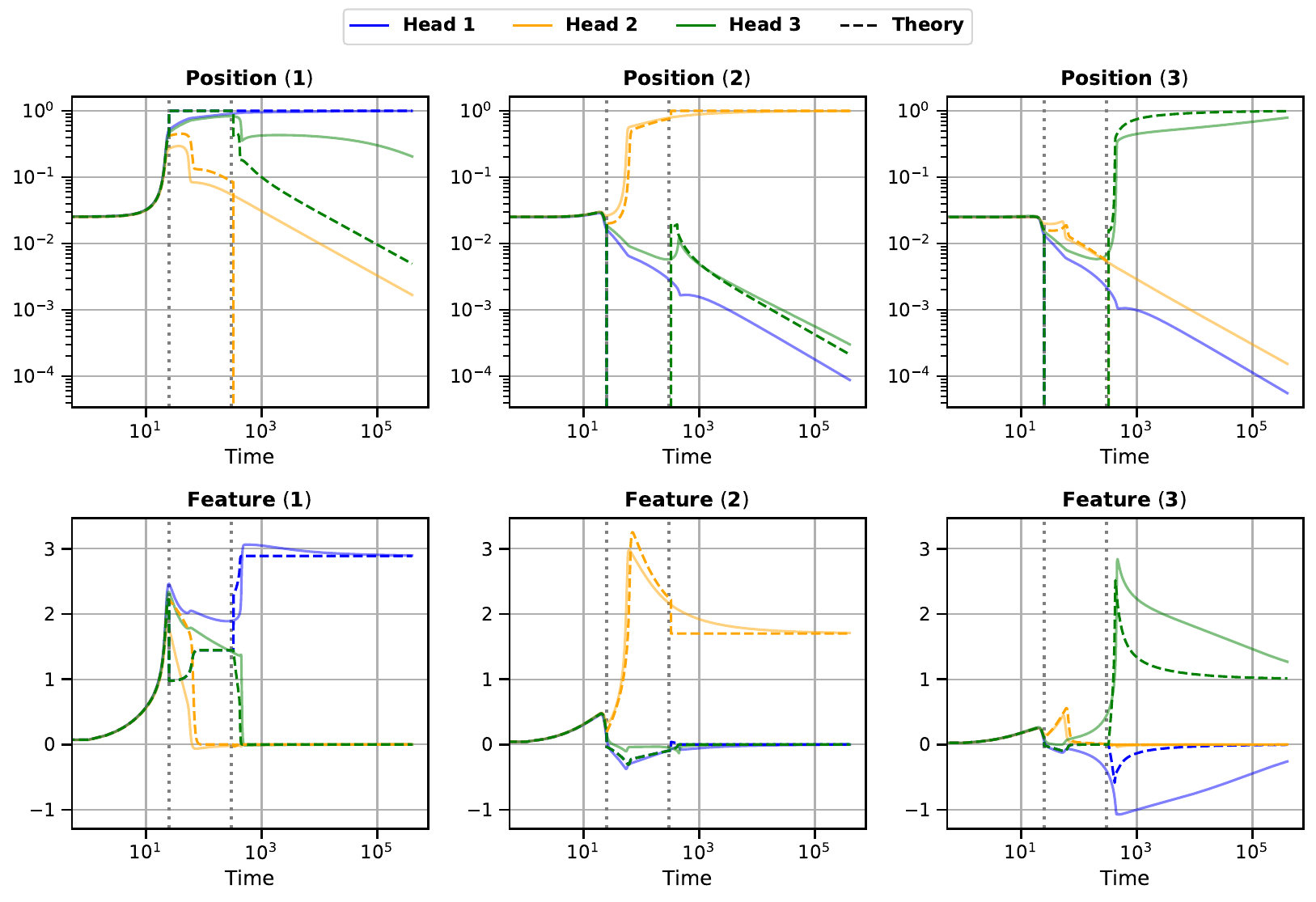}
      \end{center}
      \caption{
            Quantitative comparison between the regression gradient-flow simulation (solid) and the simplified theoretical dynamics (dashed).
            The dashed curves are obtained by integrating the simplified dynamics of \Cref{sec:cooperative,app:higher} forward from the theoretical initialization at each stage boundary, \(t = 25\) for the second stage and \(t = 300\) for the third.
            (Top) Per-head attention mass \(s_k\) on positions \((1), (2), (3)\) over time.
            (Bottom) Per-head value-matrix alignments \(\langle V_k, V_j^\star\rangle\) with the three feature directions.
            Vertical dotted lines mark the stage boundaries.
            The theoretical trajectories quantitatively track the full ODE, capturing both convergence to each saddle and the compensatory negative alignments of non-offshooting heads during the cooperative phase.
      }
      \label{fig:theory_vs_regression}
\end{figure}

\subsection{Comparison with the Classification Transformer}
\label{app:simulations:transformer}

We now compare the regression-flow simulation in \Cref{fig:simulations} with the corresponding trajectories of the minimal transformer trained on the classification task in \Cref{fig:collaboration}.
Despite the differences in architecture, loss, and parameterization, the two settings agree on the structural signatures of the dynamics rather than only on the existence of stages.
First, both exhibit a clear competitive phase in which all three heads concentrate on the same most important positions before any specialization occurs, with attention to the remaining positions decaying simultaneously across heads.
Second, the offshooting order is the same in both: heads leave the shared pattern one at a time, in the order dictated by the feature importances \(m_1^\star > m_2^\star > m_3^\star\), with the second-most important position acquired before the least important one.
Third, the value matrices reproduce the cooperative compensation predicted by the theory: when the offshooting head locks onto a new feature, the non-offshooting heads become negatively aligned with that same feature direction to cancel the cross-term induced by the offshooting head's residual attention.
This is visible in \Cref{fig:simulations} as the negative alignments of \(V_k\) for the heads that do not offshoot in the current stage, and in \Cref{fig:collaboration} as the corresponding negative alignments of the value matrices of Heads~1 and~3 during Stage~2 and of Head~1 during Stage~3.
The comparison is qualitative: the regression flow uses a continuous-time gradient flow on a population objective with singleton groups, while the transformer is trained on a finite sample with cross-entropy and absolute positional encodings, so axis scales and timescales differ by construction.

\begin{figure}[!ht]
      \begin{center}
            \includegraphics[width=\linewidth]{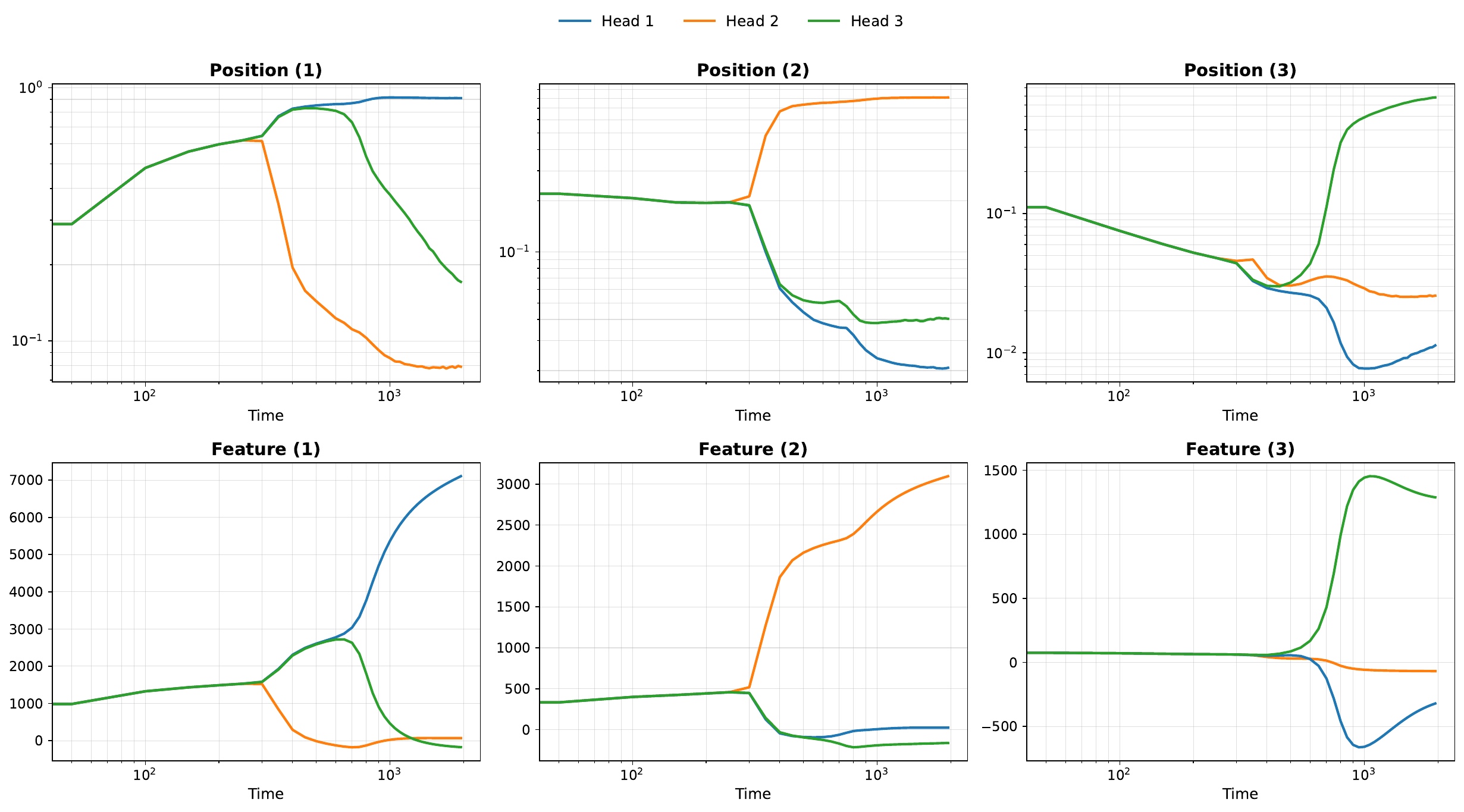}
      \end{center}
      \caption{
            Same quantities as in \Cref{fig:simulations}, but measured on the minimal model with a small initialization on the classification task rather than on the regression gradient flow.
            (Top) Per-head attention mass on positions in \(I(1), I(2), I(3)\) over training.
            (Bottom) Per-head alignment of the value matrix with the feature directions \(A_1^\star, A_2^\star, A_3^\star\).
            The transformer reproduces the structural signatures predicted by the theory:
            (i) a competitive phase in which all heads concentrate on \(I(1)\), the most important positions;
            (ii) a sequential offshoot, with Head~2 acquiring \(I(2)\) before Head~3 acquires \(I(3)\);
            (iii) compensatory negative alignments in the non-offshooting heads' value matrices (e.g., Heads~1 and~3 become negatively aligned with \(A_2^\star\) as Head~2 grows, and Head~1 becomes negatively aligned with \(A_3^\star\) as Head~3 grows), matching the cooperative mechanism analyzed in \Cref{sec:cooperative}.
            The comparison with \Cref{fig:simulations} is qualitative: the two settings differ in architecture, loss, and parameterization, so axis scales and timescales are not aligned.
      }
      \label{fig:collaboration}
\end{figure}

\section{Experimental Details}
\label{app:exp}

The full model has a standard single-block transformer decoder architecture as discussed in \Cref{sec:incremental}.
It uses absolute positional encodings with learnable embedding and unembedding matrices and has the configuration shown in \Cref{tab:transformer-param}. The minimal model, as described in 
\Cref{sec:representation}, removes layer normalization, dropout, residual connections, key and output attention matrices, and the MLP layer.
It uses one-hot positional encodings and does not have embedding and unembedding matrices. 
Both the full model and the minimal model are trained with the same optimization hyperparameters listed in \Cref{tab:optimization-param}, and the same synthetic data generation process described in \Cref{tab:synth-dataset-param}.
The main difference in the learning task between the two models is the interval lengths \(|I(k)|\) of the Markov process: the full model uses intervals of length 4, while the minimal model uses intervals of length 2, as summarized in \Cref{tab:markov-intervals}.

We train the $n$-gram models using the same architecture and optimization hyperparameters as the full transformer model, but train with windows of size $n$ sliding over the full sequence.
The source code to reproduce our experiments is available at \url{https://github.com/tml-epfl/sparse-attention-dynamics}.

\begin{table}[!h]
\caption{Synthetic dataset parameters}
\label{tab:synth-dataset-param}
\begin{center}
\begin{tabular}{lc}
\multicolumn{1}{c}{\bf Parameter} & \multicolumn{1}{c}{\bf Value} \\
\hline \\
Heads $h$               & 3 \\
Dictionary size $d$     & 50 \\
Multiplicative constant $m$ & 1.7 \\
Base scale $b_0$        & 10 \\
Sequence length $T$     & 20 \\
Train samples           & 9000 \\
Test samples            & 3000 \\
Seed                    & 0 \\
\end{tabular}
\end{center}
\end{table}

\begin{table}[!h]
\caption{Optimization hyperparameters}
\label{tab:optimization-param}
\begin{center}
\begin{tabular}{lc}
\multicolumn{1}{c}{\bf Parameter} & \multicolumn{1}{c}{\bf Value} \\
\hline \\
Steps               & 2000 \\
Batch size          & 3000 \\
Gradient clipping   & 1.0 \\
Optimizer           & AdamW \\
Weight decay        & 0.01 \\
Learning rate       & 0.003 \\
Scheduler           & ReduceLROnPlateau \\
Patience            & 10 \\
Factor              & 0.5 \\
\end{tabular}
\end{center}
\end{table}

\begin{table}[!h]
\caption{Transformer configuration}
\label{tab:transformer-param}
\begin{center}
\begin{tabular}{lc}
\multicolumn{1}{c}{\bf Parameter} & \multicolumn{1}{c}{\bf Value} \\
\hline \\
Hidden dimension        & 255 \\
Feedforward dimension   & 64 \\
Dropout                 & 0.1 \\
Initialization scale    & 1 \\
Number of blocks        & 1 \\
Number of heads         & 3 \\
\end{tabular}
\end{center}
\end{table}

\begin{table}[!h]
\caption{Markov process intervals}
\label{tab:markov-intervals}
\begin{center}
\begin{tabular}{lcc}
\multicolumn{1}{c}{} & \multicolumn{1}{c}{\bf Full} & \multicolumn{1}{c}{\bf Minimal} \\
\hline \\
$w$  & $12$              & $6$ \\
$I(1)$             & $\{1,2,3,4\}$     & $\{1,2\}$ \\
$I(2)$             & $\{5,6,7,8\}$     & $\{3,4\}$ \\
$I(3)$             & $\{9,10,11,12\}$  & $\{5,6\}$ \\
\end{tabular}
\end{center}
\end{table}

\section{Additional Experiments}
\label{app:additional_experiments}
We run additional experiments to study incremental learning behavior under different settings:
\begin{itemize}
      \item infinite data instead of a finite dataset (\Cref{app:infinite_data});
      \item a reversed order of importance with non-uniform interval lengths (\Cref{app:reverse_order});
      \item training without weight decay (\Cref{app:weight_decay});
      \item two-block transformers (\Cref{app:two_block});
      \item non-uniform \(\alpha\) values within each interval (\Cref{app:non_uniform_alpha});
      \item overlapping intervals (\Cref{app:overlapping});
      \item SGD instead of AdamW (\Cref{app:sgd}).
\end{itemize}

\subsection{Infinite Data}
\label{app:infinite_data}
Instead of training on a finite dataset of 9000 samples, we train the model with infinite data by sampling a new batch of data at each step. This removes any effect of overfitting on incremental learning.
We observe in \Cref{fig:attention-infinite-data} and \Cref{fig:plots-infinite-data} that the model still exhibits the same behavior.
This experiment is run with the minimal architecture described in \Cref{sec:minimal}.
\begin{figure}[!h]
      \begin{center}
            \includegraphics[width=\linewidth]{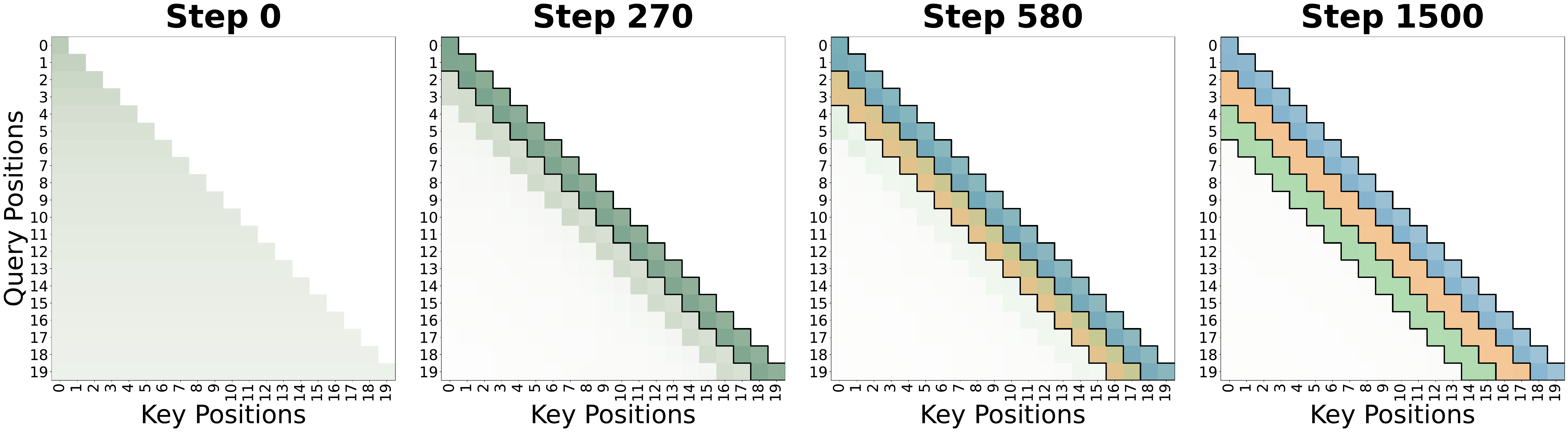}
      \end{center}
      \caption{Attention patterns over the training steps with online sampling of data.}
      \label{fig:attention-infinite-data}
\end{figure}
\begin{figure}[!h]
      \begin{center}
            \includegraphics[page=1,width=0.48\linewidth]{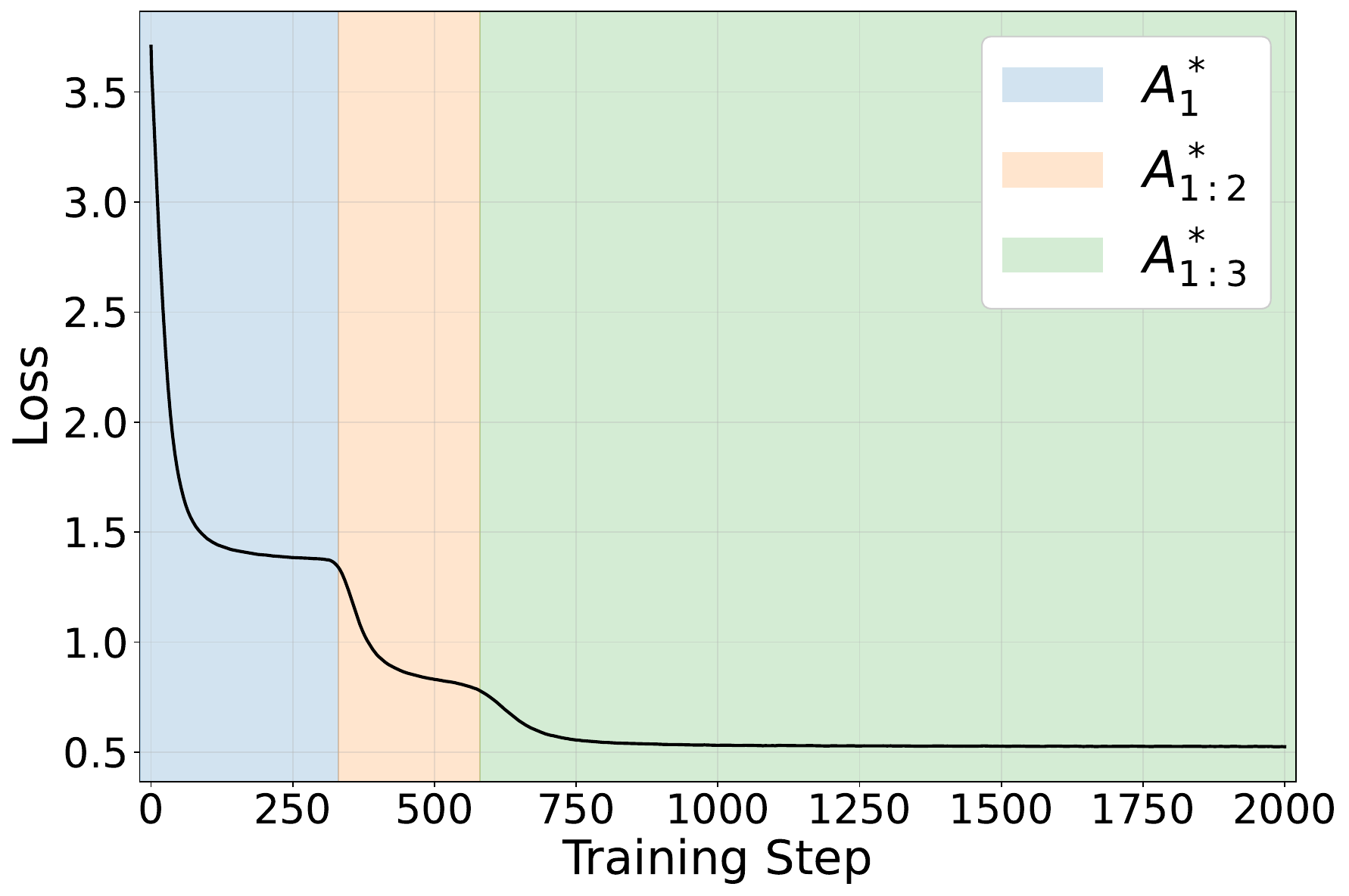}
            \includegraphics[page=1,width=0.48\linewidth]{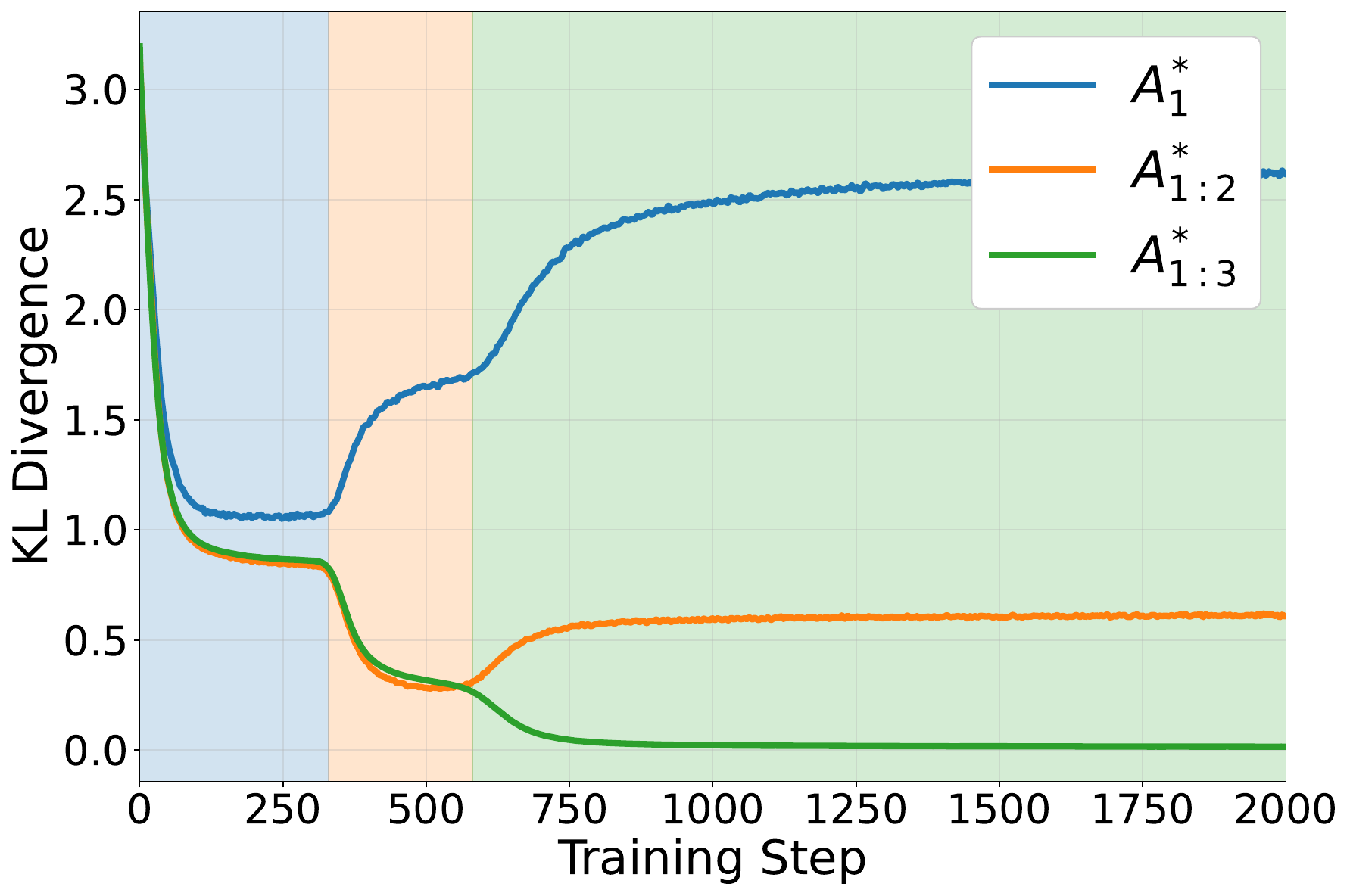}
      \end{center}
      \caption{Validation loss and KL divergence over the training steps with online sampling of data.}
      \label{fig:plots-infinite-data}
\end{figure}

\FloatBarrier
\subsection{Reverse Order}
\label{app:reverse_order}

We reverse the order of importance of the intervals such that the most important interval is the furthest one.
\Cref{fig:attention-reversed} and \Cref{fig:plots-reversed} show the results when \(I(3)=\{13, 14\}\), \(I(2)=\{9, 10, 11, 12\}\), and \(I(1)=\{1, 2, 3, 4, 5, 6, 7, 8\}\), which reveal the same behavior as the original order.
We also note that it is generally easier to observe incremental learning behavior when the most important interval is the furthest one.
This indicates that the learning dynamics are impacted by the sequential structure of the task.
This experiment is run with the full architecture described in \Cref{sec:incremental}.

\begin{figure}[!h]
      \begin{center}
            \includegraphics[width=\linewidth]{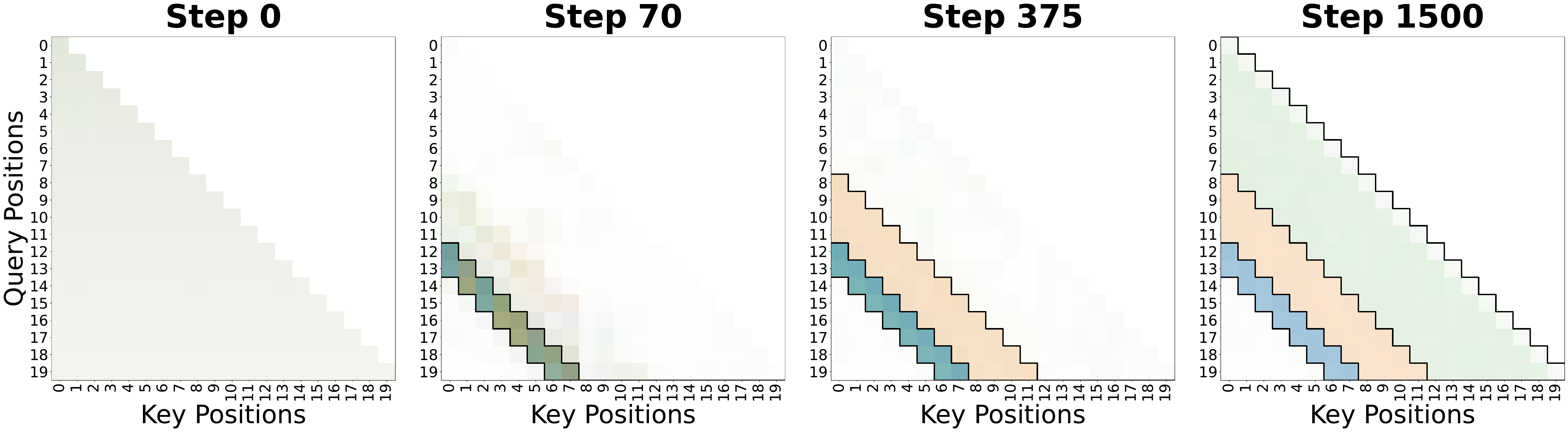}
      \end{center}
      \caption{Attention patterns over the training steps with reversed order of importance and varying interval lengths.}
      \label{fig:attention-reversed}
\end{figure}
\begin{figure}[!h]
      \begin{center}
            \includegraphics[page=1,width=0.48\linewidth]{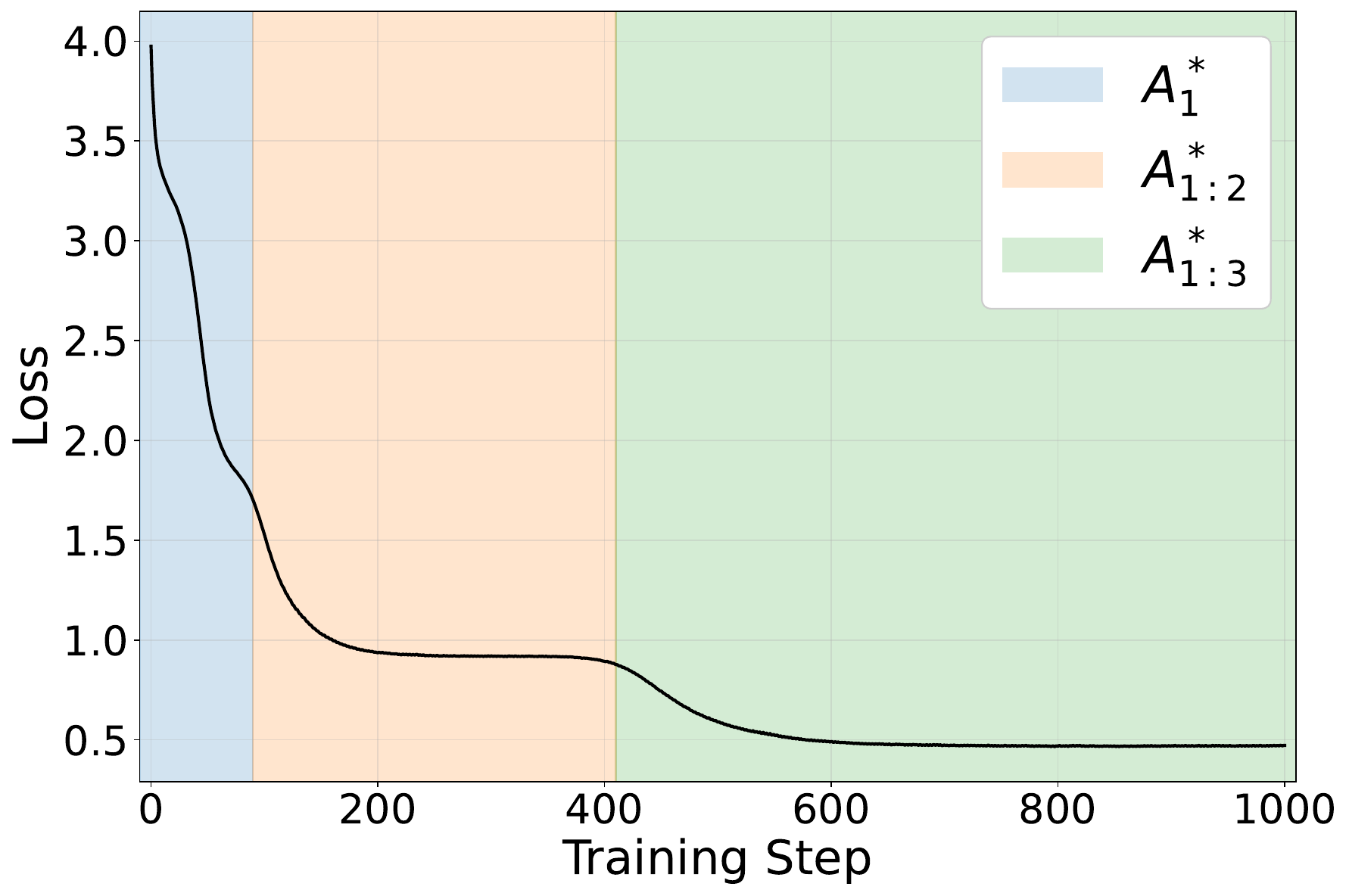}
            \includegraphics[page=1,width=0.48\linewidth]{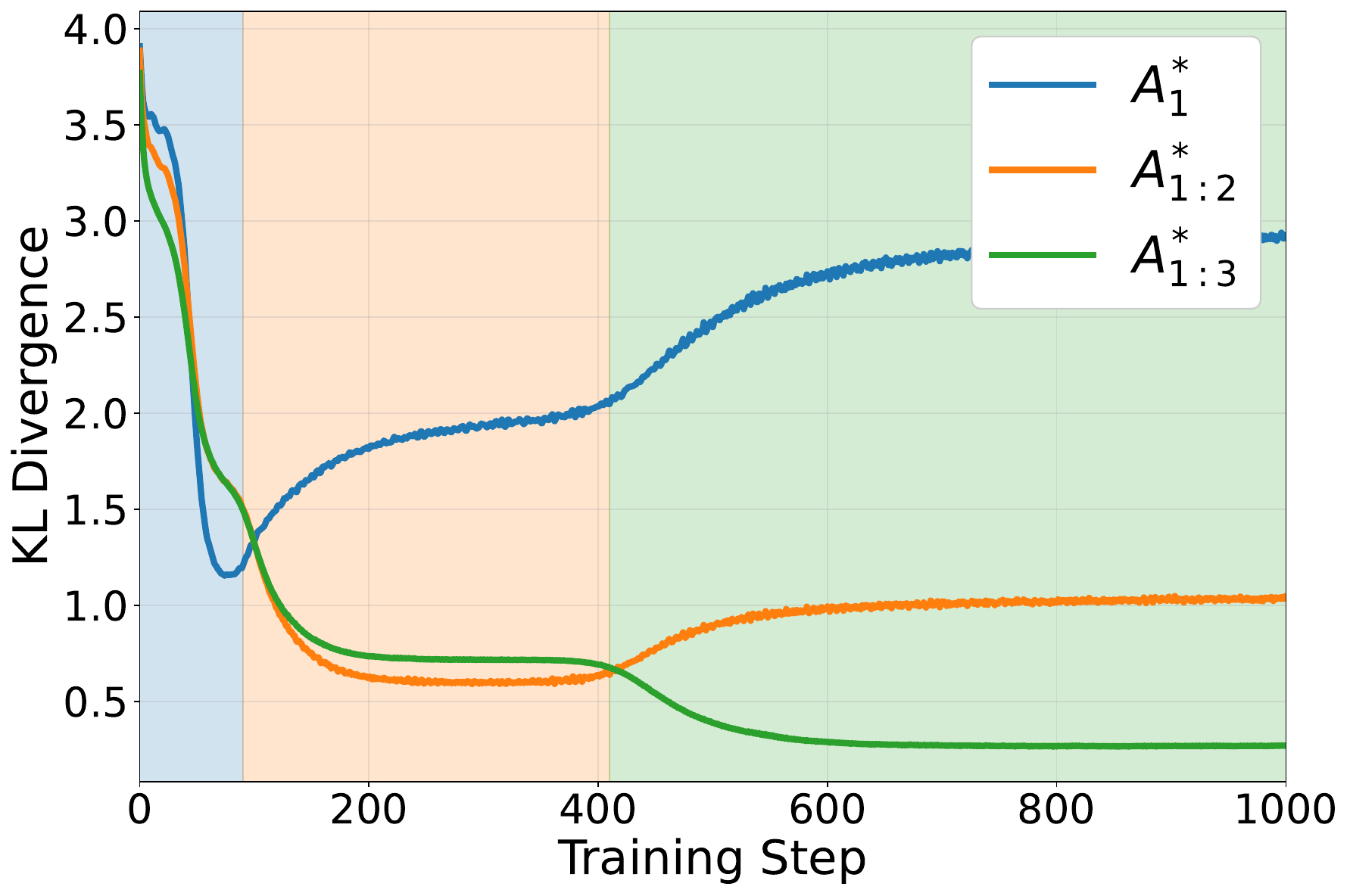}
      \end{center}
      \caption{Validation loss and KL divergence over the training steps with reversed order of importance and varying interval lengths.}
      \label{fig:plots-reversed}
\end{figure}

\FloatBarrier
\subsection{Weight Decay}
\label{app:weight_decay}

We also study the impact of weight decay on the learning dynamics.
We observe almost no difference in the learning dynamics when weight decay is not applied so we do not report the results.

\FloatBarrier
\subsection{Two-block Transformers}
\label{app:two_block}

We train two-block minimal and full transformers with the same configuration as in \Cref{app:exp} but
adjusting the learning rate and number of training examples.
\Cref{fig:2layer-minimal,fig:2layer-full} show that the incremental learning behavior is similar to the single-block case.
We observe that the first region corresponding to the first feature matrix is less pronounced.

\begin{figure}[!h]
      \begin{center}
            \includegraphics[page=1,width=0.48\linewidth]{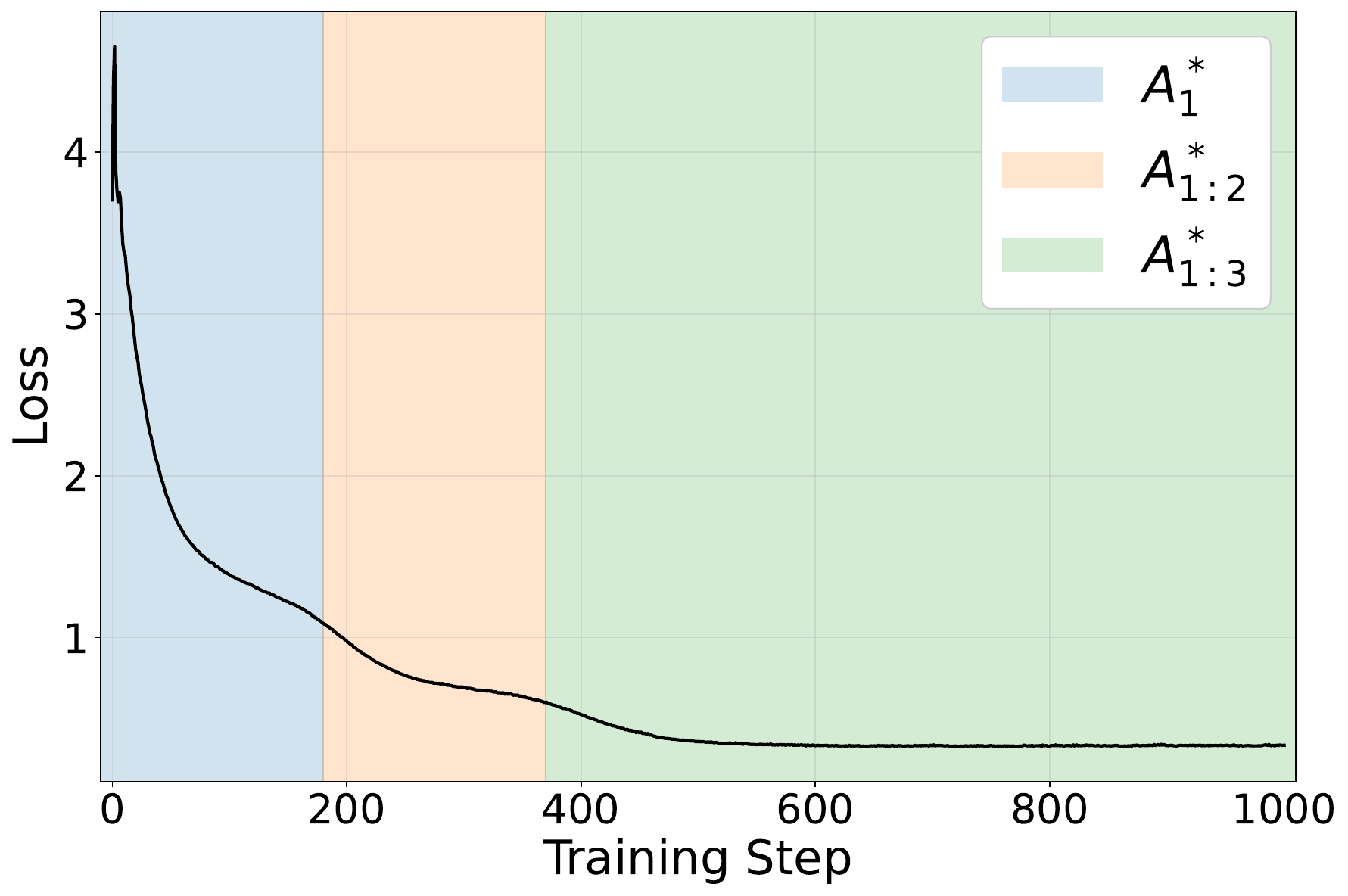}
            \includegraphics[page=1,width=0.48\linewidth]{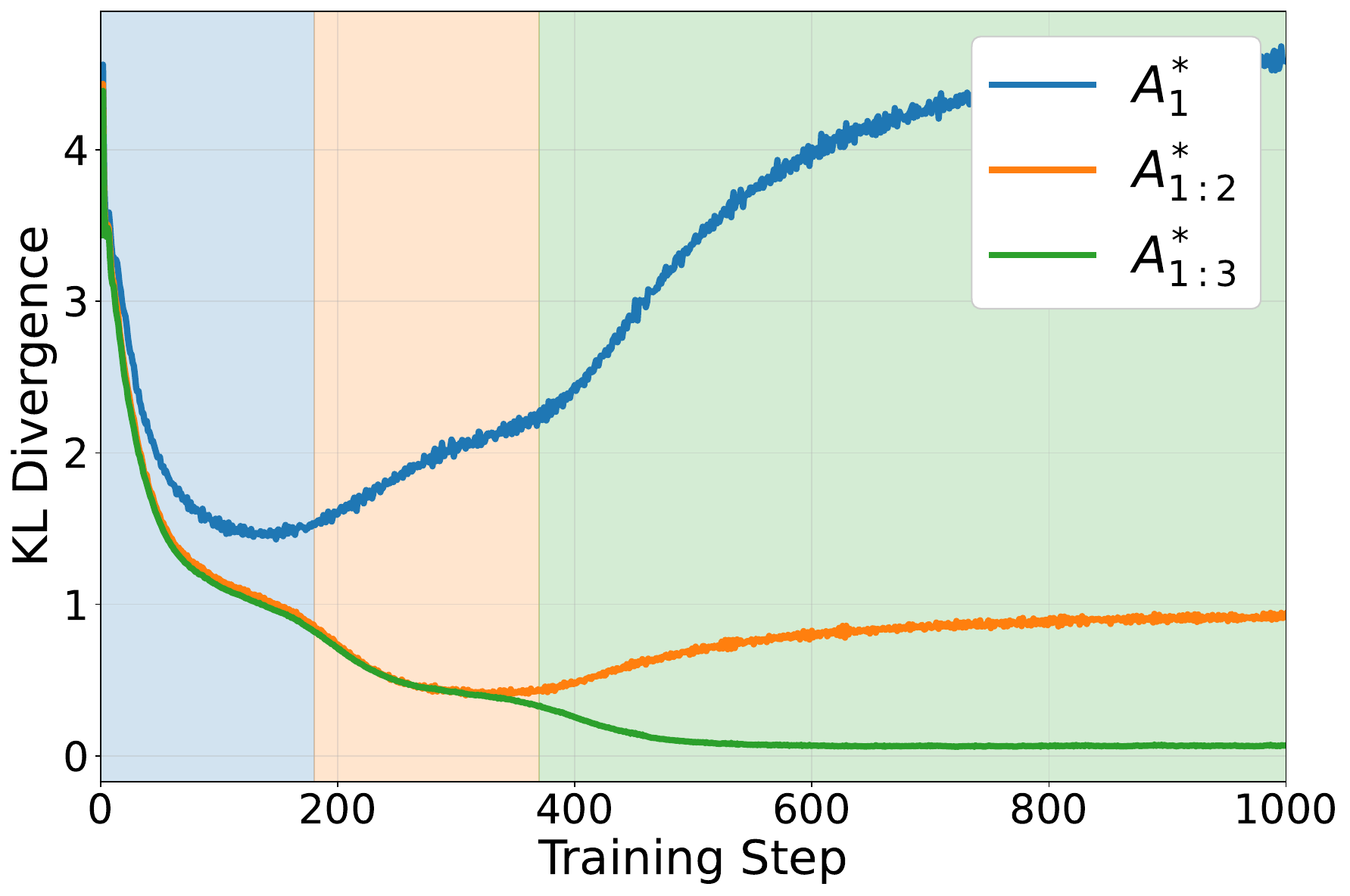}
      \end{center}
      \caption{Validation loss and KL divergence over the training steps for a two-block minimal transformer.}
      \label{fig:2layer-minimal}
\end{figure}

\begin{figure}[!h]
      \begin{center}
            \includegraphics[page=1,width=0.48\linewidth]{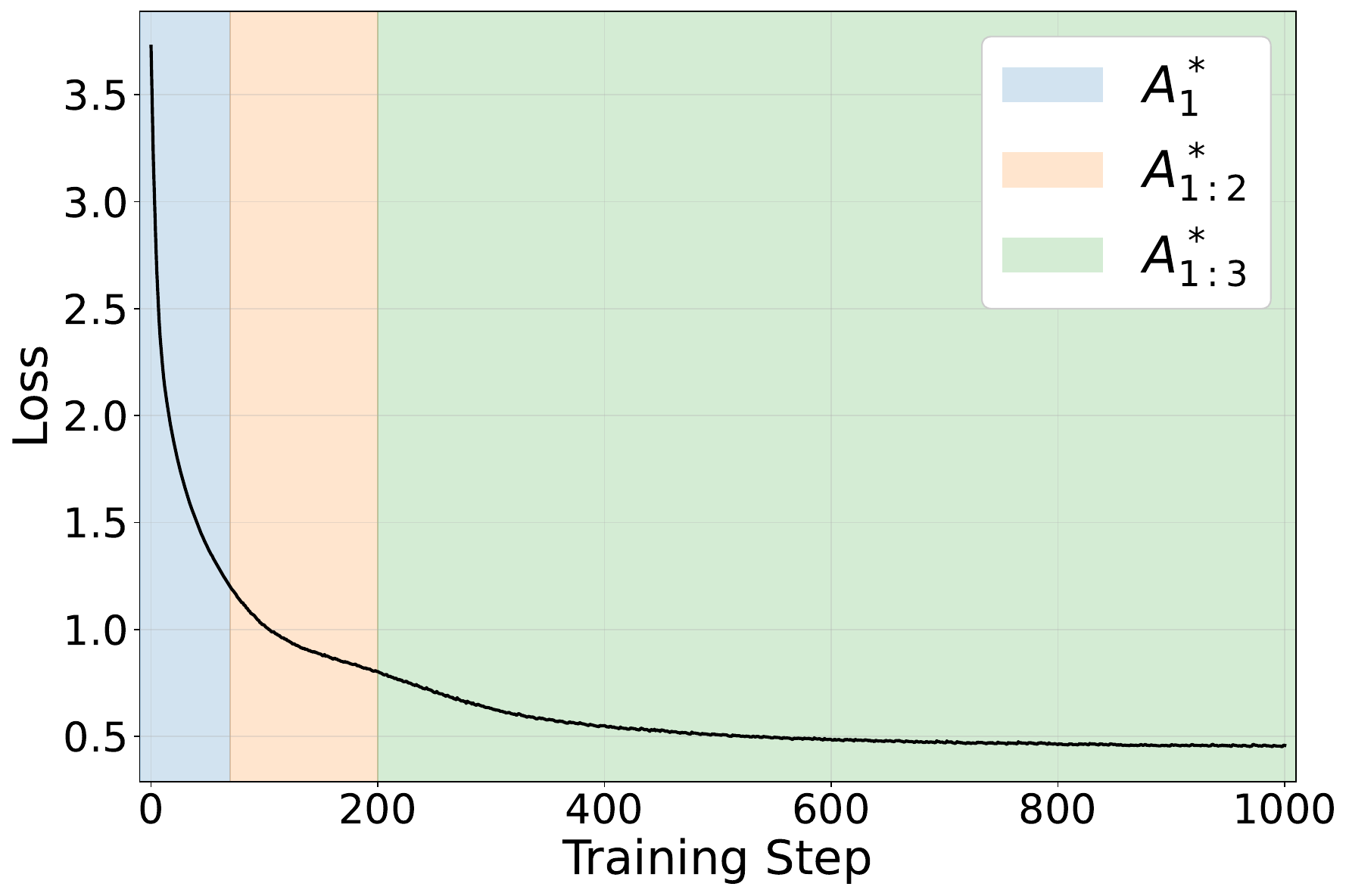}
            \includegraphics[page=1,width=0.48\linewidth]{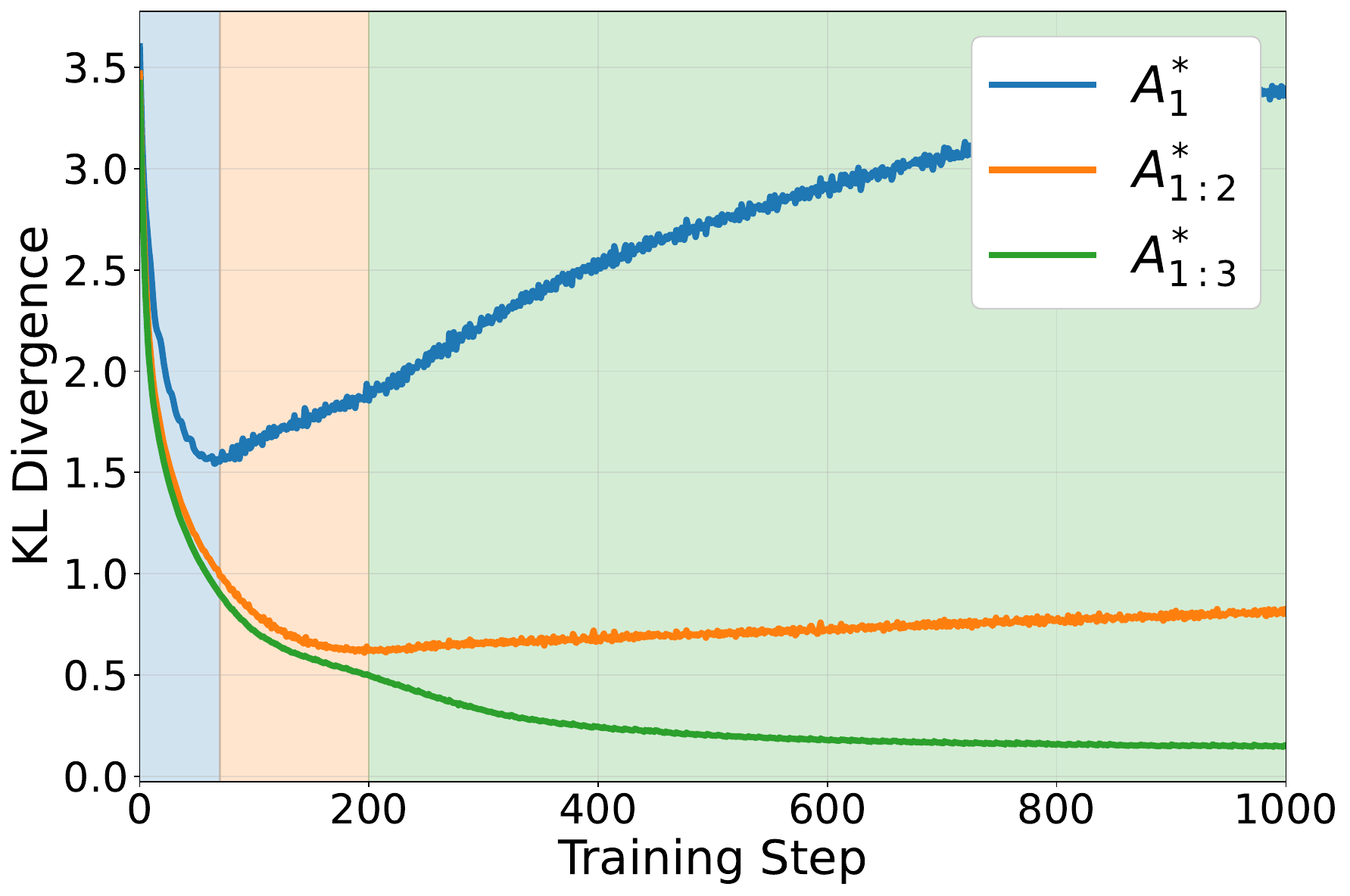}
      \end{center}
      \caption{Validation loss and KL divergence over the training steps for a two-block full transformer.}
      \label{fig:2layer-full}
\end{figure}

\FloatBarrier
\subsection{Non-uniform \texorpdfstring{\(\alpha\)}{α} values}
\label{app:non_uniform_alpha}

We run experiments with \(\alpha = [0.7, 0.3]\) in \Cref{fig:kl-non-uniform-alphas} and observe that the model still exhibits incremental learning. 
In \Cref{fig:attention-non-uniform-alphas}, we observe the checkered pattern where heads focus more attention on the position with the highest \(\alpha\) value.

\begin{figure}[!h]
      \begin{center}
            \includegraphics[page=1,width=0.48\linewidth]{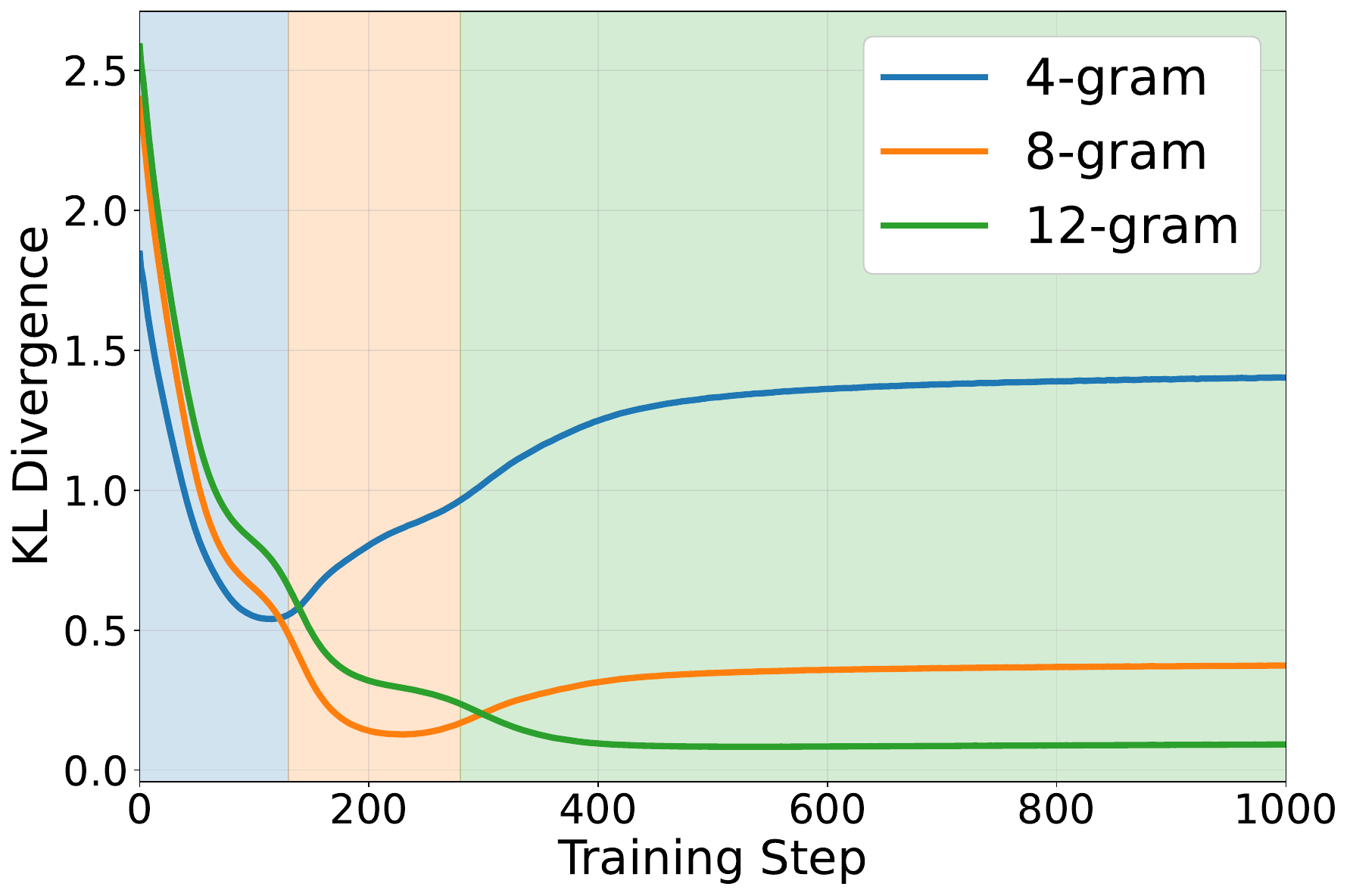}
      \end{center}
      \caption{KL divergence over the training steps with non-uniform \(\alpha\) values.}
      \label{fig:kl-non-uniform-alphas}
\end{figure}
\begin{figure}[!h]
      \begin{center}
            \includegraphics[width=\linewidth]{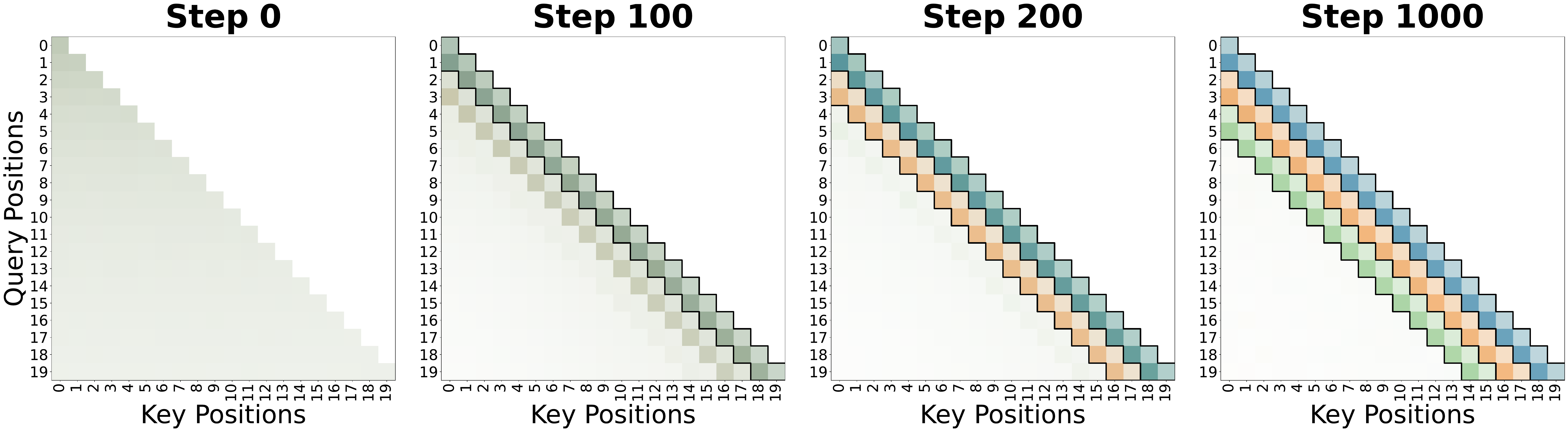}
      \end{center}
      \caption{Attention patterns over the training steps with non-uniform \(\alpha\) values.}
      \label{fig:attention-non-uniform-alphas}
\end{figure}

\FloatBarrier
\subsection{Overlapping Intervals}
\label{app:overlapping}
We run experiments with overlapping intervals where \(I(1) = \{5, 6, 7, 8\}\), 
\(I(2) = \{3, 4, 5, 6\}\), and \(I(3) = \{1, 2, 3, 4\}\).
This corresponds to interval lengths of 4 with an overlap or stride of 2.
We try learning transformers with three or four heads.
We observe in \Cref{fig:kl-overlapping-4heads} that the model with four heads still exhibits incremental learning behavior.
Similar results are observed for the model with three heads and thus omitted.
Attention patterns in \Cref{fig:attention-three-heads-overlapping-intervals,fig:attention-four-heads-overlapping-intervals} reveal the different learning orderings for three and four heads.
When the intervals are overlapping, it is unclear which positions are statistically the most significant, and transformers may follow different solutions based on the feature matrices.

\begin{figure}[!h]
      \begin{center}
            \includegraphics[page=1,width=0.48\linewidth]{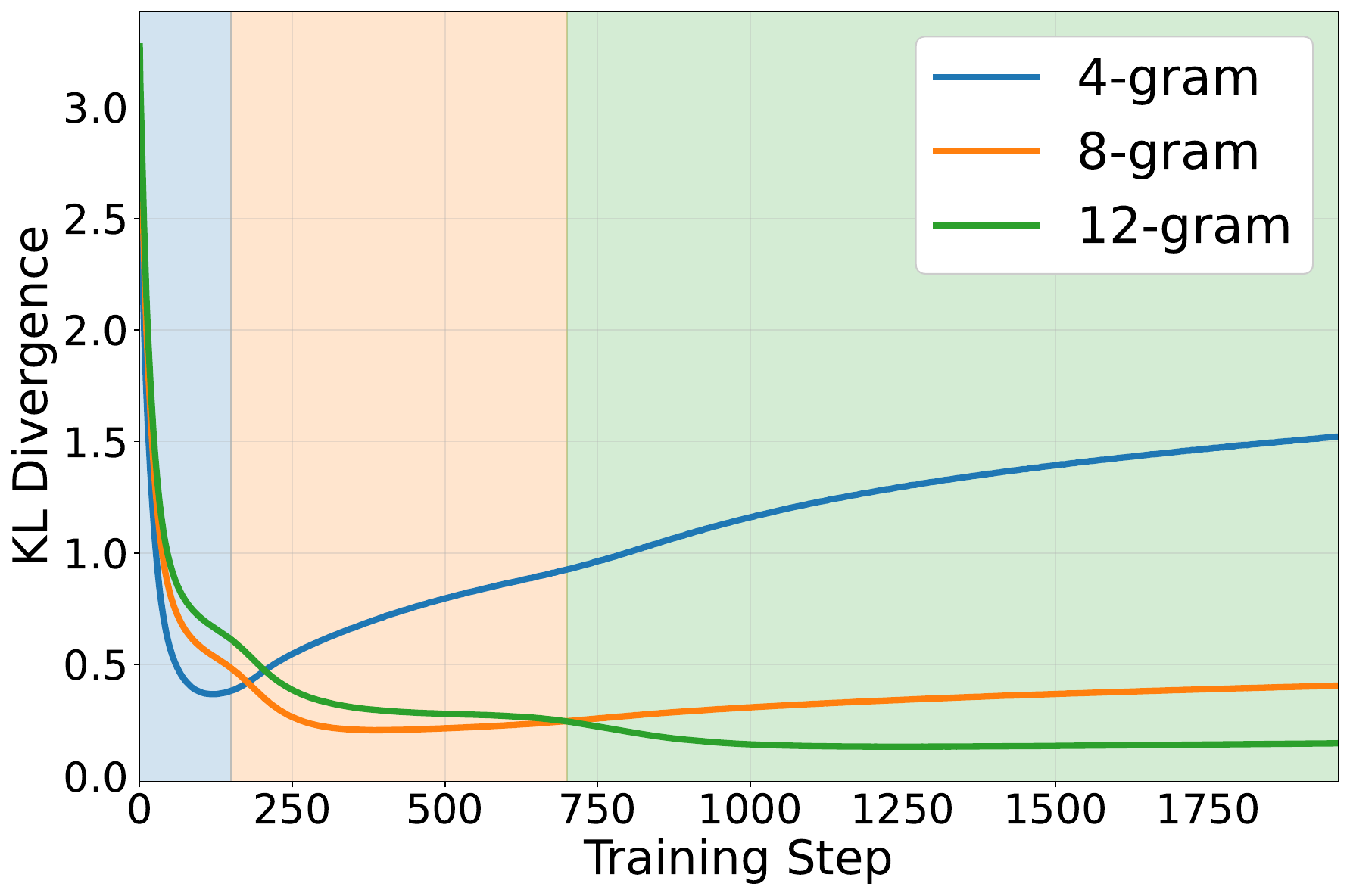}
      \end{center}
      \caption{KL divergence over the training steps with intervals of size 4 and overlap of 2 for a transformer model with 4 heads.}
      \label{fig:kl-overlapping-4heads}
\end{figure}

\begin{figure}[!h]
      \begin{center}
            \includegraphics[page=1,width=\linewidth]{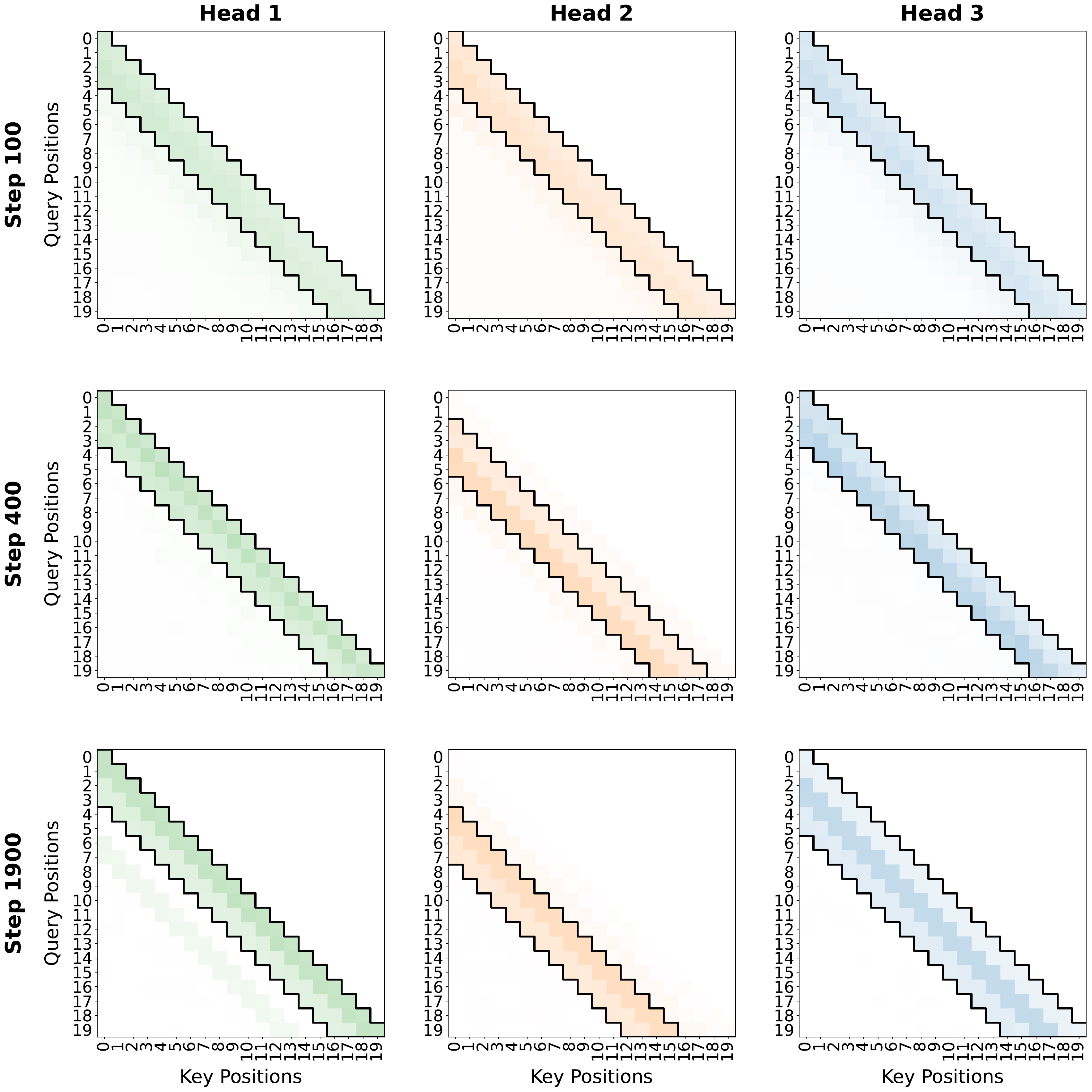}
      \end{center}
      \caption{Attention patterns for 3 heads over the training steps with overlapping intervals.}
      \label{fig:attention-three-heads-overlapping-intervals}
\end{figure}

\begin{figure}[!h]
      \begin{center}
            \includegraphics[page=1,width=\linewidth]{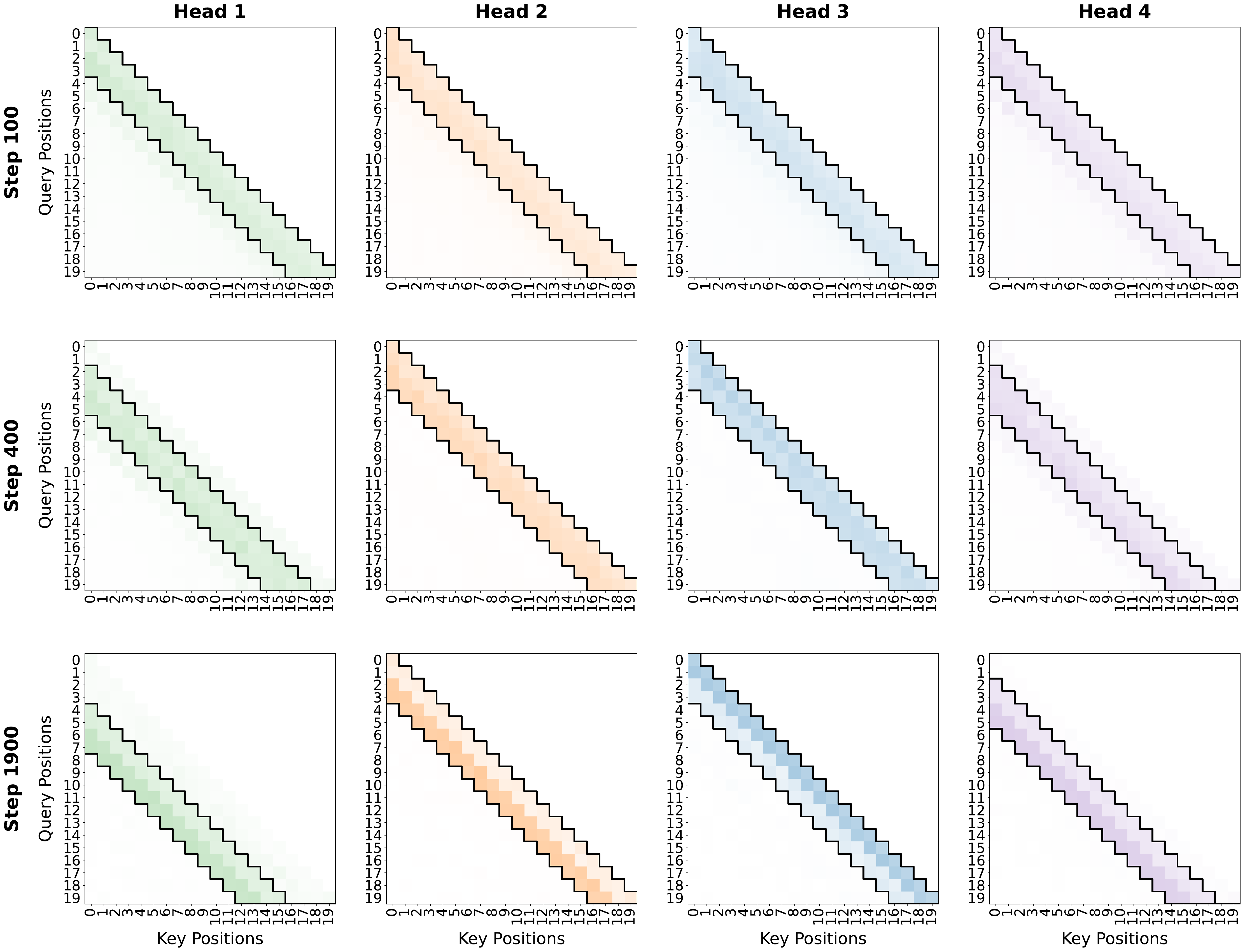}
      \end{center}
      \caption{Attention patterns for 4 heads over the training steps with overlapping intervals.}
      \label{fig:attention-four-heads-overlapping-intervals}
\end{figure}

\FloatBarrier
\subsection{Stochastic Gradient Descent (SGD)}
\label{app:sgd}
We run experiments with SGD optimizer instead of AdamW.
We observe in \Cref{fig:kl-sgd-full-one-layer} and \Cref{fig:attention-sgd-full-one-layer} that the qualitative behavior of incremental learning is the same.

\begin{figure}[!h]
      \begin{center}
            \includegraphics[page=1,width=0.48\linewidth]{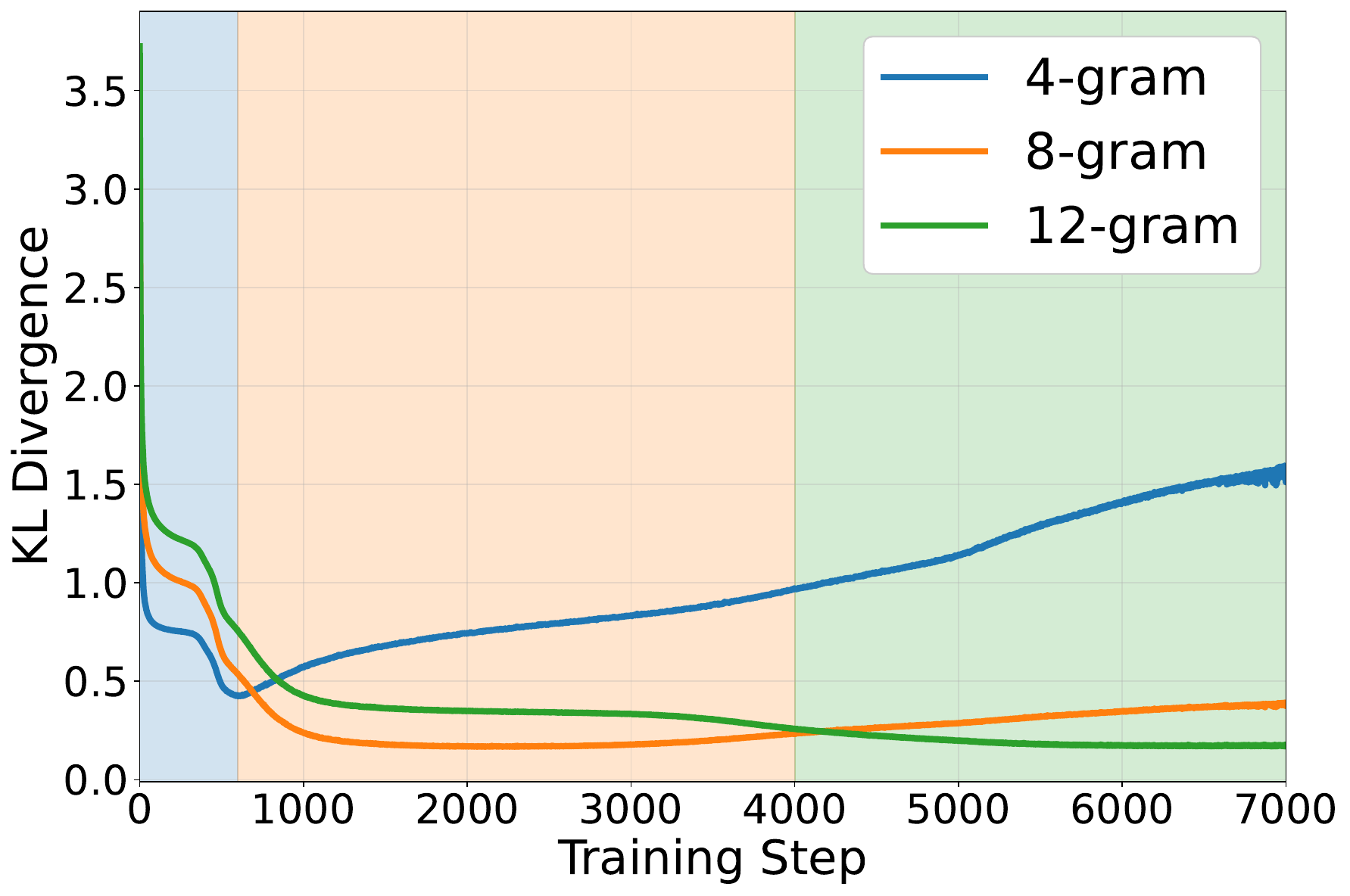}
      \end{center}
      \caption{KL divergence over the training steps with SGD optimizer.}
      \label{fig:kl-sgd-full-one-layer}
\end{figure}

\begin{figure}[!h]
      \begin{center}
            \includegraphics[width=\linewidth]{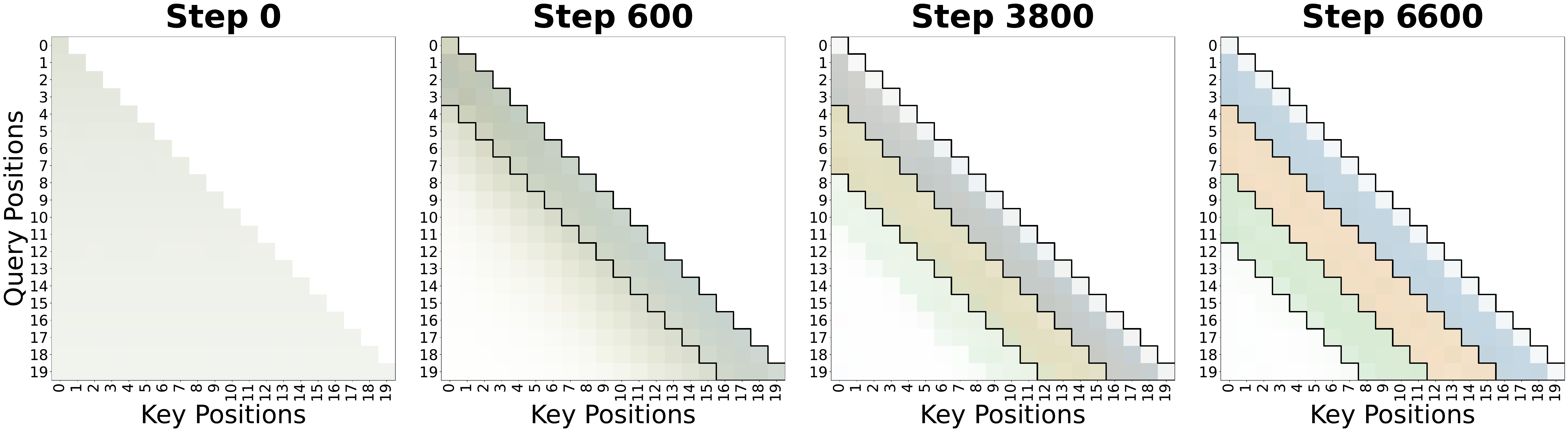}
      \end{center}
      \caption{Attention patterns over the training steps with SGD optimizer.}
      \label{fig:attention-sgd-full-one-layer}
\end{figure}

\end{appendices}

\end{document}